\documentclass[twoside]{article}

\usepackage[preprint]{aistats2027}
\usepackage[round]{natbib}

\usepackage{microtype}
\usepackage{graphicx}
\usepackage{subcaption}
\usepackage{booktabs}

\usepackage{amsmath}
\usepackage{amssymb}
\usepackage{mathtools}
\usepackage{amsthm}

\usepackage{bbm}
\usepackage{caption}
\usepackage{subcaption}
\usepackage{mathtools}
\usepackage{enumitem}
\usepackage{algorithm}
\usepackage{algpseudocode}

\usepackage{multirow}

\usepackage[dvipsnames]{xcolor}
\definecolor{citecolor}{RGB}{30,60,200}
\definecolor{linkcolor}{RGB}{100,180,50}

\usepackage{hyperref}
\hypersetup{
  colorlinks=true,
  citecolor=citecolor,
  linkcolor=linkcolor,
  urlcolor=black,
}

\providecommand{\best}{}\renewcommand{\best}[1]{\textcolor{bestgreen}{#1}}
\providecommand{\second}{}\renewcommand{\second}[1]{\textcolor{secondblue}{#1}}
\definecolor{bestgreen}{HTML}{1B7A46}
\definecolor{secondblue}{HTML}{1F5FA8}

\definecolor{acnotepink}{cmyk}{0.1, 0.4, 0, 0}

\usepackage[capitalize,noabbrev,nosort]{cleveref}
\usepackage{wrapfig}

\usepackage{tikz}
\usetikzlibrary{positioning}
\usetikzlibrary{arrows}
\usetikzlibrary{calc,fit}
\usetikzlibrary{shapes.geometric}
\usetikzlibrary{shapes.misc}
\usetikzlibrary{decorations.pathmorphing}
\usetikzlibrary{decorations.pathreplacing}
\usepackage{standalone}

\theoremstyle{plain}
\newtheorem{theorem}{Theorem}[section]

\newtheorem{lemma}[theorem]{Lemma}
\newtheorem{corollary}[theorem]{Corollary}

\theoremstyle{definition}

\theoremstyle{remark}

\makeatletter
\crefname{section}{\S\@gobble}{\S\@gobble}
\crefname{subsection}{\S\@gobble}{\S\@gobble}
\crefname{proposition}{Prop.}{Props.}
\crefname{figure}{Fig.}{Figs.}
\crefname{table}{Table}{Tables}
\crefname{conjecture}{Conjecture}{Conjectures}
\makeatother

\DeclareMathOperator*{\argmin}{arg\,min}

\DeclareMathOperator{\KL}{KL}

\newcommand{\dd}{\mathrm{d}}

\begin{document}

\twocolumn[

\aistatstitle{Amortised Adaptive Importance Sampling}

\aistatsauthor{ Arran Carter\textsuperscript{*} \And V\'ictor Elvira \And  Esmeralda S. Whitammer }
\runningauthor{Arran Carter, V\'ictor Elvira, Esmeralda S. Whitammer}

\aistatsaddress{ University of Edinburgh \And University of Edinburgh \And University of Edinburgh\\CIFAR Fellow } ]

{
\renewcommand{\thefootnote}{*}\footnotetext{
\small Correspondence to \texttt{arran.carter@ed.ac.uk}.\\ Code available \href{https://github.com/ArranCarter/Amortised_adaptive_importance_sampling}{\textcolor{citecolor}{here}}.
}
}

\begin{abstract}
    Monte Carlo sampling algorithms aim to produce samples from a target distribution with a given unnormalised probability density function.
    The high cost of running Monte Carlo algorithms, with many simulation steps or particles needed to obtain a good empirical approximation to the target, can be reduced by introducing learnable components into the sampler -- a form of amortisation.
    Once these components are trained, amortised samplers can produce high-quality samples with less computation at sampling time than training-free algorithms would require.
    We propose an amortised sampler built upon the framework of adaptive importance sampling -- a hierarchical importance sampling procedure where the proposal is iteratively adapted. 
    Our \emph{amortised adaptive importance sampling} algorithm fits a lightweight neural network that adapts the proposal using information obtained over the course of sampling.
    By training our sampler over many different targets taken from the same family of distributions, we obtain a sampler that can adapt to a specific target at sampling time without a priori access to the target's parameters. 
    We show that our procedure can efficiently sample from targets within the family the sampler was trained on, giving improvements over target-specific amortised samplers and training-free adaptive importance sampling algorithms.
    We also explore the potential for generalisation, where the trained sampler is used to sample from targets outwith its training family. 
\end{abstract}

\begin{figure*}[t]
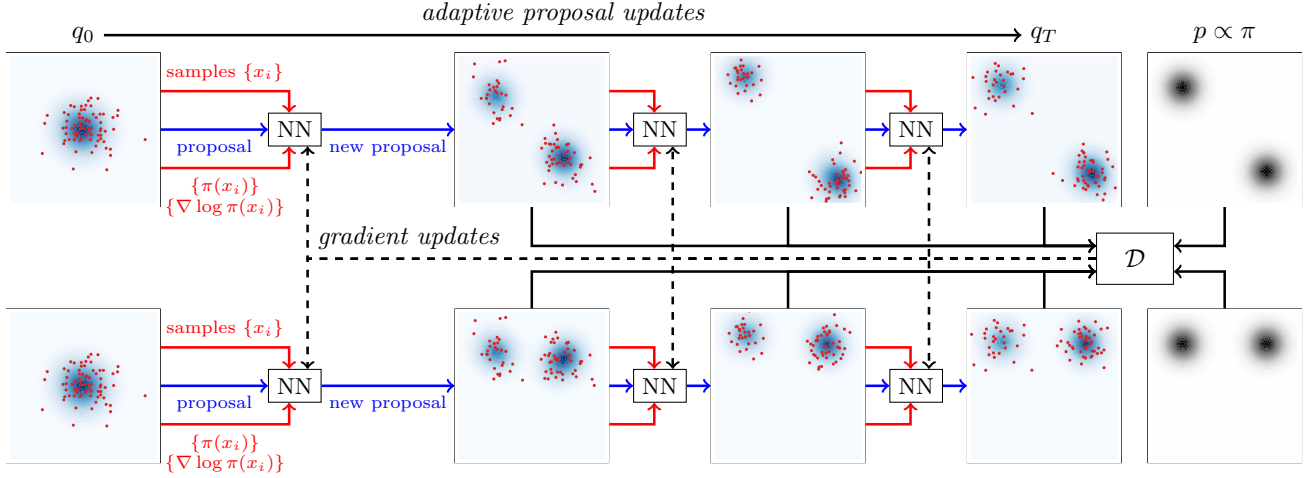

    \includestandalone[width=\linewidth]{figures/figure1}
    \caption{\textbf{Amortisation of a proposal adaptation over a family of targets.} Each row shows sequential adaptation of the importance sampling proposal in our implementation of amortised adaptive importance sampling. The samples and importance weights from each iteration's proposal $q_{t-1}$ are supplied to a neural network that predicts the parameters of the next iteration's proposal $q_t$. During training, the neural network is updated to minimise a divergence between all steps' proposals and the target distribution $p$. Training is over many distributions from a family; once trained, the sampler can be used for new targets in the same family.}
    \label{fig:figure1}
\end{figure*}

\section{Introduction}

We consider the problem of sampling from a distribution over $\mathbb{R}^d$ with probability density function
\begin{align}
    p(x) = \frac{1}{Z}\pi(x), \quad Z = \int_{\mathbb{R}^d}\pi(x)\,\dd x,
\end{align}
where we assume that we only have access to an unnormalised density $\pi \propto p$.  We also consider the related problem of estimating integrals of the form:
\begin{align}
    \label{eq:expectation}
    I = \mathbb{E}_{x\sim p}\left[f(x)\right] = \int_{\mathbb{R}^d} f(x)p(x)\,\dd x.
\end{align}
(We use the same notation for an absolutely continuous distribution and its density.)
These problems are of key importance in many applications, including Bayesian inference \citep{gelman1995bayesian} and statistical mechanics \citep{noe2019boltzmann}. One may want to solve both problems for many related distributions $p$ and test functions $f$.

\paragraph{Monte Carlo Methods} Monte Carlo methods are a commonly used solution to sampling problems. These methods form an empirical approximation to a single target distribution by performing a stochastic simulation. Many such methods simulate a Markov chain whose stationary distribution is the target, known as Markov chain Monte Carlo \citep{metropolis1953equation, hastings1970monte,roberts2004general,neal2000slice,neal2011mcmc}. Another common category of Monte Carlo methods are importance sampling algorithms, which sample particles from some proposal distribution and reweight them \citep{elvira2021advances}. 

\looseness=-1
There are various iterative versions of importance sampling, with multiple rounds of particle reweighting, including annealed importance sampling \citep{neal2001annealed}, sequential Monte Carlo \citep{del2006sequential}, and methods that adapt the proposal towards the target, known as adaptive importance sampling \citep[AIS;][]{cappe2004population,cappe2008adaptive,bugallo2017adaptive}. 
The adaptation of the proposal in AIS algorithms uses information obtained about the target at sampling time. Various adaptation rules have been proposed \citep{cappe2004population, cappe2008adaptive, elvira2023gradient, akyildiz2021convergence, paananen2021implicitly,elvira2022optimized}.

\paragraph{Amortised Sampling} The progress of deep learning methods has spurred the development of amortised sampling, wherein learnt components are introduced into sampling algorithms. Amortised sampling algorithms range from those that aim to learn MCMC transitions parametrised by a normalising flow \citep{gabrie2022adaptive} to expressive dynamics-based approaches \citep{berner2026from}. Amortised importance sampling methods include  \citet{muller2019neuralis}, who parametrise the proposal using normalising flows. See \citet{gabrie2026leveraging} for a recent review of amortised sampling.

Many approaches to amortised sampling only consider sampling from a single distribution. 
Target-specific training becomes unattractive when sampling must be repeated across many related targets, since training cost is incurred separately for each target.
Recently, there has been interest in algorithms that amortise over multiple Monte Carlo problems, for example, producing samples for many distributions from the same family (e.g., Bayesian posteriors conditioned on different observations). Advances made in this direction include those of \citet{golinski2019amortized, sun2023meta, klein2023timewarp, klein2024transferable}. 

\paragraph{Our Contribution} We formulate a way to incorporate target-specific information into importance sampling proposals \emph{at sampling time} by learning the adaptation rule in AIS (\cref{sec:aais}). We obtain an algorithm that can efficiently adapt a starting proposal to any target within a family of distributions without requiring information that specifies the target in advance (\Cref{fig:figure1}). Empirically, we show the benefit of amortisation within an AIS procedure over AIS algorithms with static adaptation rules (\Cref{fig:funnel-banana}).

The only target-specific information required by our amortised adaptive importance sampling (AAIS) algorithm is obtained during the sampling procedure, for example, through pointwise evaluations of the (possibly unnormalised) target density function and possibly its derivatives. This allows AAIS to sample from multiple different targets within the family it was trained on, but it can also be run on targets outwith its training family. We investigate this promising direction and show that AAIS can still perform efficiently when the evaluation family differs from the training family.

\section{Preliminaries}

\subsection{Importance Sampling}

Importance sampling is a Monte Carlo algorithm that allows for estimating integrals of the form given by \eqref{eq:expectation} without requiring samples from $p$. Instead, samples $\{x_s\}_{s=1}^S$ are drawn from a \emph{proposal distribution}, $q$, and the corresponding Monte Carlo estimator of $I$ is appropriately reweighted:
\begin{align}\label{eq:unbiased-IS}
    \widehat I = \frac{1}{SZ}\sum_{s=1}^Sw(x_s)f(x_s),  \qquad w(x_s) = \frac{\pi(x_s)}{q(x_s)}.
\end{align}
This estimator is unbiased and consistent, and setting $q=p$ recovers the standard Monte Carlo estimator for \eqref{eq:expectation}. If $Z$ is unknown, it can be substituted with an estimator:
\begin{align}\label{eq:Z-estimator}
    \widehat Z = \frac{1}{S}\sum_{s=1}^Sw(x_s).
\end{align}
Replacing $Z$ in \eqref{eq:unbiased-IS} with $\widehat Z$, we obtain the self-normalised importance sampling (SNIS) estimator:
\begin{align}
    \widehat I_{\rm SNIS} = \frac{\sum_{s=1}^Sw(x_s)f(x_s)}{\sum_{s=1}^Sw(x_s)}
    &=\sum_{s=1}^S\bar{w}(x_s)f(x_s),\\
    \bar w(x_s)&=\frac{w(x_s)}{\sum_{s'=1}^Sw(x_{s'})}
    .
\end{align}
This estimator is biased but consistent. In our experiments, we only consider the SNIS estimator.

The choice of the proposal distribution $q$ is crucial to the efficiency of importance sampling. In particular, the variance of $\widehat Z$ is proportional to the $\chi^2$ divergence between the target and the proposal, and large discrepancies between the two distributions may lead to highly variable weights and poor estimates \citep{agapiou2017importance}. This motivates the design or adaptation of proposal distributions that approximate the target while remaining easy to sample from and evaluate.  In practice, since an appropriate proposal is rarely known in advance, it is often advantageous to adapt it using information gathered during the sampling procedure.
 
\subsection{Adaptive Importance Sampling} 

Adaptive importance sampling (AIS) algorithms iteratively adapt the proposal distribution using information from previous sampling iterations. AIS consists of three main stages iterated for $T$ steps. Let $q_1$ be the initial proposal, at step $t$: sample $x_s^{(t)} \sim q_{t}, s=1,\dots,S$; compute weights $w(x_s^{(t)})$ using \eqref{eq:unbiased-IS}; and then adapt the proposal to obtain $q_{t+1}$; see \cref{alg:general_ais}. One option is to construct the self-normalised AIS estimator of $I$:
\begin{align}\label{eq:ais-estimator}
    \widehat I_{\rm AIS} = \frac{\sum_{t=1}^{T+1}\sum_{s=1}^S w(x_s^{(t)})f(x_s^{(t)})}{\sum_{t=1}^{T+1}\sum_{s=1}^S w(x_s^{(t)})}.
\end{align}

\begin{algorithm}
  \begin{algorithmic}[1]
    \caption{Adaptive Importance Sampling}
    \label{alg:general_ais}
    \Require $S$, $T$, initial proposal $q_1$, target density $\pi$.
    \For{$t=1,\dots, T$}
    \State Draw $S$ samples $\{x_s^{(t)}\}_{s=1}^S$from $q_{t}$
    \State Calculate $w(x_s^{(t)}), s=1,\dots,S$
    \State Adapt $q_{t} \rightarrow q_{t+1}$
    \EndFor
    \State Draw and weight samples from $q_{T+1}$
    \State \Return $S(T+1)$ pairs of samples and their weights, $\{(w_s^{(t)}, x_s^{(t)})\}_{s=1, t=1}^{S, T+1}$, and the proposals $q_{1:T+1}$. 
  \end{algorithmic}
\end{algorithm}

Beyond improving the efficiency of importance sampling estimators, AIS can be used to construct proposals $q_t$ that progressively approximate the target $p$.
The resulting approximation to $p$ may be useful in its own right, for example, to draw posterior samples in a Bayesian problem and estimate posterior predictives.

A wide variety of AIS algorithms have been proposed, differing in several aspects, most notably in the adaptation rule -- the mechanism used to adapt the proposal parameters across iterations -- and in the parametric family adopted for the proposal. Common adaptation strategies include resampling-based mechanisms, where the locations of mixture components are adapted using the weighted sample population \citep{cappe2004population,elvira2017improving}; geometry-based approaches, which exploit gradient and curvature information from the target to adapt the locations and covariance matrices of the proposal components \citep[GAPIS;][]{elvira2015gradient} \citep[GRAMIS;][]{elvira2023gradient}; and parametric fitting approaches, where the proposal parameters are adapted to better approximate the target. The latter include moment-matching schemes, as in AMIS \citep{cornuet2012adaptive}, as well as methods that formulate the adaptation explicitly as the minimisation of a statistical divergence between the proposal and the target \citep{ryu2016convex,akyildiz2021convergence}.

\section{Amortised Adaptive Importance Sampling}
\label{sec:aais}

We propose AAIS, which replaces the prespecified adaptation rule of AIS with a learnt function $u_\phi$. The function is trained across AIS rollouts with different target distributions and maps information about the target collected during sampling to the next proposal.

\paragraph{Single-Target Amortisation}  We briefly introduce the idea through the single-target amortisation setting. In order to amortise AIS, we propose to learn the adaptation rule used in line $4$ of \cref{alg:general_ais}. We parametrise this adaptation rule with a function $u_\phi$ with learnable parameters $\phi$. At step $t$, this function takes in the observables, $O_{1:t}$, of the AIS rollout up to and including time $t$ (e.g., target density evaluations, gradients) as well as the proposals up to time $t$, $q_{1:t}$, and outputs the new proposal $q_{t+1}$. 

Let $R$ denote the output of an AIS procedure (\cref{alg:general_ais}) -- the pairs of samples with their weights from each step and the proposals $q_{1:T+1}$. The parameters of this adaptation function, $u_\phi$, can then be trained to minimise some real-valued objective, $\mathcal{D}(R, \pi,f)$, which takes in the AIS output, $R$, as well as the target density $\pi$ and, if desired, the function $f$ for which the expectation in \eqref{eq:expectation} is to be estimated. The output, $R$, of AIS is random, so the objective is to minimise:
\begin{align}\label{eq:single-target-obj}
    \ell(\phi) = \mathbb{E}_R\left[\mathcal{D}(R, \pi, f)\right],
\end{align}
with respect to the parameters $\phi$.

An example choice for $\mathcal{D}$ is the $\KL$ divergence of the final proposal to the target $\pi$:
\begin{align}
    \mathcal{D}(R, \pi, f) &= \KL(q_{T+1} \mid\mid p) -\log Z\\
    &= \int_{\mathbb{R}^d} q_{T+1}(x)\log\frac{q_{T+1}(x)}{\pi(x)}\dd x.
\end{align}
Note that $R$ depends on $\phi$, so extra care is required in order to perform gradient descent on \eqref{eq:single-target-obj}. The REINFORCE estimator \citep{williams1992simple} can be used, or certain choices of the $q_t$ can allow reprarametrisation.

Many choices for $\mathcal{D}$ are possible; we explore some of them empirically in \cref{sec:obj-choice,app:obj-abl}. For notational simplicity, we will drop the dependence of $\mathcal{D}$ on $f$, since in our experiments the objectives do not depend on it.

\paragraph{Amortisation over Distribution Families} 

All AIS algorithms use information about the target distribution $p$, via its unnormalised density $\pi$ or its derivatives at sampled points, to adapt the proposal. As a result, the (distributions of) sequences of proposals over AIS iterations differ between targets.

However, in AAIS, the adaptation rule itself is learnt. The optimal way to adapt the proposal using new information about the target may depend on information that is not contained in the current AIS rollout's observables. Instead, the information may come from training over a family of target distributions that share some structure.
For instance, in \cref{fig:figure1}, the optimal adaptation of the proposal could take into account the knowledge, acquired during training on many targets, that the target is a mixture of two Gaussians.

Extending the single-target amortisation procedure, we train the adaptation rule over targets $\pi \sim \Pi$, where $\Pi$ is a distribution over unnormalised density functions, by minimising  \eqref{eq:single-target-obj} in expectation over $\pi$:
\begin{align}\label{eq:family-obj}
    \ell(\phi) = \mathbb{E}_{\pi\sim \Pi}\left[\mathbb{E}_R\left[\mathcal{D}(R, \pi)\right]\right],
\end{align}
in a similar way to that considered in \citet{paige2016inference, golinski2019amortized}. \cref{alg:amortised_ais} details the training of the $u_\phi$. Once trained, the $u_\phi$ can be used to perform the adaptation step in \cref{alg:general_ais}.

Minimising \eqref{eq:family-obj} is not the same as attempting to find a single proposal that fits many distributions well. The output $R$ is dependent on both the adaptation rule parameters $\phi$ and the target $\pi$, so the objective in \eqref{eq:family-obj} aims for the best proposal for every individual $\pi$ sampled from $\Pi$. 

\paragraph{Characterising the Optimal Adaptation Rule} To give further clarification, below we consider the optimal adaptation rules. 

First, the optimal adaptation rule $u_\pi^*(O_{1:t}, q_{1:t})$ for a proposal $q_{t}$ with a single target $\pi$ is: 
\begin{align}
    u_\pi^*(O_{1:t}, q_{1:t}) = \argmin_{q_{t+1}}\mathbb{E}_{R}\left[\mathcal{D}(R, \pi) \mid O_{1:t}, q_{1:{t}}\right].
\end{align}
In the family amortisation setting, we can only infer which target is being sampled from based on the information gathered at each step, thus the corresponding optimal rule for amortisation over $\Pi$ is:
\begin{align}
    &u_\Pi^*(O_{1:t}, q_{1:t}) \nonumber \\&= \argmin_{q_{t+1}}\mathbb{E}_{\pi\sim \Pi}\left[\mathbb{E}_R\left[\mathcal{D}(R, \pi)\right] \mid O_{1:t}, q_{1:t}\right]\\
    &= \argmin_{q_{t+1}} \int \mathbb{E}_R[\mathcal{D}(R, \pi) \mid O_{1:t}, q_{1:t}] \,\dd\Pi(\pi\mid O_{1:t}).    \nonumber
\end{align}
Comparing this with the single-target optimal adaptation, we note the additional weighting by the posterior probability of each $\pi$ given the observed $O_{1:t}$. So, in a Bayesian decision theory framework with loss $\mathbb{E}_R[\mathcal{D}(R, \pi)\mid O_{1:t}, q_{1:t}]$, the optimal AAIS adaptation coincides with the optimal decision. 

In the preceding setup, using intermediate targets $\pi_t$ in the objective is also possible. The adaptation rule $u_\phi$ could also return a distribution over proposals.

\begin{figure*}[t]
\vspace*{-0.5em}
\centering
\begin{tabular}{@{}c@{\hspace{1pt}}c@{\hspace{1pt}}c@{\hspace{1pt}}c@{\hspace{1pt}}c@{\hspace{1pt}}c@{}}
& True Samples & AAIS & DM-PMC & M-PMC & GAPIS\\
        \rotatebox{90}{Funnel ($d=50$)}&\fbox{\includegraphics[width=0.16\textwidth]{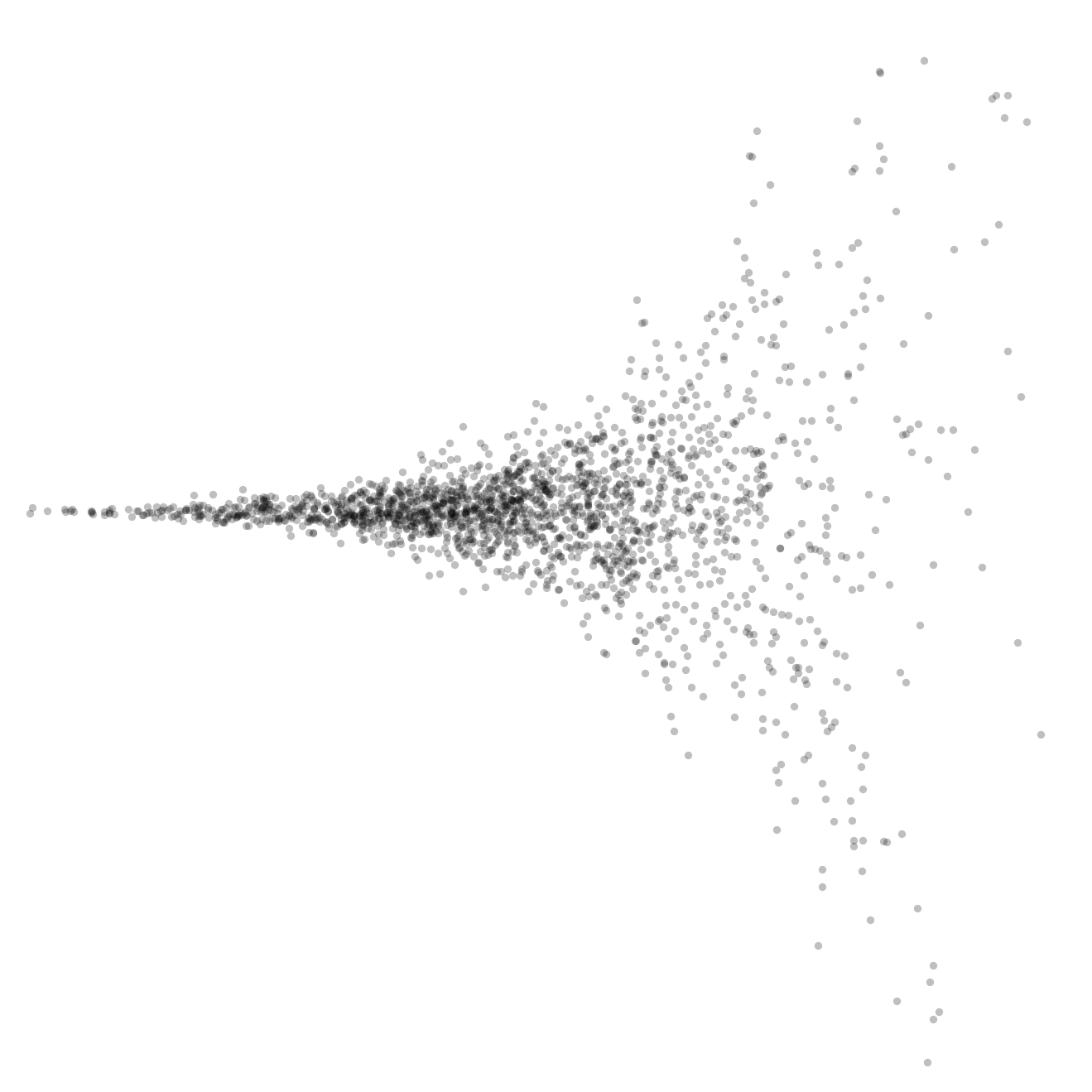}}
        &\fbox{\includegraphics[width=0.16\textwidth]{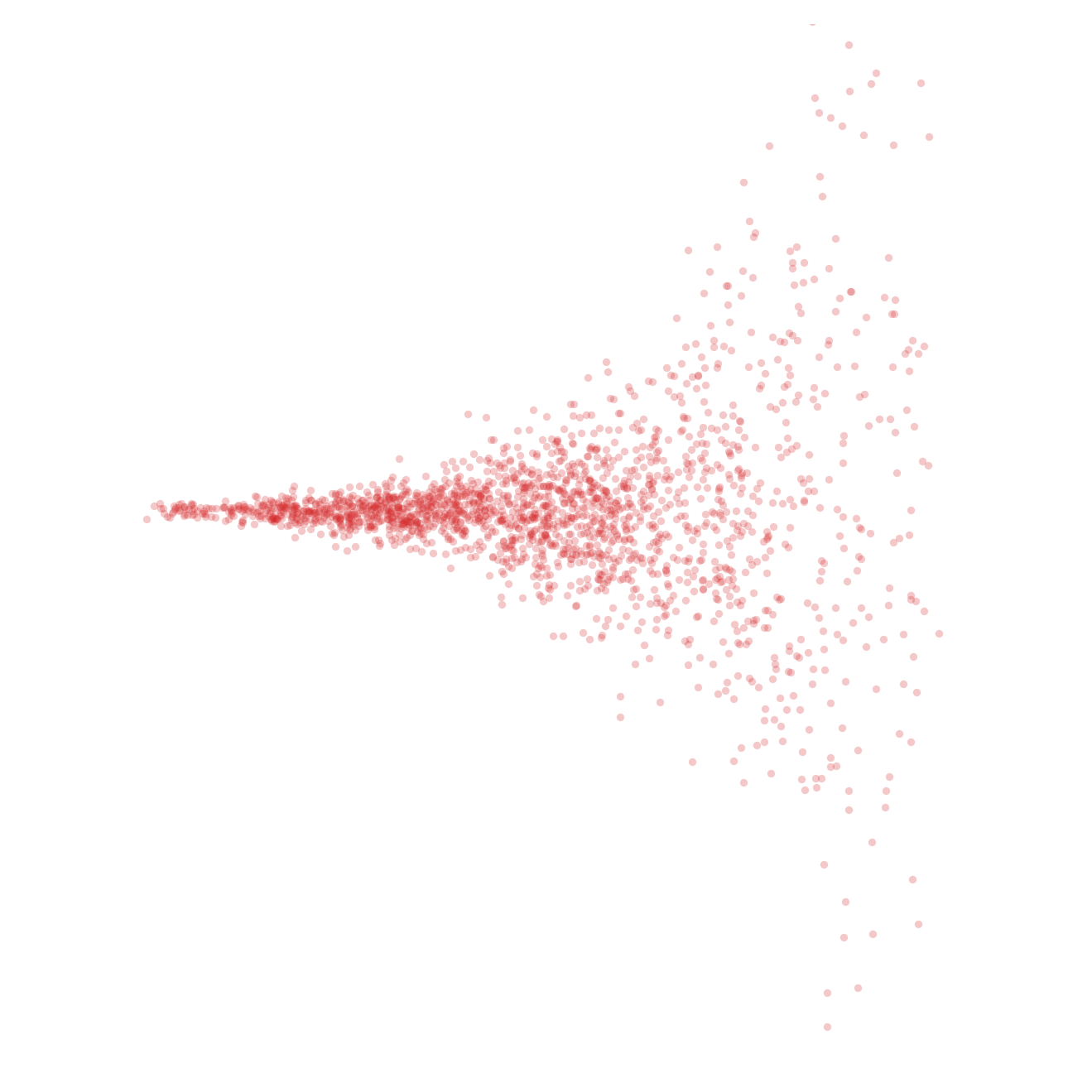}}
        &\fbox{\includegraphics[width=0.16\textwidth]{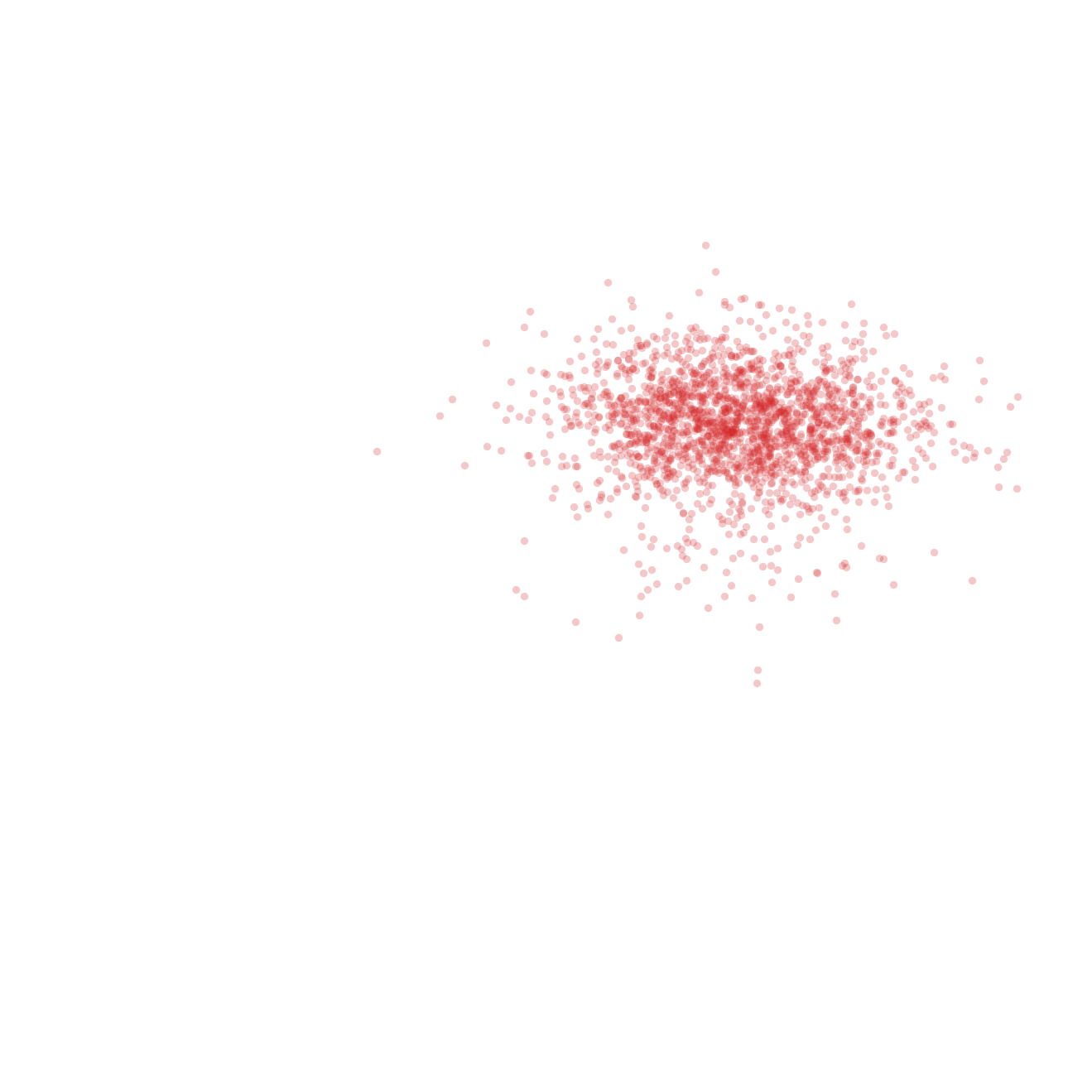}}
        &\fbox{\includegraphics[width=0.16\textwidth]{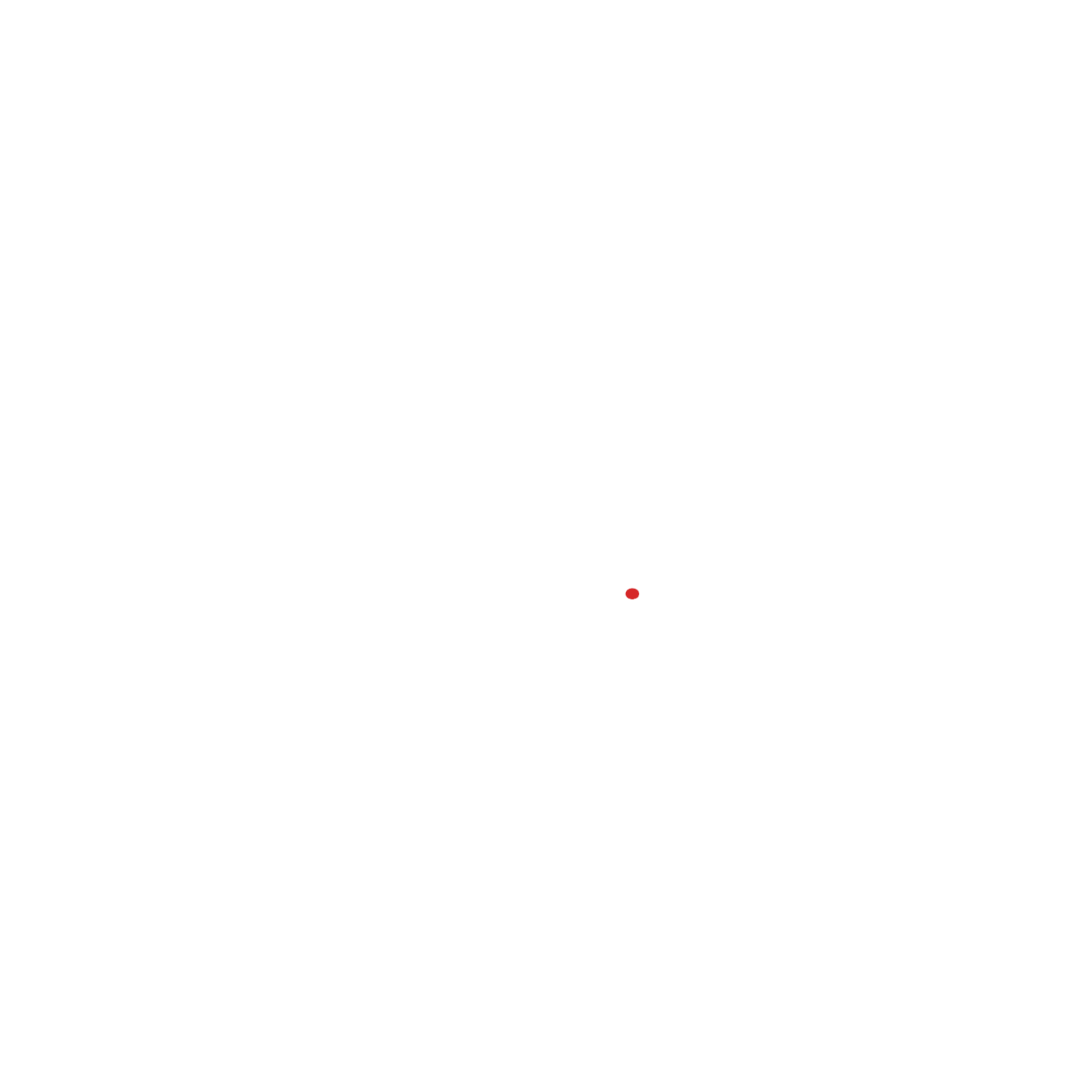}}
        &\fbox{\includegraphics[width=0.16\textwidth]{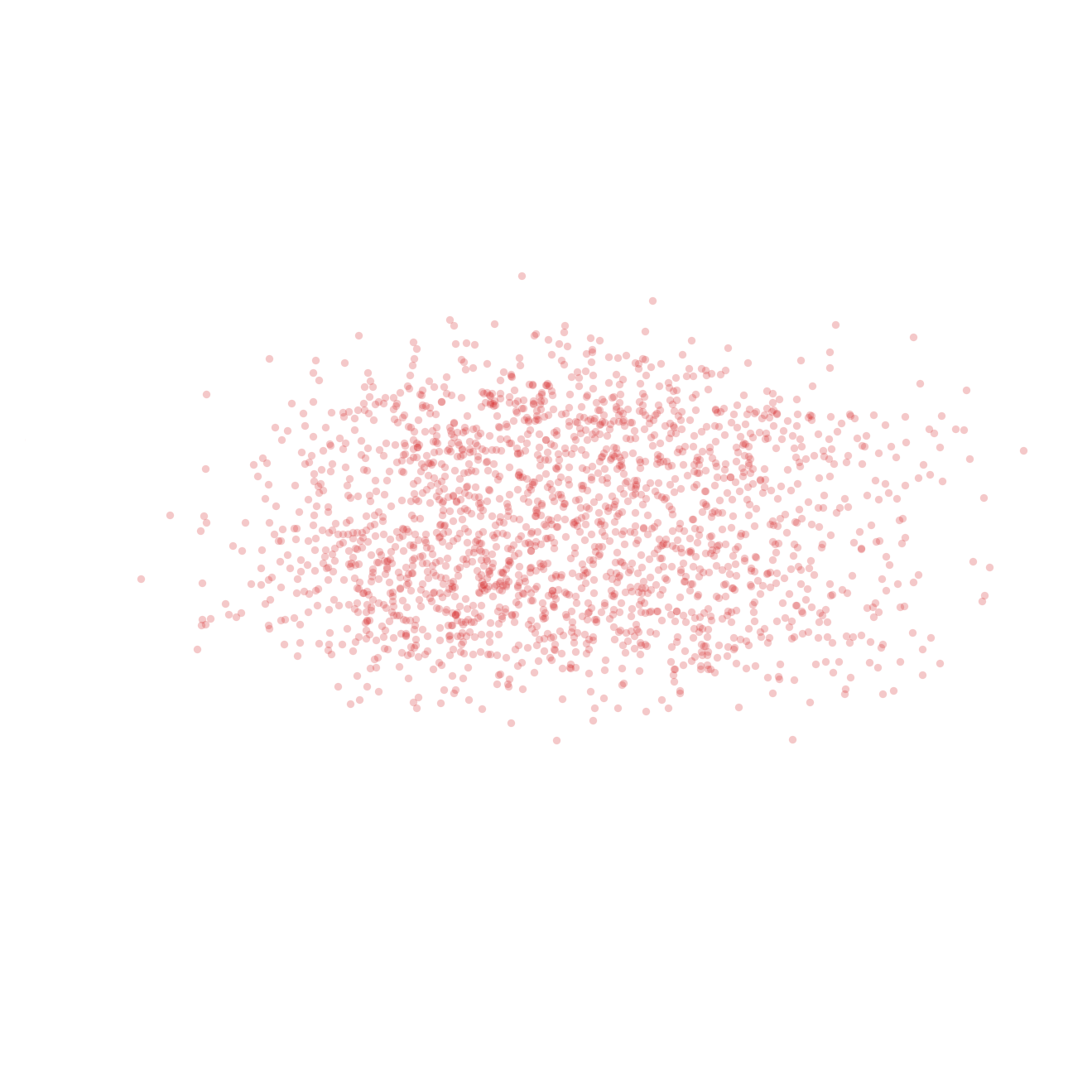}}
\\
        \rotatebox{90}{Banana ($d=5$)}&\fbox{\includegraphics[width=0.16\textwidth]{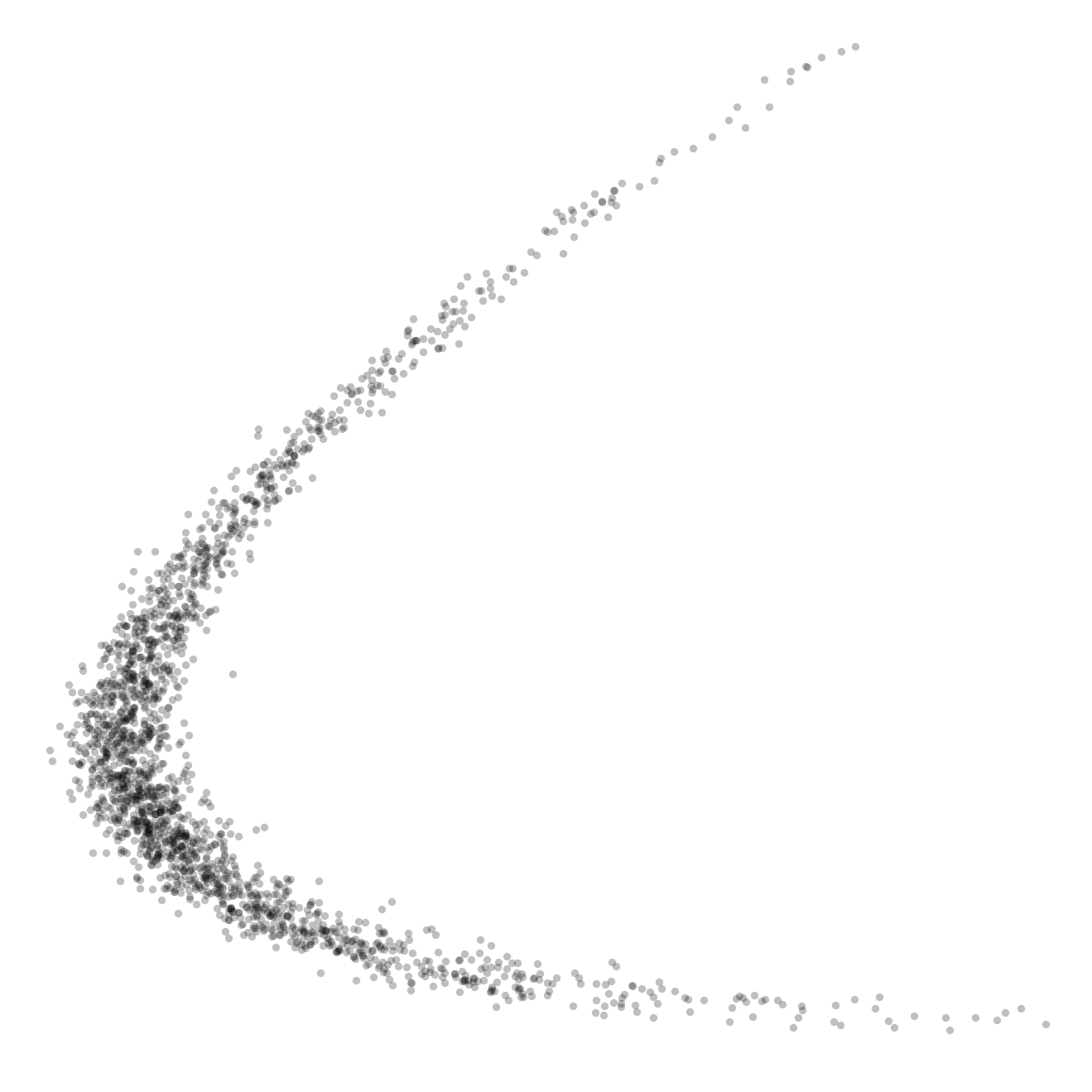}}
        &\fbox{\includegraphics[width=0.16\textwidth]{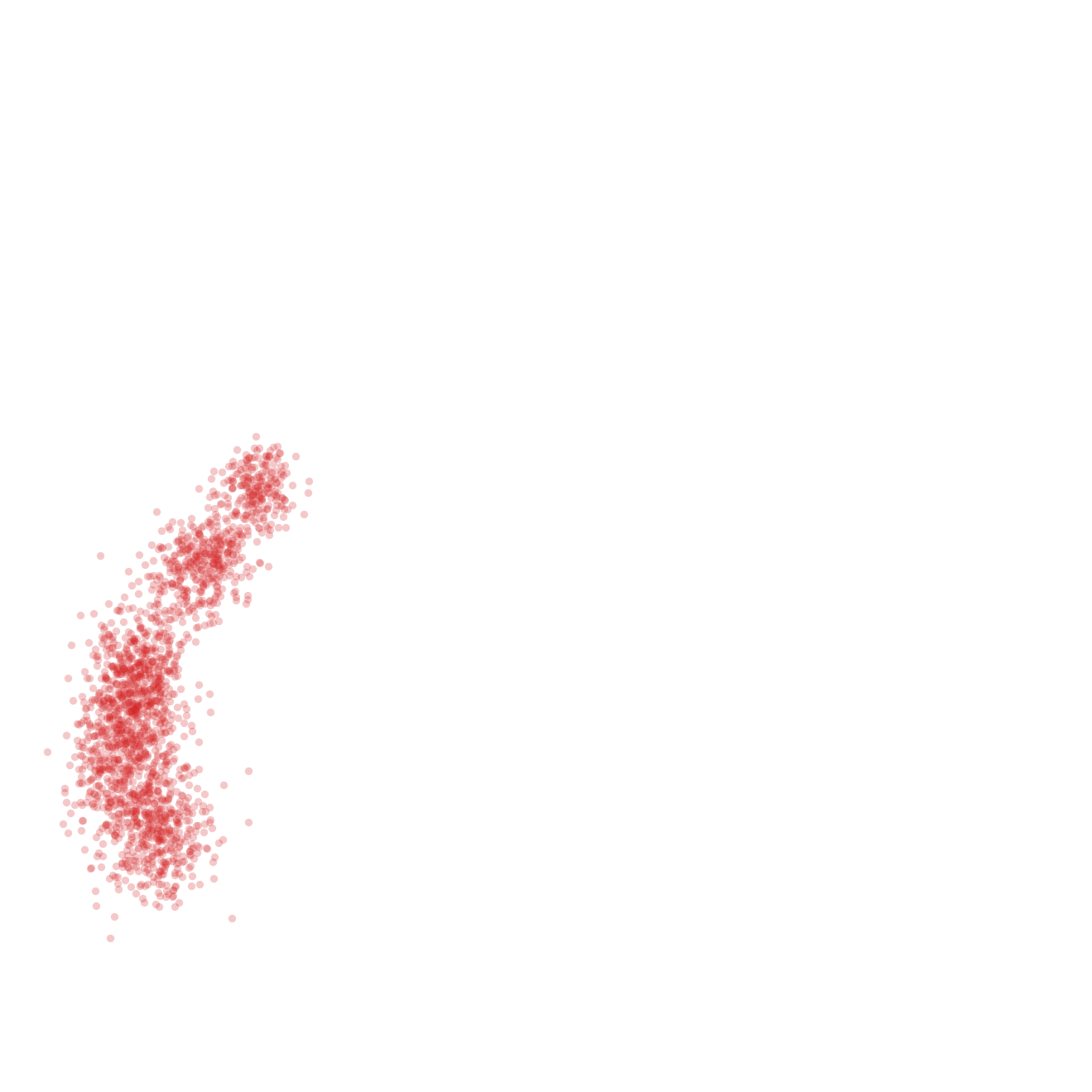}}
        &\fbox{\includegraphics[width=0.16\textwidth]{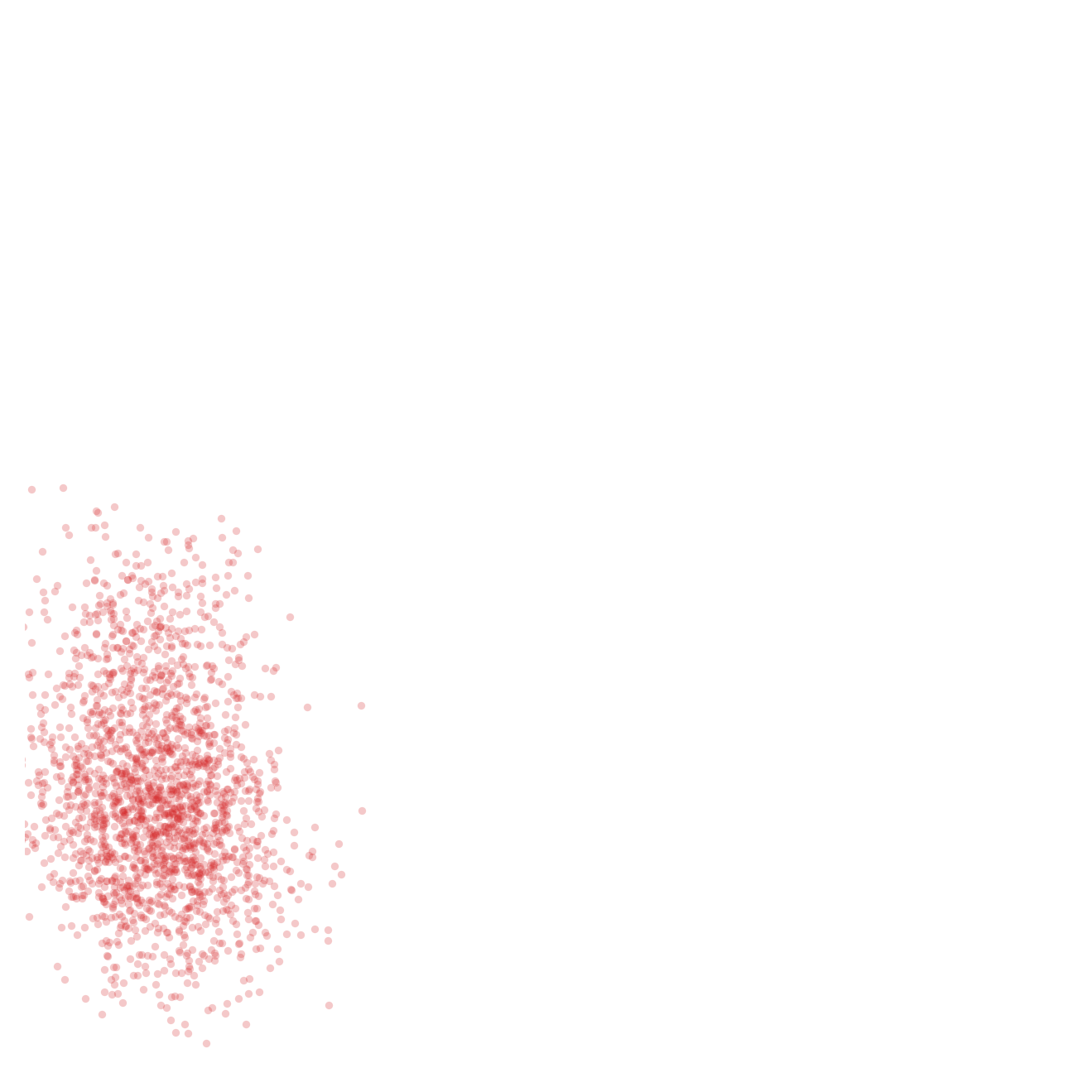}}
        &\fbox{\includegraphics[width=0.16\textwidth]{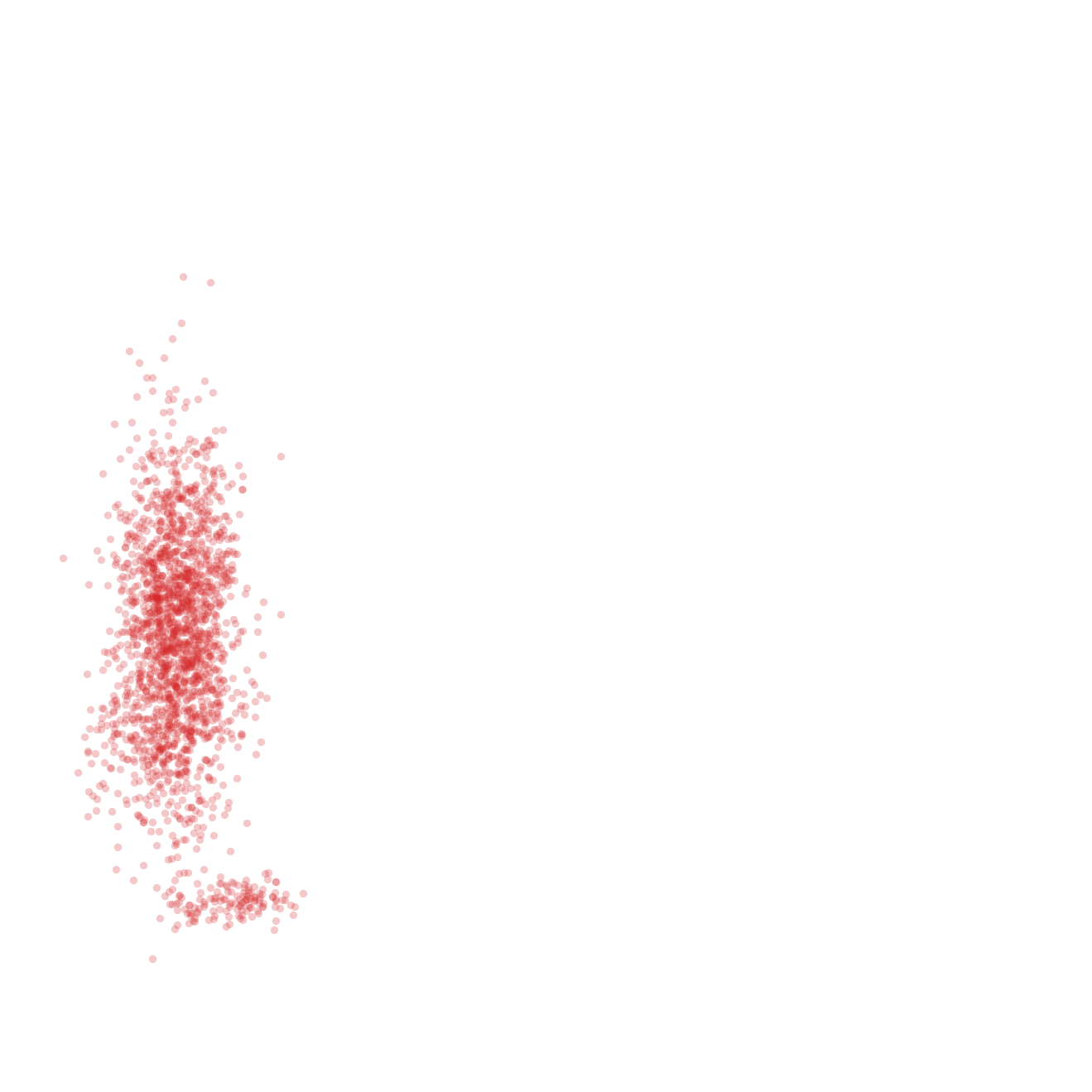}}
        &\fbox{\includegraphics[width=0.16\textwidth]{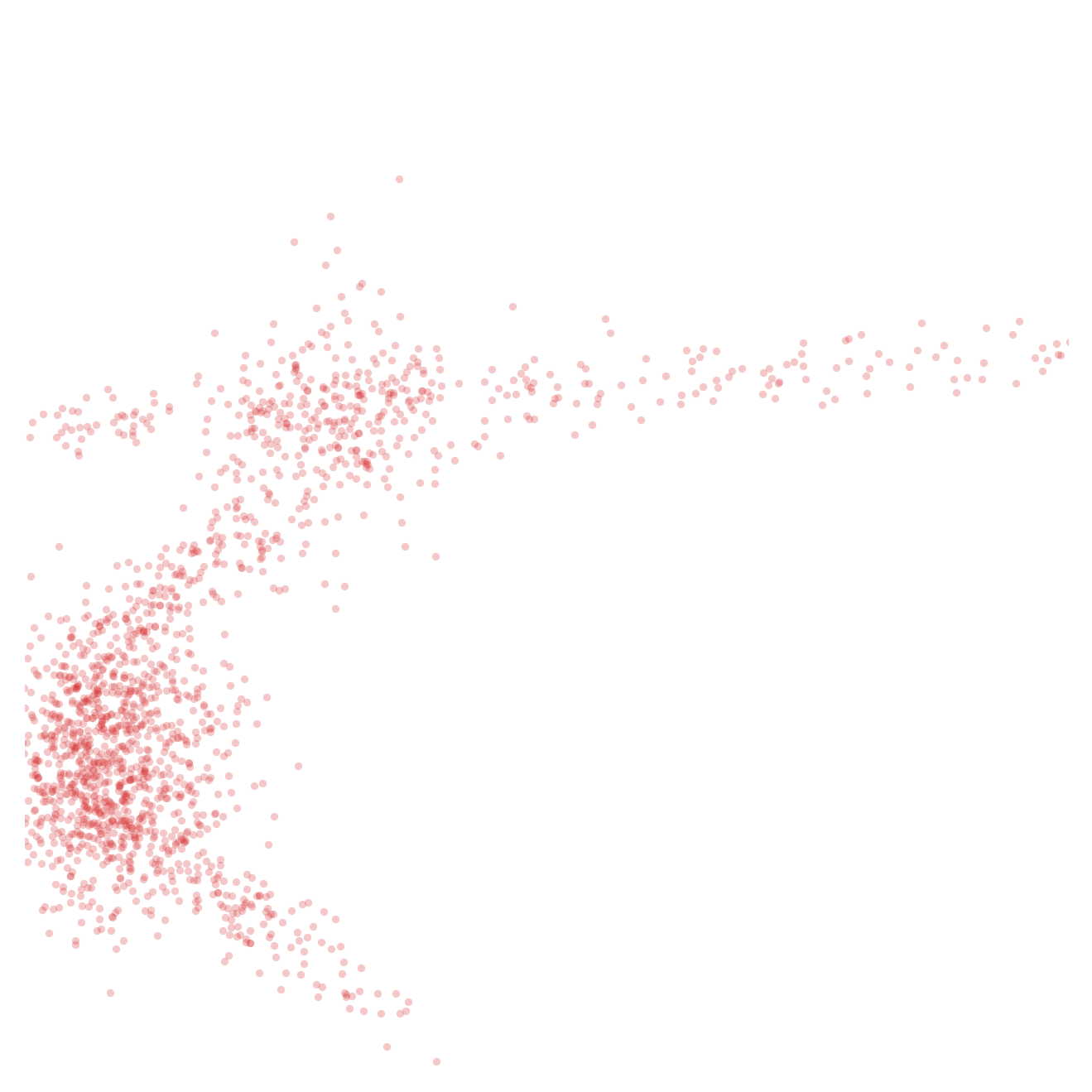}}
\end{tabular}
    \caption{Samples generated by the final step proposal for each method for $50$-dimensional funnel (top row) with $N=50$ components in the proposal mixture, and $5$-dimensional banana (bottom row) with $N=10$ components in the proposal mixture, each projected into their first two coordinates.}
    \label{fig:funnel-banana}
    \vspace*{-0.5em}
\end{figure*}

\paragraph{Our AAIS Implementation} 

The above is a general setup for AAIS. Choices must be made for: the objective $\mathcal{D}$, the parametrisation of $u_\phi$, and the form of proposal.

We parametrise our proposals as equally-weighted Gaussian mixtures with $N$ components:
\begin{align}
    q_t(x) = \frac{1}{N}\sum_{n=1}^N \mathcal{N}(\mu_n^{(t)}, \Sigma_n^{(t)}),
\end{align}
where the $\mu_n^{(t)}$ and $\Sigma_n^{(t)}$ are adapted at each step. We use a diagonal $\Sigma_n^{(t)}$ in our experiments.

\looseness=-1
Our adaptation function is chosen to be scale-equivariant (see \cref{app:equivariance}). This gives the sampler better generalisation (see \cref{app:equivariance-abl}) and also removes the requirement for scale-specific hyperparameter tuning. Our $u_\phi$ produces additive adaptations to the $\mu_n^{(t)}$ and multiplicative adaptations to the $\Sigma_n^{(t)}$, taking in only the current time step's information and proposal -- the samples $x_{s}^{(t)}$, their weights $w(x_s^{(t)})$, target gradient evaluations at the $x_s^{(t)}$ and $\mu_n^{(t)}$ , the time step $t$, and $q_{t}$. We use a neural network structure similar to \citet{zaheer2017deep} making use of multiple lightweight multilayer perceptrons and pooling to ensure permutation-invariance with respect to the input sample-weight pairs. This allows the final AAIS algorithm to run with any number of samples taken at each step and for any number of components -- not necessarily just the values used when training. See \cref{app:update-function} for full adaptation function details.

Our choice for $\mathcal{D}$ is a step-wise KL divergence:
\begin{align}\label{eq:obj-implementation}
    \mathcal{D}(R, \pi) = \frac{1}{T}\sum_{t=2}^{T+1}\KL(q_t \mid\mid p)+\text{const}.
\end{align}
To allow for gradients to propagate from $\KL(q_t \mid \mid p)$ back to the previous adaptations we make use of the reparametrisation trick for Gaussians and the deterministic-mixture sampling procedure of \citet{elvira2017improving}, where an equal number of samples are taken from each component in $q_{t}$.

To stabilise training we make use of annealing in the early steps of training, and we modify the estimator of the $\KL$ divergences to use compressed log-density ratios. Details are in \cref{app:training-details}, and the full algorithm for training our AAIS implementation is given in \cref{alg:amortised_ais_implementation}, along with details of algorithmic cost in \cref{app:update-function}.

\section{Experiments}

\looseness=-1
We empirically evaluate the benefit of amortisation over families by comparing our AAIS implementation to previous AIS algorithms with different adaptation rules. We compare against: deterministic-mixture population Monte Carlo \citep[DM-PMC;][]{elvira2017improving} -- a resampling based adaptation rule, mixture PMC \citep[M-PMC;][]{cappe2008adaptive} -- an EM adaptation rule (equivalent to minimising the $\KL$ divergence between the target and a Gaussian mixture proposal $q_t$, $\KL(p\,\mid\mid\,q_t)$), and the gradient adaptive population importance sampler \citep[GAPIS;][]{elvira2015gradient} -- adaptation based on target gradients and Hessian information. All samplers are run using a matched number of target density evaluations, and where gradients are calculated they count as $2$ evaluations. These baselines give a range of types of adaptation rules, and they require little hyperparameter tuning. GAPIS requires the gradient step size to be chosen; for every target family we run GAPIS for different values of the step size on a validation set of targets and select the parameter value with the best performance. 

We compare samplers using three main metrics: sliced 1-Wasserstein distance \citep{bonneel2015sliced} between the final proposal and the target (sliced $W_1^{\rm prop}$), the effective sample size (ESS) using the temporal mixture of all proposals, and the error in the estimated log-normalising constant $\log Z$. To assess sampling on a family of distributions, we average these metrics over a batch of test targets and also compute the worst $5\%$ quantile for each metric over the batch. All metric details are given in \cref{app:metrics}.

\subsection{Sampling in Families}

Details of all target families are given in \cref{app:targets}.

\paragraph{Synthetic Densities}
We evaluate AAIS on a family formed from Neal's funnel distribution \citep{neal2000slice}:
\begin{align}\label{eq:funnel}
    x_1\sim \mathcal{N}(0, \sigma^2), \quad x_{2:d}\mid x_1 \sim \mathcal{N}(0, e^{x_1}I_{d-1}).
\end{align}
We turn this into a family of distributions by randomly sampling $\sigma\sim \mathcal{U}\left[\frac{3}{2}, 3\right]$.

\cref{tab:results-funnel} shows that AAIS outperforms the baselines consistently, with a final proposal better matched to the target, better ESS, and lower normalising constant estimation error. AAIS shows far better scalability for the problem as it is the only sampler that effectively samples from the funnel at $50$ dimensions, having magnitudes larger ESS and a much more accurate estimation of the normalising constant. AAIS outperforms on the metric means and also has better performance on the worst 5\% quantile of targets. \cref{fig:funnel-banana} (top) shows samples taken from the final proposal, $q_{T+1}$, of each AIS algorithm considered, where we see that AAIS is the only algorithm with a well-matched final proposal for the $50$-dimensional funnel. 

We also evaluate on a family built from banana-shaped distributions -- \cref{fig:funnel-banana} (bottom) -- where amortisation performs similarly well, see \cref{tab:results-banana-full} for full results.

\begin{table}[t]
  \centering
  \caption{Experimental results for Funnel $d=5,25, 50$. We report the mean over 256 test targets from the family and (for AAIS) the std.dev.\ of this mean over $3$ training seeds. The worst $5\%$ quantile of the metric over the batch of targets is shown in parentheses. Best in \best{green}, second-best in \second{blue}.}
  \vspace*{-0.5em}
  \resizebox{\linewidth}{!}{
    \begin{tabular}{@{}llccccccc}
      \toprule
         & Method $\downarrow$ Metric $\rightarrow$ & sliced $W_1^{\mathrm{prop}}$ $\downarrow$ & ESS $\uparrow$ & $\log Z$ err. $\downarrow$ \\ \midrule
        \multirow{5}{*}{\rotatebox{90}{$d=5$}} & AAIS & \best{0.4658\,{\scriptsize$\pm$0.002}\,{\scriptsize(0.978)}} & \best{0.2605\,{\scriptsize$\pm$0.008}\,{\scriptsize(0.015)}} & \best{0.1250\,{\scriptsize$\pm$0.005}\,{\scriptsize(0.241)}} \\
         & DM-PMC & 0.9592\,{\scriptsize(1.69)} & 0.0301\,{\scriptsize(0.002)} & 0.2691\,{\scriptsize(0.537)} \\
         & M-PMC & \second{0.7222\,{\scriptsize(1.43)}} & \second{0.0965\,{\scriptsize(0.004)}} & \second{0.1777\,{\scriptsize(0.347)}} \\
         & GAPIS & 1.678\,{\scriptsize(2.44)} & 0.0055\,{\scriptsize(0.001)} & 0.4400\,{\scriptsize(1.05)} \\ \cmidrule(l){2-5} 
         & True samples & 0.4451\,{\scriptsize(0.848)} & 1.000\,{\scriptsize(1.00)} & 0.0000\,{\scriptsize(0.000)} \\ \midrule
        \multirow{5}{*}{\rotatebox{90}{$d=25$}} & AAIS & \best{0.5069\,{\scriptsize$\pm$0.006}\,{\scriptsize(1.08)}} & \best{0.1953\,{\scriptsize$\pm$0.006}\,{\scriptsize(0.003)}} & \best{0.2672\,{\scriptsize$\pm$0.007}\,{\scriptsize(0.413)}} \\
         & DM-PMC & 1.423\,{\scriptsize(2.22)} & 0.0004\,{\scriptsize(0.000)} & \second{4.277\,{\scriptsize(7.06)}} \\
         & M-PMC & 2.060\,{\scriptsize(2.85)} & 0.0003\,{\scriptsize(0.000)} & 10.626\,{\scriptsize(13.29)} \\
         & GAPIS & \second{1.125\,{\scriptsize(1.68)}} & \second{0.0013\,{\scriptsize(0.000)}} & 9.789\,{\scriptsize(12.11)} \\ \cmidrule(l){2-5} 
         & True samples & 0.4787\,{\scriptsize(0.941)} & 1.000\,{\scriptsize(1.00)} & 0.0000\,{\scriptsize(0.000)} \\ \midrule
        \multirow{5}{*}{\rotatebox{90}{$d=50$}} & AAIS & \best{0.5123\,{\scriptsize$\pm$0.002}\,{\scriptsize(1.11)}} & \best{0.1327\,{\scriptsize$\pm$0.003}\,{\scriptsize(0.002)}} & \best{0.2921\,{\scriptsize$\pm$0.003}\,{\scriptsize(0.437)}} \\
         & DM-PMC & 1.837\,{\scriptsize(2.65)} & 0.0001\,{\scriptsize(0.000)} & \second{18.597\,{\scriptsize(22.77)}} \\
         & M-PMC & 2.263\,{\scriptsize(3.07)} & 0.0001\,{\scriptsize(0.000)} & 22.685\,{\scriptsize(26.08)} \\
         & GAPIS & \second{1.070\,{\scriptsize(1.64)}} & \second{0.0006\,{\scriptsize(0.000)}} & 23.015\,{\scriptsize(26.58)} \\
         & True samples & 0.4830\,{\scriptsize(0.967)} & 1.000\,{\scriptsize(1.00)} & 0.0000\,{\scriptsize(0.000)} \\ \bottomrule
    \end{tabular}
  }
  \label{tab:results-funnel}
\end{table}

\paragraph{Bayesian Posteriors}
We show here the usefulness of AAIS for Bayesian inference problems by training on a family of many related posteriors. 
We consider the following hierarchical Bayesian inference problem:
\begin{align}
    \mu \sim \mathcal{N}(0, \sigma_\mu^2),\quad \log \tau \sim \mathcal{N}(0, \sigma_\tau^2),~ j=1,\dots, J,\\
    y_j \mid \theta_j \sim \mathcal{N}(\theta_j, \sigma_j^2), \quad \theta_j \mid \mu, \tau \sim \mathcal{N}(\mu, \tau^2).
\end{align}
\looseness=-1
The $\sigma^2_j$ are assumed to be known, and the task is to infer $\theta_{1:J}, \mu$, and $\tau$, so $d=J+2$. To form a family of distributions we randomly sample: observations $y_{1:J}$, the parameters of the data generating process $\theta_{1:J},\mu,\tau$, and the likelihood model parameters $\sigma_j^2$. \cref{tab:results-hierarchical} shows the usefulness of amortisation over families for this problem. Again, AAIS scales much better to the higher-dimensional problems than the baselines and is able to perform stable adaptation of the proposal. On this Bayesian hierarchical problem AAIS has amortised not only over different datasets, but also over different data generating processes and different models for the likelihood (different $\sigma_j^2)$. We also evaluate performance on a Bayesian linear regression (BLR) problem in \cref{tab:results-blr-full}, where AAIS performs similarly well.

\begin{table}[t]
  \centering
  \caption{Experimental results for the Bayesian hierarchical model $d=5,25$. Metrics reported follow the convention of \cref{tab:results-funnel}.}
  \vspace*{-0.5em}
  \resizebox{\linewidth}{!}{
    \begin{tabular}{@{}llccccccc}
      \toprule
         & Method $\downarrow$ Metric $\rightarrow$ & sliced $W_1^{\mathrm{prop}}$ $\downarrow$ & ESS $\uparrow$ & $\log Z$ err. $\downarrow$ \\ \midrule
        \multirow{5}{*}{\rotatebox{90}{$d=5$}} & AAIS & \best{0.8718\,{\scriptsize$\pm$0.029}\,{\scriptsize(1.50)}} & \best{0.0358\,{\scriptsize$\pm$0.003}\,{\scriptsize(0.002)}} & \best{0.5812\,{\scriptsize$\pm$0.007}\,{\scriptsize(0.972)}} \\
         & DM-PMC & \second{1.594\,{\scriptsize(3.16)}} & \second{0.0120\,{\scriptsize(0.001)}} & \second{0.8585\,{\scriptsize(1.69)}} \\
         & M-PMC & 2.005\,{\scriptsize(4.31)} & 0.0081\,{\scriptsize(0.000)} & 1.674\,{\scriptsize(3.81)} \\
         & GAPIS & 138.7\,{\scriptsize(543.3)} & 0.0016\,{\scriptsize(0.000)} & 2.449\,{\scriptsize(5.90)} \\ \cmidrule(l){2-5} 
         & True samples & 0.3458\,{\scriptsize(0.565)} & 1.000\,{\scriptsize(1.00)} & 0.0000\,{\scriptsize(0.000)} \\ \midrule
        \multirow{5}{*}{\rotatebox{90}{$d=25$}} & AAIS & \best{0.9877\,{\scriptsize$\pm$0.020}\,{\scriptsize(2.41)}} & \best{0.0125\,{\scriptsize$\pm$0.001}\,{\scriptsize(0.001)}} & \best{1.279\,{\scriptsize$\pm$0.022}\,{\scriptsize(2.19)}} \\
         & DM-PMC & \second{3.675\,{\scriptsize(6.99)}} & 0.0003\,{\scriptsize(0.000)} & 20.770\,{\scriptsize(33.99)} \\
         & M-PMC & 5.413\,{\scriptsize(9.27)} & 0.0001\,{\scriptsize(0.000)} & 48.089\,{\scriptsize(76.65)} \\
         & GAPIS & 429.0\,{\scriptsize($1.2\!\times\!10^{3}$)} & \second{0.0007\,{\scriptsize(0.000)}} & \second{19.314\,{\scriptsize(41.64)}} \\ \cmidrule(l){2-5} 
         & True samples & 0.3343\,{\scriptsize(0.598)} & 1.000\,{\scriptsize(1.00)} & 0.0000\,{\scriptsize(0.000)} \\ \bottomrule
    \end{tabular}
  }
  \label{tab:results-hierarchical}
\end{table}

\begin{table}[t]
  \centering
  \caption{Performance of AAIS amortised over a single target vs over a whole family, all targets are at $d=25$. Metrics reported follow the convention of \cref{tab:results-funnel}.}
  \vspace*{-0.5em}
  \resizebox{\linewidth}{!}{
    \begin{tabular}{@{}clccccc}
      \toprule
        Target Family & Amortised over & Evaluated on & $\log Z$ err. $\downarrow$ \\ \midrule
        \multirow{4}{*}{\rotatebox{0}{Funnel}} & Single target & \multirow{2}{*}{\rotatebox{0}{Single target}} & {$0.2656$\,{\scriptsize$\pm 0.0062$}\,{\scriptsize$(0.3187)$}} \\
         & Family &  & \best{$0.2647$\,{\scriptsize$\pm 0.0081$}\,{\scriptsize$(0.3222)$}} \\ \cmidrule(l){2-4} 
         & Single target & \multirow{2}{*}{\rotatebox{0}{Family}} & \best{$0.2661$\,{\scriptsize$\pm 0.0049$}\,{\scriptsize$(0.4080)$}} \\
         & Family &  & {$0.2672$\,{\scriptsize$\pm 0.0074$}\,{\scriptsize$(0.4131)$}} \\ \midrule
        \multirow{4}{*}{\rotatebox{0}{Banana}} & Single target & \multirow{2}{*}{\rotatebox{0}{Single target}} & \best{$1.096$\,{\scriptsize$\pm 0.0005$}\,{\scriptsize$(1.529)$}} \\
         & Family &  & {$1.500$\,{\scriptsize$\pm 0.0244$}\,{\scriptsize$(2.102)$}} \\ \cmidrule(l){2-4} 
         & Single target & \multirow{2}{*}{\rotatebox{0}{Family}} & {$2.563$\,{\scriptsize$\pm 0.0315$}\,{\scriptsize$(4.500)$}} \\
         & Family &  & \best{$2.127$\,{\scriptsize$\pm 0.0138$}\,{\scriptsize$(3.340)$}} \\ \midrule
        \multirow{4}{*}{\rotatebox{0}{Hierarchical}} & Single target & \multirow{2}{*}{\rotatebox{0}{Single target}} & \best{$0.9722$\,{\scriptsize$\pm 0.0194$}\,{\scriptsize$(1.208)$}} \\
         & Family &  & {$1.281$\,{\scriptsize$\pm 0.0208$}\,{\scriptsize$(1.572)$}} \\ \cmidrule(l){2-4} 
         & Single target & \multirow{2}{*}{\rotatebox{0}{Family}} & {$4.720$\,{\scriptsize$\pm 0.1137$}\,{\scriptsize$(15.600)$}} \\
         & Family &  & \best{$1.279$\,{\scriptsize$\pm 0.0214$}\,{\scriptsize$(2.196)$}} \\ \midrule
        \multirow{4}{*}{\rotatebox{0}{BLR}} & Single target & \multirow{2}{*}{\rotatebox{0}{Single target}} & \best{$293.7$\,{\scriptsize$\pm 5.751$}\,{\scriptsize$(335.1)$}} \\
         & Family &  & {$428.8$\,{\scriptsize$\pm 7.762$}\,{\scriptsize$(461.9)$}} \\ \cmidrule(l){2-4} 
         & Single target & \multirow{2}{*}{\rotatebox{0}{Family}} & {$605.2$\,{\scriptsize$\pm 9.940$}\,{\scriptsize$(1.1\!\times\!10^{3})$}} \\
         & Family &  & \best{$374.3$\,{\scriptsize$\pm 4.139$}\,{\scriptsize$(998.9)$}} \\ \bottomrule
    \end{tabular}
  }
  \label{tab:results-single-vs-family}
\end{table}
\begin{table*}[t]
  \caption{Generalisation of a trained AAIS algorithm to shifted targets of the same form, see family definitions in \cref{app:targets}. Sliced 1-Wasserstein distance of the final proposal to the target is reported and follows the convention of \cref{tab:results-funnel}. A ``-" indicates that the mean could not be calculated due to numerical divergence.}
  \vspace*{-0.5em}
  \resizebox{\linewidth}{!}{
    \begin{tabular}{@{}llcccccc}
      \toprule
      Base Target & Shift $\downarrow$ Method $\rightarrow$
      & AAIS & DM-PMC & M-PMC & GAPIS 
      \\
      \midrule
        \multirow{2}{*}{Banana ($d=5$)} & Curvature range: $\left[\frac{1}{10}, \frac{1}{2}\right] \to \left[\frac{1}{2}, 1\right]$ & \best{$1.365$\,{\scriptsize$\pm0.0344$}\,{\scriptsize$(2.459)$}} & $1.530$\,{\scriptsize($3.040$)} & $1.518$\,{\scriptsize($2.951$)} & \second{$1.387$\,{\scriptsize($1.996$)}} \\
         & Curvature range: $\left[\frac{1}{10}, \frac{1}{2}\right] \to [1,2]$ & \second{$2.912$\,{\scriptsize$\pm0.0520$}\,{\scriptsize$(5.382)$}} & $3.370$\,{\scriptsize($6.450$)} & $3.429$\,{\scriptsize($6.192$)} & \best{$2.278$\,{\scriptsize($4.172$)}} \\\midrule
        \multirow{4}{*}{BLR ($d=5$)} & $\beta^*\sim\mathcal{N}(0, 4I_r) \to \beta^*\sim\mathcal{N}(0, 64I_r)$ & \second{$3.980$\,{\scriptsize$\pm0.0428$}\,{\scriptsize$(7.859)$}} & \best{$1.323$\,{\scriptsize$(4.059)$}} & $5.108$\,{\scriptsize$(9.767)$} & $131.1$\,{\scriptsize$(611.8)$} \\
        & $\beta^*\sim\mathcal{N}(0, 4I_r) \to \beta^*\sim\mathcal{N}(0, 400I_r)$ & \second{$12.231$\,{\scriptsize$\pm0.0404$}\,{\scriptsize$(20.850)$}} & \best{6.793\,{\scriptsize(16.909)}} & 13.203\,{\scriptsize(22.528)} & 684.2\,{\scriptsize$(3.8\!\times\!10^{3})$} \\\cmidrule(lr){2-6}
        & $n$ range: $[3r, 12r] \to [12r, 48r]$ & \best{$0.7802$\,{\scriptsize$\pm0.0302$}\,{\scriptsize$(2.521)$}} & \second{$1.152$\,{\scriptsize$(2.197)$}} & $2.624$\,{\scriptsize$(4.347)$} & $132.9$\,{\scriptsize$(668.3)$} \\
        & $n$ range: $[3r, 12r] \to [48r, 192r]$ & \best{$0.8607$\,{\scriptsize$\pm0.0162$}\,{\scriptsize$(2.648)$}} & \second{$1.190$\,{\scriptsize$(2.101)$}} & $2.753$\,{\scriptsize$(4.494)$} & $515.2$\,{\scriptsize$(2.5\!\times\!10^{3})$} \\\midrule
        \multirow{2}{*}{Funnel ($d=5$)} & $x_1$ variance range: $\left[\frac{3}{2}, 3\right] \to \left[3, \frac{9}{2}\right]$ & \best{$3.280$\,{\scriptsize$\pm0.0184$}\,{\scriptsize$(7.076)$}} & $3.955$\,{\scriptsize$(7.728)$} & $3.970$\,{\scriptsize$(7.916)$} & \second{$3.927$\,{\scriptsize$(7.551)$}} \\
        & $x_1$ variance range: $\left[\frac{3}{2}, 3\right] \to \left[\frac{9}{2}, 6\right]$ & \second{$21.723$\,{\scriptsize$\pm0.0426$}\,{\scriptsize$(59.748)$}} & $22.256$\,{\scriptsize$(60.264)$} & $22.385$\,{\scriptsize$(60.376)$} & \best{$21.531$\,{\scriptsize$(58.947)$}} \\\midrule
        \multirow{2}{*}{Hierarchical ($d=5$)} & $\tau$ range: $\left[\frac{1}{2}, 8\right] \to \left[8, 20\right]$ & \best{$1.848$\,{\scriptsize$\pm0.1378$}\,{\scriptsize$(5.819)$}} & \second{$3.581$\,{\scriptsize$(8.298)$}} & $4.929$\,{\scriptsize$(12.708)$} & $9.2\!\times\!10^{12}$\,{\scriptsize$(5.0\!\times\!10^{13})$} \\ 
        & $\tau$ range: $\left[\frac{1}{2}, 8\right] \to \left[20, 32\right]$  & - & \best{$9.220$\,{\scriptsize$(19.125)$}} & \second{$12.100$\,{\scriptsize$(24.291)$}} & -\\\midrule
        \multirow{2}{*}{Banana ($d=25$)} & Curvature range: $\left[\frac{1}{10}, \frac{1}{2}\right] \to \left[\frac{1}{2}, 1\right]$ & \second{$2.493$\,{\scriptsize$\pm0.0234$}\,{\scriptsize$(3.211)$}} & $2.751$\,{\scriptsize$(3.606)$} & $3.340$\,{\scriptsize$(4.011)$} & \best{$1.109$\,{\scriptsize$(1.684)$}} \\
        & Curvature range: $\left[\frac{1}{10}, \frac{1}{2}\right] \to [1,2]$ & \second{$4.253$\,{\scriptsize$\pm0.0454$}\,{\scriptsize$(5.549)$}} & $5.114$\,{\scriptsize$(6.662)$} & $5.556$\,{\scriptsize$(7.076)$} & \best{$3.245$\,{\scriptsize$(4.806)$}} \\\midrule
        \multirow{4}{*}{BLR ($d=25$)} & $\beta^*\sim\mathcal{N}(0, 4I_r) \to \beta^*\sim\mathcal{N}(0, 64I_r)$ & \best{$3.366$\,{\scriptsize$\pm0.0550$}\,{\scriptsize$(5.645)$}} & \second{$4.164$\,{\scriptsize$(6.531)$}} & $6.825$\,{\scriptsize$(8.417)$} & $1.4\!\times\!10^{3}$\,{\scriptsize$(7.3\!\times\!10^{3})$} \\
        &  $\beta^*\sim\mathcal{N}(0, 4I_r) \to \beta^*\sim\mathcal{N}(0, 400I_r)$ & \best{$11.207$\,{\scriptsize$\pm0.2200$}\,{\scriptsize$(15.456)$}} & \second{$13.439$\,{\scriptsize$(17.379)$}} & $15.990$\,{\scriptsize$(19.547)$} & $7.3\!\times\!10^{3}$\,{\scriptsize$(3.9\!\times\!10^{4})$} \\\cmidrule(lr){2-6}
        & $n$ range: $[3r, 12r] \to [12r, 48r]$ & \best{$1.310$\,{\scriptsize$\pm0.0673$}\,{\scriptsize$(2.897)$}} & \second{$1.961$\,{\scriptsize$(3.197)$}} & $3.181$\,{\scriptsize(3.852)} & $1.1\!\times\!10^{3}$\,{\scriptsize($5.5\!\times\!10^{3}$)} \\
        & $n$ range: $[3r, 12r] \to [48r, 192r]$ & \best{$1.512$\,{\scriptsize$\pm0.0747$}\,{\scriptsize$(3.192)$}} & \second{$1.962$\,{\scriptsize$(3.212)$}} & $3.177$\,{\scriptsize$(3.954)$} & $5.2\!\times\!10^{3}$\,{\scriptsize$(3.0\!\times\!10^{4})$} \\\midrule
        \multirow{2}{*}{Funnel ($d=25$)} & $x_1$ variance range: $\left[\frac{3}{2}, 3\right] \to \left[3, \frac{9}{2}\right]$ & \best{$3.609$\,{\scriptsize$\pm0.0564$}\,{\scriptsize$(8.537)$}} & $4.650$\,{\scriptsize$(9.820)$} & $5.760$\,{\scriptsize$(10.868)$} & \second{$4.154$\,{\scriptsize$(8.906)$}} \\
        & $x_1$ variance range: $\left[\frac{3}{2}, 3\right] \to \left[\frac{9}{2}, 6\right]$ & \best{$24.502$\,{\scriptsize$\pm0.0988$}\,{\scriptsize$(65.907)$}} & $25.417$\,{\scriptsize$(66.595)$} & $26.834$\,{\scriptsize$(67.877)$} & \second{$24.580$\,{\scriptsize$(65.867)$}} \\\midrule
        \multirow{2}{*}{Hierarchical ($d=25$)} & $\tau$ range: $\left[\frac{1}{2}, 8\right] \to \left[8, 20\right]$ & \best{$3.767$\,{\scriptsize$\pm0.0882$}\,{\scriptsize$(7.176)$}} & \second{$10.288$\,{\scriptsize$(14.992)$}} & $12.936$\,{\scriptsize$(18.187)$} & $2.7\!\times\!10^{13}$\,{\scriptsize$(1.1\!\times\!10^{14})$} \\
        & $\tau$ range: $\left[\frac{1}{2}, 8\right] \to \left[20, 32\right]$ & - & \best{$19.281$\,{\scriptsize$(26.756)$}} & \second{$22.520$\,{\scriptsize$(30.694)$}} & - \\
      \bottomrule
    \end{tabular}
  }
  \label{tab:results-generalisation-shifted}
  \vspace*{-0.5em}
\end{table*}
\paragraph{Comparison to Single-Target Amortisation} To distinguish the performance gained from family-level amortisation over the performance gained from single-target amortisation, we show how AAIS trained over a family of distributions compares with AAIS trained on only a single distribution from the family. From \cref{tab:results-single-vs-family} we see that single-target AAIS generally performs better on its training target, as expected. More importantly, we see that family-amortised AAIS performs better when evaluated on other distributions within the family, particularly on the Bayesian inference problems. Targets from the funnel family are close enough together that amortisation over the family is not meaningfully different from single-target training.

\subsection{Generalisation to Unseen Targets} 
AAIS can be run on any target for evaluation, including targets not within its training family. Using the target families considered previously, we change the random parameter distributions defining each target family and obtain modified versions of each family. For example, we change the range which we sample $\sigma$ from in \eqref{eq:funnel} from $\left[\frac{3}{2}, 3\right]$ to $\left[3,\frac{9}{2}\right]$. In \cref{tab:results-generalisation-shifted} we use AAIS trained on the original targets and test them on the modified targets. Even though AAIS was not trained on the modified families, it still often does the best. AAIS generalises well to the smaller of two different modifications performed to both the funnel and Bayesian hierarchical families, but is unable to generalise to the larger modification. AAIS generalises particularly well to the modifications performed on the Bayesian linear regression family, including shifts to the random number of data-points supplied for posterior inference. Note that GAPIS was re-tuned on the modified families.

\subsection{Objective Choice}\label{sec:obj-choice}
In \cref{tab:results-abl-obj-loss} we compare different choices for divergence inside \eqref{eq:obj-implementation} instead of $\KL(q_t\mid\mid p)$. We compare using $\KL(p\mid\mid q_t)$ estimated either with  a SNIS-estimator or using true target samples (oracle draws). The $\KL(p \mid \mid q_t)$ estimated with true target samples achieves better performance on the two synthetic densities, but we were unable to achieve stable training of AAIS with either $\KL(p\mid\mid q_t)$ objective on the Bayesian hierarchical family. A particular challenge in the use of these objectives is that if the current proposal is far from the current target, then the importance weights can be extremely large, making numerical optimisation difficult. See \cref{app:obj-abl} for further details.

\begin{table}[t]
  \centering
  \caption{Effect of three different divergences in the AAIS objective. Metrics reported follow the convention of \cref{tab:results-funnel}. All targets are at $d=25$. A ``-" indicates that the mean could not be calculated due to numerical divergence.}
  \vspace*{-0.5em}
  \resizebox{\linewidth}{!}{
    \begin{tabular}{@{}llcccc}
      \toprule
        Target Family & Objective $\downarrow$ Metric $\rightarrow$ & $\log Z$ err. $\downarrow$ \\ \midrule
        \multirow{3}{*}{Funnel} & $\KL(q_t \mid \mid p)$ & \second{$0.2672$\,{\scriptsize$\pm 0.0074$}\,{\scriptsize$(0.4131)$}} \\
         & $\KL(p \mid \mid q_t)$ (SNIS) & $4.496$\,{\scriptsize$\pm 0.7189$}\,{\scriptsize$(6.445)$} \\
         & $\KL(p \mid \mid q_t)$ (oracle draws) & \best{$0.2145$\,{\scriptsize$\pm 0.0077$}\,{\scriptsize$(0.3352)$}} \\ \midrule
        \multirow{3}{*}{Banana} & $\KL(q_t \mid \mid p)$ & \second{$2.127$\,{\scriptsize$\pm 0.0138$}\,{\scriptsize$(3.341)$}} \\
         & $\KL(p \mid \mid q_t)$ (SNIS) & $6.705$\,{\scriptsize$\pm 0.6376$}\,{\scriptsize$(11.362)$} \\
         & $\KL(p \mid \mid q_t)$ (oracle draws) & \best{$1.130$\,{\scriptsize$\pm 0.0392$}\,{\scriptsize$(2.398)$}} \\ \midrule
        \multirow{3}{*}{Hierarchical} & $\KL(q_t \mid \mid p)$ & \best{$1.279$\,{\scriptsize$\pm 0.0215$}\,{\scriptsize$(2.194)$}} \\
         & $\KL(p \mid \mid q_t)$ (SNIS) & - \\
         & $\KL(p \mid \mid q_t)$ (oracle draws) & - \\ \midrule
        \multirow{3}{*}{BLR} & $\KL(q_t \mid \mid p)$ & \best{$374.4$\,{\scriptsize$\pm 4.126$}\,{\scriptsize$(994.0)$}} \\
         & $\KL(p \mid \mid q_t)$ (SNIS) & $1.6\!\times\!10^{3}$\,{\scriptsize$\pm 1.4\!\times\!10^{3}$}\,{\scriptsize$(5.0\!\times\!10^{3})$} \\
         & $\KL(p \mid \mid q_t)$ (oracle draws) & \second{$583.8$\,{\scriptsize$\pm 41.198$}\,{\scriptsize$(1.1\!\times\!10^{3})$}} \\ \bottomrule
    \end{tabular}
  }
  \label{tab:results-abl-obj-loss}
  \vspace*{-1em}
\end{table}

\subsection{Additional Experiments}

In \cref{app:additional-experiments} we provide various ablation studies, which justify our algorithm choices, such as the use of: target gradients (\cref{app:grad-abl}), expectation-maximisation-based features (\cref{app:em-abl}), diagonal covariance structure (\cref{app:covariance-abl}), and a step-wise objective (\cref{app:obj-abl}). We also show that equivariance gives better generalisation performance at the cost of some performance on the training family (\cref{app:equivariance-abl}). We explore the dependence on model capacity (\cref{app:capacity-abl}) and the generalisation of each trained model to targets from the other families (\cref{app:generalisation-other}). We also study the limits of scaling AAIS to multimodal distributions (\cref{tab:results-gmm5-full}).

Finally, we also give extra metrics for the results presented in the preceding subsections (\cref{app:additional-experiments-main}).

\section{Related Work}
\paragraph{AIS with Learnt Components} The idea to use adaptation rules based on the gradient of a divergence between proposal and target within an AIS algorithm has been explored previously, e.g., \citet{el2019variational, han2017stein, el2019stochastic, akyildiz2021convergence, perello2023adaptively} -- these methods learn in an online fashion rather than amortising. \citet{dowling2018improved} propose to use variational inference to learn an initialisation for AIS. The most similar approach to ours in the AIS literature is that of \citet{el2021policy} who use a reinforcement learning setup and, like AAIS, propose to learn the adaptation between iterations. Unlike our setup, information gained from the AIS generation process is detached from the updating procedure -- if the model is already trained, one could sample all of the parameters for every time step before running the algorithm. Under this formulation adaptation can never occur on a new target distribution, and thus amortisation over multiple targets is not possible.
\paragraph{Amortisation over Multiple Distributions} AAIS tackles the problem of sampling and integral estimation over many different distributions/tasks. Prior works in this direction include \citet{xi2018bayesian,chen2024conditional,sun2023vector, sun2023meta}, where the problem of jointly estimating multiple integrals is explored -- these works differ from our problem setup in that they usually assume access to true samples from the target when training. Recent progress in amortised sampling for molecular dynamics was made by \citet{klein2023timewarp, klein2024transferable}. Closely related is the area of amortised Bayesian inference, where amortisation is over datasets; such works include \citet{ritchie2016deep,paige2016inference,le2017auto}, and \citet[AMCI;][]{golinski2019amortized}, who set up a general framework for amortising the calculation of expectations. Many of these works, in a Bayesian inference setting, propose to learn a mapping from a dataset to the parameters of an importance sampling proposal. We adopt a similar framework, also through importance sampling, but we emphasise that our algorithm is able to solve a wider range of problems: these previous approaches require conditioning information that fully specifies the target distribution, whereas AAIS obtains information about the specific sampling target online during the sampling procedure.

\section{Conclusion}

We have shown that learning the adaptation rule in adaptive importance sampling by training over a family of distributions can give better sampling and estimation performance on that family. Unlike many amortisation approaches, AAIS does not require access to conditioning information about the target in advance of the sampling procedure. This is a useful freedom because we do not need to explicitly define the structure of the distribution family through a conditioning variable; instead, the structure can be implicit and inferred by observing target information through the training procedure. An extension of our AAIS algorithm could use conditioning information as an input to its adaptation rule. Alternatively, AAIS could be used to adapt the output proposal of another method that uses conditioning, such as that of \citet{golinski2019amortized}.

Our AAIS implementation did not explore some of the many choices that the setup in \cref{sec:aais} allows for, such as the possibility to make the adaptation non-Markovian (i.e., dependent on the entire rollout history), or to use a loss function that aims only for good estimation of a particular expectation \eqref{eq:expectation} rather than a well matched proposal.  AAIS can be extended to this setting by targeting the optimal distribution for a given $f$ (see e.g., \citet{mcbook}); the methods suggested by \citet{golinski2019amortized} can be transferred to AAIS.

Our method shows promising generalisation potential to distributions that lie outside the training family. This opens the possibility for training a sampler that can perform well on many distributions that are structurally related, but have not necessarily been part of the family trained on. Our experiments deliberately made use of small-scale neural networks with small parameter counts, so scaling may bring greater generalisation performance. Our study of generalisation performance also raises the question of what kind of structure in these distribution families is amortisable, and whether we can predict if a sampler trained on one family will perform well on another.

Whilst our algorithm focused on adaptive importance sampling, the principle of using information obtained at sampling time to infer target information available due to amortisation can be extended to other iterative Monte Carlo algorithms that intake target information at each step, such as Markov chain Monte Carlo, where we expect that similar gains could be made.

\subsection*{AI Use Statement}

Generative AI tools were lightly used to assist with coding tasks. They were not used for any other tasks, including writing and figure creation.

\subsubsection*{Acknowledgements}
The work of AC is supported by the UKRI EPSRC through the CDT in Machine Learning Systems hosted in 
the School of Informatics, University of Edinburgh (EP/Y03516X/1), as well as joint sponsorship from Level E Research. The work of VE and ESW is supported by the Advanced Research and Invention Agency (ARIA). ESW acknowledges support from the CIFAR Learning in Machines and Brains programme. 

\bibliography{references}

\clearpage
\appendix
\thispagestyle{empty}

\onecolumn
\aistatstitle{Amortised Adaptive Importance Sampling: \\
Supplementary Materials}

\section{Further AAIS details}

Here we give details of the AAIS algorithm (\cref{alg:amortised_ais}) and our implementation (\cref{alg:amortised_ais_implementation}). 

\begin{algorithm}
  \begin{algorithmic}[1]
    \caption{Amortised Adaptive Importance Sampling}
    \label{alg:amortised_ais}
    \Require $S$, $T$, adaptation function $u_\phi$, $N_{epochs}$, target family $\Pi$, proposal initialisation family $\mathcal{Q}$, divergence/loss $D$, batch size $B$.
    \For{$epoch = 1, \dots, N_{epochs}$}
    \State Draw a random batch of targets, $\pi_b \sim \Pi$, $b=1,\dots,B$
    \State Draw a random starting proposal $q_{b,1} \sim \mathcal{Q}$
    \For{$t=1,\dots, T$}
    \State Draw $x_{b,s}^{(t)} {\sim} q_{b,t}, ~s=1,\dots,S,~ b=1,\dots,B$.
    \State Calculate $w_{b,s}^{(t)}=w(x_{b,s}^{(t)}),~ s=1,\dots,S,~b=1,\dots,B$, using \eqref{eq:unbiased-IS}.
    \State Calculate observables $O_{b,t},~b=1,\dots,B$.
    \State Adapt $q_{b,t+1} = u_\phi(q_{b,1:t}, O_{b,1:t}), ~b=1,\dots,B$.
    \EndFor
    \State Draw and weight $x_{b,s}^{(T+1)} \sim q_{b,T+1}, ~s=1,\dots,S,~b=1,\dots,B$.
    \State Compute $D(q_{1:T+1}, \{(x_{b,n,k}^{(t)}, w_{b,s}^{(t)})\}_{s=1, t=1}^{S, T+1}, \pi))$ for $b=1,\dots,B$.
    \State Estimate $\nabla_\phi\mathbb{E}_{\pi \sim \Pi}[\mathbb{E}[D(q_{1:T+1}, \{(x_{b,n,k}^{(t)}, w_{b,s}^{(t)})\}_{s=1, t=1}^{S, T+1}, \pi)]]$ and update $\phi$.
    \EndFor
    \State \Return $u_\phi$
  \end{algorithmic}
\end{algorithm}

\begin{algorithm}
  \begin{algorithmic}[1]
    \caption{Our AAIS implementation with an equal-weight Gaussian mixture proposal. The adaptation function is written with $\pi$ as an input. This is used as notational convenience to mean that density evaluations and gradient evaluations of $\pi$ are used as input, see \cref{app:update-function}.}
    \label{alg:amortised_ais_implementation}
    \Require Number of components $N$, number of samples per component $K$, $T$, adaptation function $u_\phi$, $N_{epochs}$, target family $\Pi$, batch size $B$, temperature schedule $\kappa$
    \For{$epoch = 1, \dots, N_{epochs}$}
    \State Draw batch of random target unnormalised densities $\pi'_b \sim \Pi$,$~b=1,\dots,B$.
    \State Temper the density $\pi_b(x) = \pi'_b(x)^{\frac{1}{\kappa(epoch)}}$
    \State Draw random starting proposal parameters $\theta_{b,1} = \{\mu_{b,n}^{(1)}, L_{b,n}^{(1)}\}_{n=1}^N$ parametrising $q_{b,1}$.
    \For{$t=1,\dots, T$}
    \State $\epsilon_{b,n,k}\sim \mathcal{N}(0,I_d),~b=1,\dots,B,~n=1,\dots,N,~k=1,\dots,K$. 
    \State $x_{b,n,k}^{(t)} = \mu_{b,n}^{(t)} + L_{b,n}^{(t)} \epsilon_{b,n,k}$ 
    \State $w_{b,n,k}^{(t)} = \frac{\pi_b(x_{b,n,k}^{(t)})}{q_{b,t}(x_{b,n,k}^{(t)})}$
    \State Adapt $\theta_{b,t+1} = u_\phi(\theta_{b,t}, \{(x_{b,n,k}^{(t)}, w_{b,n,k}^{(t)})\}_{k=1, n=1}^{K,N}, \pi, \frac{t-1}{T-1})$.
    \EndFor
    \State $\epsilon_{b,n,k}\sim \mathcal{N}(0,I_d),~b=1,\dots,B,~n=1,\dots,N,~k=1,\dots,K$.
    \State $x_{b,n,k} = \mu_{b,n}^{(T+1)} + L_{b,n}^{(T+1)} \epsilon_{b,n,k}$ 
    \State Estimate $\nabla_{\phi} \mathbb{E}_{\pi \sim \Pi}\left[\frac{1}{T}\sum_{t=2}^{T+1} \KL (q_{t} \mid\mid p)\right]$ \eqref{eq:kl-estimator} using collected samples and perform a gradient step on $\phi$.
    \EndFor
    \State \Return $u_\phi$
  \end{algorithmic}
\end{algorithm}

\subsection{AAIS Adaptation Function Details}\label{app:update-function}
We give the details of our adaptation function for AAIS here. First we calculate the EM update statistics for Gaussian mixtures as these will be used in features for our AAIS adaptation function.

\subsubsection{EM Update Statistics}
Here we motivate the calculation of various features within our adaptation function by relating them to the statistics calculated to perform the expectation-maximisation update with a mixture of Gaussians. 

The M-step for a Gaussian mixture, $g$, with components $g_n, n=1,\dots,N$, fitted to the sample-weight pairs $\{(x_s, w_s)\}_{s=1}^S$, $w_s = w(x_s)$ from \eqref{eq:unbiased-IS}, as in \citet{cappe2008adaptive}, is given by the updates:
\begin{align*}
    \alpha_n^{EM} = \frac{\sum_s w_s \rho_{n,s}}{\sum_s w_s}, \quad \mu_n^{EM} = \frac{\sum_s w_s \rho_{n,s} x_s}{\sum_s \rho_{n,s} w_s}, \quad \Sigma_n^{EM} = \frac{\sum_s w_s \rho_{n,s}(x_s - \mu^{EM}_n)(x_s - \mu^{EM}_n)^T}{\sum_i w_s \rho_s},
\end{align*}
where
\begin{align*}
    \rho_{n,s} = \mathbbm{1}_{c(s)=n},
\end{align*}
where $c(i)$ denotes the index of the component that generated sample $x_i$. 

The M-PMC \citep{cappe2008adaptive} algorithm we use as a baseline makes use of the suggested Rao-Blackwellised version where instead 
$\rho_{n,s} = \frac{\alpha_n g_n(x_s)}{\sum_{j=1}^N\alpha_jg_j(x_s)}$
is the probability of $x_s$ being generated by proposal $n$, we also restrict the covariance to be diagonal -- in line with the proposal parametrisation used by AAIS.

\subsubsection{Full Adaptation Function Details} 
The adaptation step in our algorithm is detailed here, it consists of the calculation of various features/statistics, 3 MLP layers and a final linear output layer. Some inputs are clamped where required for stability, and the final adaptation is constrained in the amount it is allowed to change the current proposal by.

\paragraph{Per-particle Features}
First, each particle-weight pair $(x_s, w_s)$ is processed. Let $c(s)$ denote the index of the proposal component that generated $x_s$, we calculate features:

\begin{align*}
    z_s &= L_{c(s)}^{-1}(x_s - \mu_{c(s)})\\
    \bar w_s &= \log \frac{w_s}{\frac{1}{|\{j: c(j) = c(s)\}|}\sum_jw_j\mathbbm{1}_{(c(j) = c(s))}}\\
    g_s &= L_n^T\nabla\log\pi(x_s).
\end{align*}

These are then passed through an MLP, $\xi$, with: an input layer, a SiLU activation, and a single hidden-layer. Passing the features through $\xi: \mathbb{R}^{2d+1} \to \mathbb{R}^{H}$, where $H$ is the hidden-dimension, gives a per-particle encoding. These are pooled over each generating component:

\begin{align*}
    h_n = \frac{1}{|\{j: c(j) = n\}|}\sum_{s=1}^S \xi(z_s, \bar w_s, g_s)\mathbbm{1}_{c(s)=n}.
\end{align*}

\paragraph{Per-component Features}
The per-component features can mostly be split into $2$ main groups of features. The first gives the adaptation access to the current parameters of the proposal, however, to maintain scale-equivariance, the features are relative component-wise features with respect to the global features of the proposal, and where appropriate are rescaled using the proposal's pooled variance:
\begin{align}\label{eq:proposal-pooled-var}
    \zeta^2 = \frac{1}{Nd}\sum_{n=1}^N\sum_{j=1}^d \left((\mu_n)_j - (\bar\mu)_j\right)^2 + \frac{1}{Nd}\sum_{n=1}^N\sum_{j=1}^d \left(\Sigma_n\right)_{jj}, \quad \bar \mu = \frac{1}{N}\sum_{n=1}^N \mu_n.
\end{align}

The second group of features are aimed at providing the adaptation with access to the EM updates. We make use of the EM updates (the non-Rao-Blackwellised version) from the previous section:

\begin{align*}
    \nu_n = L_n^{-1}(\mu^{EM} - \mu), \quad C_n = L_n^{^{-1}}\Sigma^{EM}L_n^{-T}.
\end{align*}

Where, in the same style as the first group of features, we have expressed the mean and covariance EM-informed features in terms of a residual with respect to the current parameters.

We also pass in the pooled particle encodings from the previous section, the time index, and the gradients of the log-target evaluated at the current proposal means (this is the feature used by GAPIS \citep{elvira2015gradient} to perform its adaptation).

We summarise the component-level features in \cref{tab:component_features}. The correlation features are specific to the case of updating a full covariance on each step, in cases where only the diagonal is adapted these are omitted.

\begin{table*}[h]
  \centering
  \caption{Summary of component level features.}
    \begin{tabular}{ll}
        \toprule
        Feature & Dimension \\
        \midrule
        $h_n$ -- the pooled particle encodings. & $H$ \\
        $\frac{\left(\mu_n-\bar\mu\right)}{\zeta}$, position relative to the ensemble. & $d$ \\
        $\frac12\log\Sigma_{n,jj}-\log \zeta$, variance relative to the ensemble. & $d$ \\
        Correlations of $\Sigma_n$. & $d(d{-}1)/2$ \\
        $\log N m_n$ -- EM mass share. & $1$ \\
        $\log N\alpha_n$ -- Current mixture weight multiplied by $N$. & $1$ \\
        $\frac{t-1}{T-1}$ -- Step index, normalised to $[0,1]$. & $1$ \\
        $\nu_n$ -- scaled EM mean residual. & $d$ \\
        $\log \text{diag}(C_n)$ -- log of the residual EM covariance diagonal. & $d$ \\
        Correlations of $C_n$. & $d(d{-}1)/2$ \\
        $L_n^\top\nabla\log\pi(\mu_n)$ -- Scaled target gradient at the component mean. & $d$ \\
        \bottomrule
    \end{tabular}
  \label{tab:component_features}
\end{table*}

The features $\log Nm_n, \log N\alpha_n, \log\text{diag}(C_n),$ and $ L_n^T\nabla\log\pi(\mu_n)$ are clipped with the function $f(x) = \frac{1}{5}\text{clip}(x, \pm 20)$.

\paragraph{Proposal Parameter Updating}
The component features in the previous section, denoted $\text{feat}_n$ are passed through $2$ single-layer MLP's with SiLU activation functions -- $\psi\in\mathbb{R}^{d_{\rm comp}}\to \mathbb{R}^H, \psi_{\rm cross}\in \mathbb{R}^{2H}\to \mathbb{R}^H$ where $d_{\rm comp}$ is the sum of dimensions listed in \cref{tab:component_features} -- with a pooling step in order to mix information between components in the mixture, then they are passed through a linear layer with the appropriate output dimension depending on the number of parameters needed to output to adapt the covariance:

\begin{align*}
    z_n = \psi\left(\text{feat}_n\right), \quad z_n = \psi_{\rm cross}\left(z_n, \frac{1}{N}\sum_{n=1}^Nz_n\right), \quad \Delta_n = Wz_n + b.
\end{align*}

We split the output $\Delta_n  =(\Delta\mu'_n, \Delta L'_n)$, and apply a smooth clamping to the outputs to constrain the maximum adaptation to the proposal:
\begin{align}
    \Delta\mu_n = \Delta\mu_{\rm max} \tanh\left(\frac{\Delta\mu'_n}{\Delta\mu_{\rm max}}\right).
\end{align}
The means are adapted as 
\begin{align*}
    \mu_n \gets \mu_n + L_n\Delta\mu_n.
\end{align*}

To adapt the covariances $\Sigma_n$, we split $\Delta L'_n\in \mathbb{R}^{\frac{d(d+1)}{2}}$ into a $d$-dimensional component $\Delta\log\sigma'$, and (if the full covariance is used to parametrise the mixture) a matrix, $K_n'$, with lower triangular values given by the remaining $\frac{d(d-1)}{2}$ components. These adaptations are constrained as $\Delta\log\sigma = \frac{1}{2}\Delta L_{\rm max} \tanh\left(\frac{\Delta\log\sigma'}{\Delta L_{\rm max}}\right)$ and $(K_n)_{ij} = \Delta L_{\rm max} \tanh\left(\frac{(K'_n)_{ij}}{\Delta L_{\rm max}}\right)$. Then, adapt the current $L_n$ by:
\begin{align*}
    L_n \gets L_n A_n, \quad  A_n = \text{diag}(e^{\Delta \log \sigma}) + K_n.
\end{align*}

\paragraph{Parameter Count} The number of learnt parameters in this model depends on the form of the covariance. In our main experiments we used a diagonal covariance so the correlation features in \cref{tab:component_features} are omitted so $d_{\rm comp}=H + 5d + 3$. This gives a parameter count of $((2d+1)H + H + H^2 + H) + (d_{\rm comp}H + H) + (2H^2 + H) + (2dH + 2d)$, where we have used brackets to separate the parameter counts for $\xi, \psi, \psi_{\rm cross},$ and the final layer $(W,b)$. See \cref{tab:hyperparams} for the exact parameter count in our experiments.

\paragraph{Cost} The cost of the AAIS algorithm above in terms of target density evaluations, counting gradient evaluations as double cost (one forward pass + backward pass), is:
\begin{align*}
    C = 2N(K+1)(T+1).
\end{align*}
This is calculated as: $NK(T+1)$ evaluations of the density and its gradient at the samples $x_s$ with cost $2NK(T+1)$, added to the gradient at the component means $\mu_n$ giving a further $2N(T+1)$ cost.
\subsection{Scale Equivariance of AAIS}\label{app:equivariance}
Here we give the precise statement about the equivariance properties of our AAIS algorithm and give the brief proof. This is specific to our implementation of AAIS. 

\begin{lemma}\label{lem:equivariance}
    Let $\pi^*(x) = \pi\left(\frac{x - b}{a}\right)$ and $u_\phi$ be the AAIS adaptation function as in \cref{app:update-function}, then, if we denote the output parameters from the adaptation function as 
    \begin{align*}
        \{\mu_n^{new}\}_{n=1}^N, \{L_n^{new}\}_{n=1}^N = u_\phi\left(\{\mu_{n}^{old}, L_n^{old}\}_{n=1}^N,\{x_s\}_{s=1}^S, \pi, \frac{t-1}{T-1}\right),
    \end{align*}
    we have the following:
    \begin{align*}
        \{a\mu_n^{new} + b\}_{n=1}^N, \{aL_n^{new}\}_{n=1}^N=u_\phi\left(\{a\mu_{n}^{old}+b, aL_n^{old}\}_{n=1}^N,\{ax_s + b\}_{s=1}^S, \pi^*, \frac{t-1}{T-1}\right).
    \end{align*}
\end{lemma}
\begin{proof}
    The proof proceeds by showing the scale-invariance of the procedure which produces the $\Delta_n$. From there, the adaptation to $\mu_n$ is:
    \begin{align*}
        a\mu_n + b \gets a\mu_n + b + aL_n\Delta\mu_n = a(\mu_n + L_n\Delta\mu_n) + b,
    \end{align*}
    and the covariance adaptation is:
    \begin{align*}
        aL_n \gets aL_nA_n, \quad A_n =(\text{diag}(e^{\Delta\log\sigma^2}) + K_n).
    \end{align*}

    Starting with the per-particle features which are encoded, note first that the weights $w_s$ will be identical because of the way $\pi^*$ is defined, let $y_s = ax_s + b$ be the transformed particles given to the function, then:

    \begin{align*}
        z_s = \frac{1}{a}L^{-1}_n(y_{c(s)} - a\mu_{c(s)} - b) = \frac{1}{a} L_n^{-1}(ax_{c(s)} + b - a\mu_{c(s)}-b) = L_n^{-1}(x_{c(s)} - \mu_{c(s)}),\\
        g_s = aL_n^T\nabla_y\log\pi^*(y_s) = aL_n^T\nabla_y\log\pi\left(\frac{y_s-b}{a}\right) = aL_n^T\nabla_x\log\pi(x_s) \cdot \frac{1}{a} = L_n^T\nabla_x\log\pi(x_s).
    \end{align*}
    
    Proving the invariance for the component-level features follows the same pattern as above and so we skip it, only noting that:
    \begin{align*}
        \zeta^2 &= \frac{1}{Nd}\sum_{n=1}^N\sum_{j=1}^d \left((a\mu_n + b)_j - (a\bar\mu + b)_j\right)^2 + \frac{1}{Nd}\sum_{n=1}^N\sum_{j=1}^d \left(a^2\Sigma_n\right)_{jj}\\ &= a^2\left(\frac{1}{Nd}\sum_{n=1}^N\sum_{j=1}^d \left((\mu_n)_j - (\bar\mu)_j\right)^2 + \frac{1}{Nd}\sum_{n=1}^N\sum_{j=1}^d \left(\Sigma_n\right)_{jj}\right),
    \end{align*}
    so the pooled variance scales as we would expect it to.

    Thus all the features that lead into producing $\Delta_n$ are invariant under a linear transformation, so the statement follows using the argument given at the start.
\end{proof}

This has a useful corollary:

\begin{corollary}
    Let a given AAIS rollout's proposal parameters on each step be $\{\mu_n^{(t)}, L_n^{(t)}\}_{n=1, t=1}^{N, T+1}$ then, assuming that the target is changed such that the new target $\pi^*(x) = \pi(\frac{x-b}{a})$ and the same random noises are used throughout, using the same trained AAIS network parameters and starting the rollout at $\{a\mu_n^{(1)} + b, aL_n^{(1)}\}_{n=1}^N$ will give the sequence of proposal parameters $\{a\mu_n^{(t)} + b, aL_n^{(t)}\}_{n=1, t=1}^{N,T+1}$.
\end{corollary}
\begin{proof}
    The proof follows by repeated application of \cref{lem:equivariance} throughout the AAIS iterations.
\end{proof}

Informally, the usefulness of this corollary is that if a trained set of parameters $\phi$ works well for a particular problem, then it will work similarly for the scaled version of the problem. As an example, imagine that a set of parameters can sample a Gaussian with mean $\mu$ and variance $\sigma^2$ well, then it will be able to sample a Gaussian of mean $a\mu + b$ and variance $a^2\sigma^2$ in the same fashion as long as the starting proposal parameters are transformed accordingly.

\subsection{Non-equivariant Version} \label{app:non-scale-equivariant}
The non-equivariant version of AAIS, evaluated in \cref{tab:results-abl-equivariance,tab:results-generalisation-equivariance}, adds in absolute features into the adaptation, breaking equivariance but allowing for learning the scale of the target. The non-equivariant adaptation function takes the same form as that described in \cref{app:update-function} with additional component-level features $\mu_n, \log \text{diag}(L_n), \log s$ -- the component mean, log of the diagonal of the scale matrix, and log of the pooled proposal standard deviation \eqref{eq:proposal-pooled-var}. The output is also given an extra $d$ values $\Delta\mu_n^{abs}$ and the mean adaptation becomes $\mu_n \gets \mu_n + L_n\Delta\mu_n + l_\phi\Delta\mu_n^{abs}$ where $l_\phi \in \mathbb{R}$ is a learnt scale for the adaptation.

\section{Experimental Details}

Here we give details of our experimental setup. All experiments were performed on a cluster using NVIDIA L40S GPU's, per training run these took approximately $3$ minutes for the lower dimensional training runs up to $15$ minutes for the $d=50$ runs. Our models are lightweight enough to run on CPU, although the batching over many targets during training is significantly sped up through GPU use.

\subsection{Training Details}\label{app:training-details}
Our objective function requires estimation of $\KL(q_t \mid \mid p)$. In order to maintain stable training when $q$ and $\pi$ are mismatched, we use compression with $\arcsin$ on the log-density ratio, the estimator we use is:
\begin{align}\label{eq:kl-estimator}
    \widehat \KL(q_t \mid\mid p) = \frac{1}{NK}\sum_{n=1}^{N}\sum_{k=1}^K C \arcsin\left(\frac{g(x_{n,k})}{C}\right)+\text{const}, \quad g(x_{n,k}) = \log q_t(x_{n,k}) - \log \pi(x_{n,k}).
\end{align} 
$C=30$ was chosen so that the compression has a negligible effect for ratios within a suitable range and only compresses extreme values.

The Adam optimiser \citep{kingma2014adam} is used with a cosine scheduled learning rate starting at $3\times 10^{-3}$ and with minimum learning rate $(3\times 10^{-3}) \cdot \frac{1}{20}$.

Annealing was performed through training to aid stability with a temperature schedule, as in \cref{alg:amortised_ais_implementation}, $\kappa(epoch) = 0.3 + (1-0.3)\cdot \min(1, \frac{epoch}{0.1\cdot N_{epochs}})$.

After training, the final model is chosen by selecting the model at the epoch with the lowest loss on a held out set of targets from the family.

\subsection{Target Densities}\label{app:targets}

Here we detail the different choices for target families $\Pi$. Both Bayesian posteriors are based on similar posteriors from \citet{gelman1995bayesian}.

\paragraph{Funnel} A family consisting of Neal's funnel \citep{neal2000slice} in $d$-dimensions, a single target has the form:
\begin{align*}
    x_1 \sim \mathcal{N}(0, \sigma^2), \qquad x_{i} \mid x_1 \sim \mathcal{N}(0, e^{x_1}), \quad i=2,\dots,d.
\end{align*}
The family of distributions obtained by sampling $\sigma \sim \mathcal{U}\left[\frac{3}{2}, 3\right]$.

\paragraph{Banana} A product of independent 2-dimensional banana distributions as in \citet{haario1999adaptive}. An individual block is given by:
\begin{align*}
    x_1 \sim \mathcal{N}(0, \sigma_1^2), \qquad x_2 \sim \mathcal{N}(b(x_1^2 - \sigma_1^2), \sigma_2^2),
\end{align*}
where $\sigma_1^2, \sigma_2^2$ and $b$ are the variances and the curvature. The samples are then rotated by some angle $\theta$. The parameters for every block are sampled as $b\sim \mathcal{U}\left[\frac{1}{10}, \frac{1}{5}\right], \sigma_1\sim \mathcal{U}\left[1, 3\right], \sigma_2 \sim \mathcal{U}\left[\frac{1}{2}, \frac{3}{2}\right], \theta\sim\mathcal{U}[0,2\pi]$, and, if $d$ is even, an individual target is formed by composing $\frac{d}{2}$ blocks, if $d$ is odd then $\frac{d-1}{2}$ blocks are composed along with a final independent Gaussian $\mathcal{N}(0, \sigma_1^2)$.

\paragraph{Bayesian Linear Regression Posterior}
This is a Bayesian linear regression model. Let $r$ be the number of parameters of the linear model, define a design matrix $X\in \mathbb{R}^{n_{\rm data}\times r}$, parameter vector $\beta \in \mathbb{R}^r$, variance $\sigma^2\in \mathbb{R}$, and the hyperparameters $a,b\in \mathbb{R}$. Let $y\in \mathbb{R}^{n_{\rm data}}$ be the response, the model is:
\begin{align*}
    y \mid \beta \sim \mathcal{N}\left(X\beta, \sigma^2I_{n_{\rm data}}\right), \qquad \beta \mid \sigma^2 \sim \mathcal{N}\left(0, \sigma^2V_0^{-1}\right), \quad \sigma^2 \sim \text{Inv-Gamma}(a_0, b_0), 
\end{align*}
where inference is over the $\beta$ and $\sigma^2$ so the dimension of the posterior is $d=p+1$.

This has posterior:
\begin{align*}
    \pi(\beta, \sigma^2) = \mathcal{N}(\beta; \mu, \sigma^2 V^{-1})&\cdot \text{Inv-Gamma}(\sigma^2;a^*,b^*), \\
    V = X^TX + V_0, \quad \mu = V^{-1}X^Ty, \quad a^* &= a_0 + \frac{{n_{\rm data}}}{2}, \quad b^* = b_0 + \frac{1}{2}(y^Ty - \mu^T V\mu), 
\end{align*}
which can be tractably sampled from by sampling $\sigma^2 \sim \text{Inv-Gamma}(a^*, b^*)$ followed by $\beta \mid \sigma^2 \sim \mathcal{N}(\mu, \sigma^2 V^{-1})$. To ensure every sampling method produces valid $\sigma^2$ values, each method produces samples of $\log \sigma^2$ instead and we adjust the above posterior accordingly. The above also has tractable expected values for the parameters to be sampled, which can be used instead of approximations for evaluation.

We take $V_0 = \frac{1}{10}I_r, a_0 = b_0 = 1, \mu_0 = 0$. Only the data and the parameters of the data generating process differ between targets in the family. We give the details below. 

The data generating process randomly constructs the design matrix $X$ such that the family has a range of different difficulties within it, we do this by randomly sampling $X$ with a given condition number $\kappa$ (approximately) and with a random rotation:
\begin{align}
    X_i \sim \mathcal{N}\left(0, \Sigma\right), i=1,\dots, r, \qquad \Sigma = QDQ^T,
\end{align}
where $Q$ is a random orthogonal $r\times r$ matrix formed by taking the QR decomposition of a $r\times r$ matrix of i.i.d. standard Gaussian random variables, and $D$ is the diagonal matrix with elements $s_i^2$ where $\log s_i = \frac{i-1}{r}\log \kappa$ -- this gives the smallest and largest singular values of $\Sigma$ as $1$ and $\kappa^2$ respectively, so the condition number of $\Sigma$ is exactly $\kappa^2$, for each target we sample $\log_{10}\kappa\sim \mathcal{U}[0,2]$ so $\kappa\in[1,100]$.

After generating $X$, the data is produced as:
\begin{align}
    y_i \sim \mathcal{N}\left(X\beta^*, \left(\sigma^*\right)^2\right), i=1,\dots,{n_{\rm data}},
\end{align}
where remaining parameters are sampled as $\beta^* \sim \mathcal{N}(0, 4I_r)$ and $\sigma^*\sim \mathcal{U}\left[\frac{3}{10}, \frac{3}{2}\right]$. We also randomise the number of data points ${n_{\rm data}}\sim \mathcal{U}\left[3r, 12r\right]$.

\paragraph{Hierarchical Bayesian Posterior} This is a hierarchical Bayesian model. In $d$-dimensions, let $J=d-2$, then the model is given by:
\begin{align*}
    y_j \mid \theta_j \sim \mathcal{N}(\theta_j, \sigma^2_j), \quad \theta_j \mid \mu, \tau \sim \mathcal{N}(\mu, \tau^2), \quad \mu \sim \mathcal{N}(0, \sigma_\mu^2), \quad \log \tau \sim \mathcal{N}(0, \sigma^2_\tau), 
\end{align*}
where all values belong to $\mathbb{R}$, the $\sigma^2_j$ are known, and inference is over the $\theta_{1:J}, \mu, $ and $\tau$. 

This has posterior density:
\begin{align*}
    \pi(\theta_1,\dots,\theta_J, \mu, \tau) = \mathcal{N}(\mu, \sigma_\mu^2)\mathcal{N}(\log\tau, \sigma_\tau^2)\Pi_{j=1}^J\left(\mathcal{N}(y_j; \theta_j, \sigma_j^2)\mathcal{N}(\theta_j, \mu, \tau^2)\right).
\end{align*}

To differ members within a family of such distributions, we use different sets of generated data with different parameters for the data generating process: we use $\mu^*, \tau^*, \theta_j^*$ to indicate the true parameters used in the data generation for a particular target, then the data generating process is $\mu^* \sim \mathcal{N}(0, 25), \tau^* \sim \mathcal{U}\left[\frac{1}{2}, 8\right], \theta_j^* \sim \mathcal{N}(\mu^*, (\tau^*)^2), y_j \sim \mathcal{N}(\theta_j^*, \sigma_j^2)$, we also randomise the known $\sigma_j^2$ as $\sigma^2_j \sim \mathcal{U}\left[2,18\right]$, so each target not only varies in the data $y$ supplied, but also in the likelihood model itself.

In order to obtain ground truth samples to use as reference for evaluations, we see that, for a given $(\mu, \tau)$, we can simply sample $\theta_j$ as 
\begin{align*}
    \theta_j \mid y, \mu, \tau \sim \mathcal{N}(a,b), \quad a= \mu + \frac{\tau^2 }{\sigma_j^2 + \tau^2}(y - \mu), \quad b= \frac{\sigma_j^2\tau^2}{\sigma^2_j + \tau^2}.
\end{align*}

In order to sample $(\mu, \tau)$, note first that we can marginalise out the $\theta_j$ to obtain:
\begin{align*}
    \pi(\mu, \tau) = \mathcal{N}(\mu, \sigma_\mu^2)\mathcal{N}(\log\tau, \sigma_\tau^2) \Pi_{j=1}^J\mathcal{N}(y_j, \mu, \sigma_j^2 + \tau^2),
\end{align*}
which we then sample from using the inverse-CDF approximated by quadrature over a sufficiently fine grid.

\subsection{Metrics}\label{app:metrics}
Here we give details on the metrics reported in the tables in both the main text and those in \cref{app:additional-experiments}.

\paragraph{Sliced 1-Wasserstein Distance} \citep{bonneel2015sliced} This is a tractable surrogate for the 1-Wasserstein distance $\mathcal{W}_1(q_{T+1}, \pi)$:
\begin{align}
    \mathcal{W}_1(q_{T+1}, p) &= \inf_{\gamma \in \Gamma(q_{T+1}, p)} \int |x - y| d\gamma(x,y)\\
    \text{SWD}_1(q_{T+1}, p) &= \mathbb{E}_{\omega \sim \mathcal{U}(\mathbb{S}^{d-1})}[\mathcal{W}_1((\tau_w)_{\#}q_{T+1}, (\tau_\omega)_{\#}p)],
\end{align}
where $\Gamma(q_{T+1}, p)$ is the set of probability measures with marginals $q_{T+1}$ and $p$, $\mathbb{S}^{d-1} = \{\omega \in \mathbb{R}^d : ||\omega|| = 1 \}$, $\tau_w(x) = \omega^Tx$ is the projection of $x\in \mathbb{R}^d$ onto $\omega$, and $(\tau_\omega)_{\#}p$ is the pushforward of $p$ by $\tau_\omega$. We compute the Monte Carlo estimate of this by sampling projections $\omega$, and we calculate the distance between the empirical distributions of $q_{T+1}$ and $p$ formed by $300$ samples from each. This metric is denoted ``sliced $\mathcal{W}_1^{\rm prop}$" in our tables.

\paragraph{Effective Sample Size}
The effective sample size for an AIS output with weights $\{w(x_{s}^{(t)})\}_{s=1,t=1}^{S,T+1}$ is 
\begin{align}
    \text{ESS} = \frac{(\sum_t\sum_{s}w(x_{s}^{(t)}))^2}{\sum_t\sum_{s}w(x_{s}^{(t)})^2}, 
\end{align}
and we report the normalised version $\frac{1}{S(T+1)}\text{ESS} \in [0,1]$.

\paragraph{Error in Log-Estimate of Normalising Constant} We compute the estimate of $Z$:
\begin{align}
    \widehat Z = \frac{1}{S(T+1)}\sum_{t=1}^{T+1}\sum_{s=1}^S w(x_s^{(t)}),
\end{align}
and then report $|\log \widehat Z - \log Z|$.

\paragraph{Sliced 1-Wasserstein Distance with Resampled Points} This is the same calculation as the previous sliced 1-Wasserstein distance, except that the empirical distribution for $q_{T+1}$ is not formed from new samples from the final proposal, but instead by sampling with replacement from the final $S$ samples on the last step of the AIS procedure proportional to their importance weights. 

\paragraph{KL divergence} We compute the estimated $\KL(q_{T+1} \mid \mid p)$ as
\begin{align}
    \widehat \KL(q_{T+1} \mid \mid p) = \sum_{s=1}^{300}(\log q_{T+1}(x_s) - \log \pi(x_s)), 
\end{align}
where we draw $300$ samples $x_s\sim q_{T+1}$.

\paragraph{Expectation Estimation Errors} We compute the estimate of $\mathbb{E}[f(x)]$:

\begin{align}
    \widehat{\mathbb{E}}[f(x)] = \frac{1}{\sum_{t=1}^{T+1}\sum_{s=1}^{S}w(x_s^{(t)})}\sum_{t=1}^{T+1}\sum_{s=1}^{S}w(x_s^{(t)})f(x_s^{(t)}), 
\end{align}
and then take the absolute error between this and $\mathbb{E}[f(x)]$ estimated by $20000$ draws from $\pi$.

\subsection{Hyperparameters}\label{app:hyperparamters}
\cref{tab:hyperparams} gives the hyperparameters used in the experiments. We vary the value of $K$ in order to maintain a matched cost for each sampler. The GAPIS step size $\lambda$ was chosen by selecting the best value (in terms of sliced 1-Wasserstein distance of the final proposal to the target) from a grid of 11 different values. The $\Delta\mu_{\rm max}$ parameter for AAIS that controls the maximum amount the proposal means can shift between time steps (see \cref{app:update-function}) was increased for the Bayesian linear regression problem. Other, fixed, parameters are the number of training epochs $N_{epochs}=1500$, and the target batch size used during training $B=32$.

\begin{table*}[h]
  \centering
  \caption{Hyperparameter settings for experiments. The final three columns are specific to AAIS, the Num. params column was calculated using the formula from \cref{app:update-function}.}
  \resizebox{1\linewidth}{!}{
    \begin{tabular}{@{}lccccccccccc}
      \toprule 
        Target Family & $d$ & $N$ & $T$ & $K$ (AAIS) & $K$ (DM-PMC) & $K$ (M-PMC) & $K$ (GAPIS)  & $\lambda$ (GAPIS) & $\Delta\mu_{\max}$ & H & Num. params.\\
        \midrule
        BLR d5 & 5 & 10 & 15 & 15 & 32 & 32 & 25  & 1e-05 & 3.0 & 32 & 5802 \\
        Hierarchical d5 & 5 & 10 & 15 & 15 & 32 & 32 & 25  & 1e-05 & 1.0 & 32 & 5802 \\
        Funnel d5 & 5 & 10 & 15 & 15 & 32 & 32 & 25  & 0.01 & 1.0 & 32 & 5802 \\
        Banana d5 & 5 & 10 & 15 & 15 & 32 & 32 & 25  & 0.03 & 1.0 & 32 & 5802 \\
        BLR d25 & 25 & 25 & 15 & 15 & 32 & 32 & 6  & 1e-05 & 3.0 & 64 & 31346 \\
        Hierarchical d25 & 25 & 25 & 15 & 15 & 32 & 32 & 6  & 1e-05 & 1.0 & 64 & 31346  \\
        Funnel d25 & 25 & 25 & 15 & 15 & 32 & 32 & 6  & 0.0003 & 1.0 & 64 & 31346  \\
        Banana d25 & 25 & 25 & 15 & 15 & 32 & 32 & 6  & 0.03 & 1.0 & 64 & 31346  \\
        Funnel d50 & 50 & 50 & 15 & 25 & 53 & 53 & 5  & 0.0001 & 1.0 & 128 & 124260 \\
        Banana d50 & 50 & 50 & 15 & 25 & 53 & 53 & 5  & 0.03 & 1.0 & 128 & 124260 \\
      \bottomrule
    \end{tabular}
  }
  \label{tab:hyperparams}
\end{table*}

\section{Additional Experiments}\label{app:additional-experiments}
Here we give the results of various ablation experiments performed, as well as extra metrics for the experiments presented in the main text. 

\subsection{Target Gradient Ablation} \label{app:grad-abl}
Here we show the effect of removing the gradient information from the features of the update function. \cref{tab:results-abl-gradient} shows that removing the gradient features from the model is detrimental, particularly on the Bayesian linear regression problem.

\begin{table*}[ht]
  \centering
  \caption{Experimental results showing the effect of removing target gradients as features in the adaptation function (\cref{app:update-function}) note that this still uses target gradients when performing gradient descent in training. Metrics reported follow the convention of \cref{tab:results-funnel}. All targets are at $d=25$.}
  \vspace*{-0.5em}
  \resizebox{\linewidth}{!}{
    \begin{tabular}{@{}clcccc}
      \toprule
      Target Family &Method $\downarrow$ Metric $\rightarrow$
      & sliced $W_1^{\mathrm{prop}}$ $\downarrow$  & rev.\ KL $\downarrow$  & $\log Z$ err. $\downarrow$ & $\mathbb{E}[x]$ err. $\downarrow$  
      \\
      \midrule
        \multirow{2}{*}{BLR} &
        With gradient features & \best{$1.174$\,{\scriptsize$\pm 0.0333$}\,{\scriptsize$(2.600)$}}  & \best{$409.9$\,{\scriptsize$\pm 3.154$}\,{\scriptsize$(1.0\!\times\!10^{3})$}}  & \best{$374.4$\,{\scriptsize$\pm 4.126$}\,{\scriptsize$(994.0)$}} & \best{$6.832$\,{\scriptsize$\pm 0.2652$}\,{\scriptsize$(15.645)$}}  \\
        &No gradient features & {$5.300$\,{\scriptsize$\pm 1.431$}\,{\scriptsize$(6.316)$}}  & {$1.1\!\times\!10^{3}$\,{\scriptsize$\pm 157.7$}\,{\scriptsize$(1.8\!\times\!10^{3})$}}  & {$790.1$\,{\scriptsize$\pm 22.984$}\,{\scriptsize$(1.4\!\times\!10^{3})$}} & {$19.601$\,{\scriptsize$\pm 1.304$}\,{\scriptsize$(25.856)$}}  \\
        \cmidrule(l){1-6}
        \multirow{2}{*}{Hierarchical}
        &With gradient features & \best{$0.9877$\,{\scriptsize$\pm 0.0201$}\,{\scriptsize$(2.414)$}}  & \best{$2.157$\,{\scriptsize$\pm 0.0160$}\,{\scriptsize$(3.137)$}}  & \best{$1.279$\,{\scriptsize$\pm 0.0215$}\,{\scriptsize$(2.194)$}} & \best{$4.531$\,{\scriptsize$\pm 0.0769$}\,{\scriptsize$(12.805)$}}  \\
        &No gradient features & {$1.151$\,{\scriptsize$\pm 0.0501$}\,{\scriptsize$(2.963)$}}  & {$2.515$\,{\scriptsize$\pm 0.0460$}\,{\scriptsize$(4.160)$}}  & {$1.426$\,{\scriptsize$\pm 0.0569$}\,{\scriptsize$(2.583)$}} & {$5.134$\,{\scriptsize$\pm 0.0410$}\,{\scriptsize$(15.701)$}}  \\
        \cmidrule(l){1-6}
        \multirow{2}{*}{Funnel}
        &With gradient features & \best{$0.5069$\,{\scriptsize$\pm 0.0061$}\,{\scriptsize$(1.082)$}}  & \best{$0.1535$\,{\scriptsize$\pm 0.0032$}\,{\scriptsize$(0.2379)$}}  & \best{$0.2672$\,{\scriptsize$\pm 0.0074$}\,{\scriptsize$(0.4131)$}} & {$0.6550$\,{\scriptsize$\pm 0.1238$}\,{\scriptsize$(2.128)$}}  \\
        &No gradient features & {$0.5200$\,{\scriptsize$\pm 0.0059$}\,{\scriptsize$(1.095)$}}  & {$0.2003$\,{\scriptsize$\pm 0.0055$}\,{\scriptsize$(0.3083)$}}  & {$0.2784$\,{\scriptsize$\pm 0.0061$}\,{\scriptsize$(0.4373)$}} & \best{$0.5170$\,{\scriptsize$\pm 0.0218$}\,{\scriptsize$(1.560)$}}  \\
        \cmidrule(l){1-6}
        \multirow{2}{*}{Banana}
        &With gradient features & {$1.093$\,{\scriptsize$\pm 0.0069$}\,{\scriptsize$(1.492)$}}  & \best{$6.112$\,{\scriptsize$\pm 0.0565$}\,{\scriptsize$(7.823)$}}  & \best{$2.127$\,{\scriptsize$\pm 0.0138$}\,{\scriptsize$(3.341)$}} & \best{$4.295$\,{\scriptsize$\pm 0.0227$}\,{\scriptsize$(6.514)$}}  \\
        &No gradient features & \best{$1.079$\,{\scriptsize$\pm 0.0040$}\,{\scriptsize$(1.548)$}}  & {$7.048$\,{\scriptsize$\pm 0.0381$}\,{\scriptsize$(9.012)$}}  & {$2.207$\,{\scriptsize$\pm 0.0190$}\,{\scriptsize$(4.196)$}} & {$4.396$\,{\scriptsize$\pm 0.0934$}\,{\scriptsize$(6.846)$}}  \\    
      \bottomrule
    \end{tabular}
  }
  \label{tab:results-abl-gradient}
\end{table*}

\subsection{EM Features Ablation} \label{app:em-abl}
Here we show the effect of removing the EM update features from \cref{app:update-function}. \cref{tab:results-abl-em} shows that there is no single best choice on all $d=25$ families, but EM features perform well on the hierarchical and funnel families. We opted to use them on all families to reduce the amount of per-family tuning that occurs. 

\begin{table*}[ht]
  \centering
  \caption{Experimental results showing the effect of the EM features in the adaptation model. Metrics reported follow the convention of \cref{tab:results-funnel}. All targets are at $d=25$.}
  \vspace*{-0.5em}
  \resizebox{\linewidth}{!}{
    \begin{tabular}{@{}clcccc}
      \toprule
      Target Family &Method $\downarrow$ Metric $\rightarrow$
      & sliced $W_1^{\mathrm{prop}}$ $\downarrow$  & rev.\ KL $\downarrow$  & $\log Z$ err. $\downarrow$ & $\mathbb{E}[x]$ err. $\downarrow$  
      \\
      \midrule
        \multirow{2}{*}{BLR} &
        With EM features & {$1.174$\,{\scriptsize$\pm 0.0333$}\,{\scriptsize$(2.600)$}} & {$409.9$\,{\scriptsize$\pm 3.154$}\,{\scriptsize$(1.0\!\times\!10^{3})$}} & {$374.4$\,{\scriptsize$\pm 4.126$}\,{\scriptsize$(994.0)$}} & {$6.832$\,{\scriptsize$\pm 0.2652$}\,{\scriptsize$(15.645)$}} \\
        &No EM features & \best{$0.9059$\,{\scriptsize$\pm 0.0132$}\,{\scriptsize$(1.908)$}} & \best{$357.5$\,{\scriptsize$\pm 5.075$}\,{\scriptsize$(960.2)$}} & \best{$335.0$\,{\scriptsize$\pm 4.293$}\,{\scriptsize$(916.3)$}} & \best{$5.496$\,{\scriptsize$\pm 0.0735$}\,{\scriptsize$(11.569)$}} \\
        \cmidrule(l){1-6}
        \multirow{2}{*}{Hierarchical} &
        With EM features & \best{$0.9877$\,{\scriptsize$\pm 0.0201$}\,{\scriptsize$(2.414)$}} & \best{$2.157$\,{\scriptsize$\pm 0.0160$}\,{\scriptsize$(3.137)$}} & \best{$1.279$\,{\scriptsize$\pm 0.0215$}\,{\scriptsize$(2.194)$}} & \best{$4.531$\,{\scriptsize$\pm 0.0769$}\,{\scriptsize$(12.805)$}} \\
        &No EM features & {$1.255$\,{\scriptsize$\pm 0.0103$}\,{\scriptsize$(3.175)$}} & {$2.683$\,{\scriptsize$\pm 0.0571$}\,{\scriptsize$(4.568)$}} & {$1.512$\,{\scriptsize$\pm 0.0402$}\,{\scriptsize$(2.838)$}} & {$5.239$\,{\scriptsize$\pm 0.0698$}\,{\scriptsize$(13.861)$}} \\
        \cmidrule(l){1-6}
        \multirow{2}{*}{Funnel} &
        With EM features & \best{$0.5069$\,{\scriptsize$\pm 0.0061$}\,{\scriptsize$(1.082)$}} & \best{$0.1535$\,{\scriptsize$\pm 0.0032$}\,{\scriptsize$(0.2379)$}} & \best{$0.2672$\,{\scriptsize$\pm 0.0074$}\,{\scriptsize$(0.4131)$}} & \best{$0.6550$\,{\scriptsize$\pm 0.1238$}\,{\scriptsize$(2.128)$}} \\
        &No EM features & {$0.5486$\,{\scriptsize$\pm 0.0096$}\,{\scriptsize$(1.218)$}} & {$0.2895$\,{\scriptsize$\pm 0.0444$}\,{\scriptsize$(0.4191)$}} & {$0.2871$\,{\scriptsize$\pm 0.0020$}\,{\scriptsize$(0.4446)$}} & {$0.6850$\,{\scriptsize$\pm 0.0509$}\,{\scriptsize$(2.014)$}} \\
        \cmidrule(l){1-6}
        \multirow{2}{*}{Banana} &
        With EM features & {$1.093$\,{\scriptsize$\pm 0.0069$}\,{\scriptsize$(1.492)$}} & {$6.112$\,{\scriptsize$\pm 0.0565$}\,{\scriptsize$(7.823)$}} & {$2.127$\,{\scriptsize$\pm 0.0138$}\,{\scriptsize$(3.341)$}} & {$4.295$\,{\scriptsize$\pm 0.0227$}\,{\scriptsize$(6.514)$}} \\
        &No EM features & \best{$1.081$\,{\scriptsize$\pm 0.0001$}\,{\scriptsize$(1.458)$}} & \best{$5.974$\,{\scriptsize$\pm 0.0550$}\,{\scriptsize$(7.592)$}} & \best{$2.057$\,{\scriptsize$\pm 0.0554$}\,{\scriptsize$(3.302)$}} & \best{$4.286$\,{\scriptsize$\pm 0.0425$}\,{\scriptsize$(6.671)$}} \\
    
      \bottomrule
    \end{tabular}
  }
  \label{tab:results-abl-em}
\end{table*}

\subsection{Objective Choice} \label{app:obj-abl}
Here we give results for using differing objectives. We give the full results for the use different stochastic divergences $D$ (\cref{tab:results-abl-obj-loss-full}), and also investigate the effect of removing the step-wise objective and of removing the pathwise gradient through the rollout.

\cref{tab:results-abl-obj-traj} shows the effect of removing one of either the step-wise objective or the gradients propagating through the whole trajectory. There is no one choice that is uniformly best, but using both seems to give the most consistent results.

\begin{table*}[ht]
  \centering
  \caption{Experimental results showing the effect of three different objective structures, a terminal only objective that uses pathwise gradients that flow through the whole trajectory, and a step-wise objective with no gradients flowing between steps in the trajectory, full model has step-wise objectives and also gradients flowing through the whole trajectory. Metrics reported follow the convention of \cref{tab:results-funnel}. All targets are at $d=25$.}
  \vspace*{-0.5em}
  \resizebox{\linewidth}{!}{
    \begin{tabular}{@{}clcccc}
      \toprule
      Target Family &Method $\downarrow$ Metric $\rightarrow$
      & sliced $W_1^{\mathrm{prop}}$ $\downarrow$  & rev.\ KL $\downarrow$  & $\log Z$ err. $\downarrow$ & $\mathbb{E}[x]$ err. $\downarrow$  
      \\
      \midrule
        \multirow{3}{*}{BLR} &
        Full model & \second{$1.174$\,{\scriptsize$\pm 0.0333$}\,{\scriptsize$(2.600)$}} & \second{$409.9$\,{\scriptsize$\pm 3.154$}\,{\scriptsize$(1.0\!\times\!10^{3})$}} & \second{$374.4$\,{\scriptsize$\pm 4.126$}\,{\scriptsize$(994.0)$}} & \second{$6.832$\,{\scriptsize$\pm 0.2652$}\,{\scriptsize$(15.645)$}} \\
        &Terminal objective only & $2.706$\,{\scriptsize$\pm 0.6514$}\,{\scriptsize$(3.839)$} & $717.5$\,{\scriptsize$\pm 101.8$}\,{\scriptsize$(1.3\!\times\!10^{3})$} & $604.1$\,{\scriptsize$\pm 97.197$}\,{\scriptsize$(1.1\!\times\!10^{3})$} & $12.982$\,{\scriptsize$\pm 3.012$}\,{\scriptsize$(20.797)$} \\
        &Step-wise & \best{$0.8079$\,{\scriptsize$\pm 0.0069$}\,{\scriptsize$(1.701)$}} & \best{$320.2$\,{\scriptsize$\pm 2.461$}\,{\scriptsize$(871.4)$}} & \best{$310.9$\,{\scriptsize$\pm 2.342$}\,{\scriptsize$(858.9)$}} & \best{$5.005$\,{\scriptsize$\pm 0.0412$}\,{\scriptsize$(10.651)$}} \\
        \cmidrule(l){1-6}
        \multirow{3}{*}{Hierarchical} &
        Full model & \best{$0.9877$\,{\scriptsize$\pm 0.0201$}\,{\scriptsize$(2.414)$}} & \best{$2.157$\,{\scriptsize$\pm 0.0160$}\,{\scriptsize$(3.137)$}} & \best{$1.279$\,{\scriptsize$\pm 0.0215$}\,{\scriptsize$(2.194)$}} & \best{$4.531$\,{\scriptsize$\pm 0.0769$}\,{\scriptsize$(12.805)$}} \\
        &Terminal objective only & $2.184$\,{\scriptsize$\pm 0.1925$}\,{\scriptsize$(3.297)$} & $7.913$\,{\scriptsize$\pm 0.6277$}\,{\scriptsize$(11.317)$} & $2.727$\,{\scriptsize$\pm 0.2046$}\,{\scriptsize$(4.947)$} & $10.343$\,{\scriptsize$\pm 0.5920$}\,{\scriptsize$(18.960)$} \\
        &Step-wise & \second{$1.227$\,{\scriptsize$\pm 0.0190$}\,{\scriptsize$(2.463)$}} & \second{$2.377$\,{\scriptsize$\pm 0.0102$}\,{\scriptsize$(4.292)$}} & \second{$1.404$\,{\scriptsize$\pm 0.0219$}\,{\scriptsize$(2.434)$}} & \second{$4.881$\,{\scriptsize$\pm 0.1296$}\,{\scriptsize$(10.874)$}} \\
        \cmidrule(l){1-6}
        \multirow{3}{*}{Funnel} &
        Full model & \best{$0.5069$\,{\scriptsize$\pm 0.0061$}\,{\scriptsize$(1.082)$}} & \best{$0.1535$\,{\scriptsize$\pm 0.0032$}\,{\scriptsize$(0.2379)$}} & \best{$0.2672$\,{\scriptsize$\pm 0.0074$}\,{\scriptsize$(0.4131)$}} & \second{$0.6550$\,{\scriptsize$\pm 0.1238$}\,{\scriptsize$(2.128)$}} \\
        &Terminal objective only & \second{$0.5172$\,{\scriptsize$\pm 0.0042$}\,{\scriptsize$(1.109)$}} & \second{$0.1971$\,{\scriptsize$\pm 0.0094$}\,{\scriptsize$(0.3104)$}} & $0.4416$\,{\scriptsize$\pm 0.0188$}\,{\scriptsize$(0.6401)$} & $0.9632$\,{\scriptsize$\pm 0.1450$}\,{\scriptsize$(2.622)$} \\
        &Step-wise & $0.5615$\,{\scriptsize$\pm 0.0638$}\,{\scriptsize$(1.201)$} & $0.2207$\,{\scriptsize$\pm 0.0622$}\,{\scriptsize$(0.3148)$} & \second{$0.2858$\,{\scriptsize$\pm 0.0114$}\,{\scriptsize$(0.4576)$}} & \best{$0.6370$\,{\scriptsize$\pm 0.0193$}\,{\scriptsize$(1.863)$}} \\
        \cmidrule(l){1-6}
        \multirow{3}{*}{Banana} &
        Full model & \second{$1.093$\,{\scriptsize$\pm 0.0069$}\,{\scriptsize$(1.492)$}} & \second{$6.112$\,{\scriptsize$\pm 0.0565$}\,{\scriptsize$(7.823)$}} & \second{$2.127$\,{\scriptsize$\pm 0.0138$}\,{\scriptsize$(3.341)$}} & \second{$4.295$\,{\scriptsize$\pm 0.0227$}\,{\scriptsize$(6.514)$}} \\
        &Terminal objective only & \best{$0.9906$\,{\scriptsize$\pm 0.0079$}\,{\scriptsize$(1.379)$}} & \best{$5.396$\,{\scriptsize$\pm 0.1244$}\,{\scriptsize$(7.116)$}} & \best{$1.543$\,{\scriptsize$\pm 0.0322$}\,{\scriptsize$(2.646)$}} & \best{$4.135$\,{\scriptsize$\pm 0.0219$}\,{\scriptsize$(6.599)$}} \\
        &Step-wise & $1.126$\,{\scriptsize$\pm 0.0014$}\,{\scriptsize$(1.571)$} & $6.778$\,{\scriptsize$\pm 0.0132$}\,{\scriptsize$(8.601)$} & $2.427$\,{\scriptsize$\pm 0.0287$}\,{\scriptsize$(3.934)$} & $4.431$\,{\scriptsize$\pm 0.0575$}\,{\scriptsize$(6.524)$} \\
      \bottomrule
    \end{tabular}
  }
  \label{tab:results-abl-obj-traj}
\end{table*}

\begin{table*}[ht]
  \centering
  \caption{Experimental results showing the effect of three different divergences in the objective. Metrics reported follow the convention of \cref{tab:results-funnel}. All targets are at $d=25$. A ``-" indicates that the mean could not be calculated due to a numerical divergence.}
  \vspace*{-0.5em}
  \resizebox{\linewidth}{!}{
    \begin{tabular}{@{}clcccc}
      \toprule
      Target Family &Method $\downarrow$ Metric $\rightarrow$
      & sliced $W_1^{\mathrm{prop}}$ $\downarrow$  & rev.\ KL $\downarrow$  & $\log Z$ err. $\downarrow$ & $\mathbb{E}[x]$ err. $\downarrow$  
      \\
      \midrule
        \multirow{3}{*}{BLR} &
        $\KL(q_t \mid \mid p)$ & \best{$1.174$\,{\scriptsize$\pm 0.0333$}\,{\scriptsize$(2.600)$}} & \best{$409.9$\,{\scriptsize$\pm 3.154$}\,{\scriptsize$(1.0\!\times\!10^{3})$}} & \best{$374.4$\,{\scriptsize$\pm 4.126$}\,{\scriptsize$(994.0)$}} & \best{$6.832$\,{\scriptsize$\pm 0.2652$}\,{\scriptsize$(15.645)$}} \\
        &$\KL(p \mid \mid q_t)$ (SNIS) & $118.2$\,{\scriptsize$\pm 94.077$}\,{\scriptsize$(128.3)$} & $1.3\!\times\!10^{14}$\,{\scriptsize$\pm 1.8\!\times\!10^{14}$}\,{\scriptsize$(2.6\!\times\!10^{35})$} & $1.6\!\times\!10^{3}$\,{\scriptsize$\pm 1.4\!\times\!10^{3}$}\,{\scriptsize$(5.0\!\times\!10^{3})$} & $21.066$\,{\scriptsize$\pm 1.232$}\,{\scriptsize$(27.944)$} \\
        &$\KL(p \mid \mid q_t)$ (oracle draws) & \second{$2.344$\,{\scriptsize$\pm 0.4549$}\,{\scriptsize$(3.144)$}} & $1.5\!\times\!10^{19}$\,{\scriptsize$\pm 2.7\!\times\!10^{19}$}\,{\scriptsize$(1.1\!\times\!10^{14})$} & \second{$583.8$\,{\scriptsize$\pm 41.198$}\,{\scriptsize$(1.1\!\times\!10^{3})$}} & \second{$11.520$\,{\scriptsize$\pm 1.502$}\,{\scriptsize$(18.700)$}} \\
        \cmidrule(l){1-6}
        \multirow{3}{*}{Hierarchical} &
        $\KL(q_t \mid \mid p)$ & \best{$0.9877$\,{\scriptsize$\pm 0.0201$}\,{\scriptsize$(2.414)$}} & \best{$2.157$\,{\scriptsize$\pm 0.0160$}\,{\scriptsize$(3.137)$}} & \best{$1.279$\,{\scriptsize$\pm 0.0215$}\,{\scriptsize$(2.194)$}} & \best{$4.531$\,{\scriptsize$\pm 0.0769$}\,{\scriptsize$(12.805)$}} \\
        &$\KL(p \mid \mid q_t)$ (SNIS) & - & - & - & - \\
        &$\KL(p \mid \mid q_t)$ (oracle draws) & - & - & - & - \\
        \cmidrule(l){1-6}
        \multirow{3}{*}{Funnel} &
        $\KL(q_t \mid \mid p)$ & \best{$0.5069$\,{\scriptsize$\pm 0.0061$}\,{\scriptsize$(1.082)$}} & \best{$0.1535$\,{\scriptsize$\pm 0.0032$}\,{\scriptsize$(0.2379)$}} & \second{$0.2672$\,{\scriptsize$\pm 0.0074$}\,{\scriptsize$(0.4131)$}} & \second{$0.6550$\,{\scriptsize$\pm 0.1238$}\,{\scriptsize$(2.128)$}} \\
        &$\KL(p \mid \mid q_t)$ (SNIS) & $58.695$\,{\scriptsize$\pm 4.989$}\,{\scriptsize$(64.920)$} & $1.6\!\times\!10^{3}$\,{\scriptsize$\pm 1.3\!\times\!10^{3}$}\,{\scriptsize$(3.0\!\times\!10^{3})$} & $4.496$\,{\scriptsize$\pm 0.7189$}\,{\scriptsize$(6.445)$} & $13.183$\,{\scriptsize$\pm 0.8694$}\,{\scriptsize$(28.047)$} \\
        &$\KL(p \mid \mid q_t)$ (oracle draws) & \second{$1.225$\,{\scriptsize$\pm 0.3780$}\,{\scriptsize$(2.005)$}} & $1.0\!\times\!10^{11}$\,{\scriptsize$\pm 1.8\!\times\!10^{11}$}\,{\scriptsize$(8.1\!\times\!10^{3})$} & \best{$0.2145$\,{\scriptsize$\pm 0.0077$}\,{\scriptsize$(0.3352)$}} & \best{$0.6345$\,{\scriptsize$\pm 0.0102$}\,{\scriptsize$(1.753)$}} \\
        \cmidrule(l){1-6}
        \multirow{3}{*}{Banana} &
        $\KL(q_t \mid \mid p)$ & $1.093$\,{\scriptsize$\pm 0.0069$}\,{\scriptsize$(1.492)$} & \best{$6.112$\,{\scriptsize$\pm 0.0565$}\,{\scriptsize$(7.823)$}} & \second{$2.127$\,{\scriptsize$\pm 0.0138$}\,{\scriptsize$(3.341)$}} & \best{$4.295$\,{\scriptsize$\pm 0.0227$}\,{\scriptsize$(6.514)$}} \\
        &$\KL(p \mid \mid q_t)$ (SNIS) & \second{$0.8083$\,{\scriptsize$\pm 0.0209$}\,{\scriptsize$(0.9314)$}} & \second{$94.821$\,{\scriptsize$\pm 4.839$}\,{\scriptsize$(125.1)$}} & $6.705$\,{\scriptsize$\pm 0.6376$}\,{\scriptsize$(11.362)$} & $7.443$\,{\scriptsize$\pm 0.0841$}\,{\scriptsize$(10.325)$} \\
        &$\KL(p \mid \mid q_t)$ (oracle draws) & \best{$0.5676$\,{\scriptsize$\pm 0.0151$}\,{\scriptsize$(0.7589)$}} & $408.9$\,{\scriptsize$\pm 88.470$}\,{\scriptsize$(1.1\!\times\!10^{3})$} & \best{$1.130$\,{\scriptsize$\pm 0.0392$}\,{\scriptsize$(2.398)$}} & \second{$4.552$\,{\scriptsize$\pm 0.0809$}\,{\scriptsize$(8.082)$}} \\
    
      \bottomrule
    \end{tabular}
  }
  \label{tab:results-abl-obj-loss-full}
\end{table*}

\subsection{Covariance Structure}\label{app:covariance-abl}
Here we compare using an isotropic variance, a diagonal variance (as used in our experiments), and a full covariance matrix. \cref{tab:results-abl-covariance} shows the effect of the different possible covariance parametrisations. We found the diagonal covariance to be the most stable in our experiments.

\begin{table*}[ht]
  \centering
  \caption{Experimental results showing the effect of using different parametrisations for the proposal covariance. Metrics reported follow the convention of \cref{tab:results-funnel}. An $^*$ indicates that only $2$ out of the $3$ training seeds were stable. All targets are at $d=25$.}
  \vspace*{-0.5em}
  \resizebox{\linewidth}{!}{
    \begin{tabular}{@{}clcccc}
      \toprule
      Target Family &Method $\downarrow$ Metric $\rightarrow$
      & sliced $W_1^{\mathrm{prop}}$ $\downarrow$  & rev.\ KL $\downarrow$  & $\log Z$ err. $\downarrow$ & $\mathbb{E}[x]$ err. $\downarrow$  
      \\
      \midrule
        \multirow{3}{*}{BLR} &
        Diagonal & \best{$1.174$\,{\scriptsize$\pm 0.0333$}\,{\scriptsize$(2.600)$}} & \best{$409.9$\,{\scriptsize$\pm 3.154$}\,{\scriptsize$(1.0\!\times\!10^{3})$}} & \best{$374.4$\,{\scriptsize$\pm 4.126$}\,{\scriptsize$(994.0)$}} & \best{$6.832$\,{\scriptsize$\pm 0.2652$}\,{\scriptsize$(15.645)$}} \\
        &Isotropic & \second{$1.956$\,{\scriptsize$\pm 0.2392$}\,{\scriptsize$(2.927)$}} & \second{$615.5$\,{\scriptsize$\pm 45.231$}\,{\scriptsize$(1.1\!\times\!10^{3})$}} & \second{$495.5$\,{\scriptsize$\pm 40.974$}\,{\scriptsize$(1.0\!\times\!10^{3})$}} & \second{$9.346$\,{\scriptsize$\pm 1.027$}\,{\scriptsize$(16.318)$}} \\
        &Full & $262.0$\,{\scriptsize$\pm 270.3$}\,{\scriptsize$(288.0)$} & $1.4\!\times\!10^{5}$\,{\scriptsize$\pm 2.5\!\times\!10^{5}$}\,{\scriptsize$(3.8\!\times\!10^{5})$} & $827.1$\,{\scriptsize$\pm 34.365$}\,{\scriptsize$(1.5\!\times\!10^{3})$} & $21.699$\,{\scriptsize$\pm 1.606$}\,{\scriptsize$(29.616)$} \\
        \cmidrule(l){1-6}
        \multirow{3}{*}{Hierarchical} &
        Diagonal & \best{$0.9877$\,{\scriptsize$\pm 0.0201$}\,{\scriptsize$(2.414)$}} & \best{$2.157$\,{\scriptsize$\pm 0.0160$}\,{\scriptsize$(3.137)$}} & \best{$1.279$\,{\scriptsize$\pm 0.0215$}\,{\scriptsize$(2.194)$}} & \best{$4.531$\,{\scriptsize$\pm 0.0769$}\,{\scriptsize$(12.805)$}} \\
        &Isotropic$^*$ & \second{$2.710$\,{\scriptsize$\pm 0.0008$}\,{\scriptsize$(5.984)$}} & \second{$14.136$\,{\scriptsize$\pm 0.2302$}\,{\scriptsize$(22.366)$}} & \second{$5.371$\,{\scriptsize$\pm 0.0832$}\,{\scriptsize$(13.536)$}} & \second{$11.226$\,{\scriptsize$\pm 0.1195$}\,{\scriptsize$(27.474)$}} \\
        &Full & $8.314$\,{\scriptsize$\pm 11.169$}\,{\scriptsize$(12.147)$} & $157.2$\,{\scriptsize$\pm 256.7$}\,{\scriptsize$(328.0)$} & $11.193$\,{\scriptsize$\pm 14.756$}\,{\scriptsize$(17.452)$} & $16.326$\,{\scriptsize$\pm 14.053$}\,{\scriptsize$(29.943)$} \\
        \cmidrule(l){1-6}
        \multirow{3}{*}{Funnel} &
        Diagonal & \best{$0.5069$\,{\scriptsize$\pm 0.0061$}\,{\scriptsize$(1.082)$}} & \best{$0.1535$\,{\scriptsize$\pm 0.0032$}\,{\scriptsize$(0.2379)$}} & \best{$0.2672$\,{\scriptsize$\pm 0.0074$}\,{\scriptsize$(0.4131)$}} & \second{$0.6550$\,{\scriptsize$\pm 0.1238$}\,{\scriptsize$(2.128)$}} \\
        &Isotropic & $1.277$\,{\scriptsize$\pm 0.0069$}\,{\scriptsize$(2.059)$} & $1.564$\,{\scriptsize$\pm 0.0150$}\,{\scriptsize$(1.848)$} & $0.8495$\,{\scriptsize$\pm 0.0168$}\,{\scriptsize$(1.112)$} & $1.459$\,{\scriptsize$\pm 0.0009$}\,{\scriptsize$(2.102)$} \\
        &Full & \second{$0.5139$\,{\scriptsize$\pm 0.0027$}\,{\scriptsize$(1.108)$}} & \second{$0.2111$\,{\scriptsize$\pm 0.0107$}\,{\scriptsize$(0.3171)$}} & \second{$0.2779$\,{\scriptsize$\pm 0.0064$}\,{\scriptsize$(0.4398)$}} & \best{$0.6449$\,{\scriptsize$\pm 0.0436$}\,{\scriptsize$(1.574)$}} \\
        \cmidrule(l){1-6}
        \multirow{3}{*}{Banana} &
        Diagonal & \second{$1.093$\,{\scriptsize$\pm 0.0069$}\,{\scriptsize$(1.492)$}} & \second{$6.112$\,{\scriptsize$\pm 0.0565$}\,{\scriptsize$(7.823)$}} & \second{$2.127$\,{\scriptsize$\pm 0.0138$}\,{\scriptsize$(3.341)$}} & \best{$4.295$\,{\scriptsize$\pm 0.0227$}\,{\scriptsize$(6.514)$}} \\
        &Isotropic & $1.124$\,{\scriptsize$\pm 0.0019$}\,{\scriptsize$(1.554)$} & $6.739$\,{\scriptsize$\pm 0.0051$}\,{\scriptsize$(8.484)$} & $2.375$\,{\scriptsize$\pm 0.0196$}\,{\scriptsize$(3.770)$} & $4.378$\,{\scriptsize$\pm 0.0107$}\,{\scriptsize$(6.478)$} \\
        &Full$^*$ & \best{$1.074$\,{\scriptsize$\pm 0.0008$}\,{\scriptsize$(1.515)$}} & \best{$6.078$\,{\scriptsize$\pm 0.0086$}\,{\scriptsize$(7.864)$}} & \best{$2.090$\,{\scriptsize$\pm 0.0150$}\,{\scriptsize$(3.396)$}} & \second{$4.308$\,{\scriptsize$\pm 0.0083$}\,{\scriptsize$(6.639)$}} \\
      \bottomrule
    \end{tabular}
  }
  \label{tab:results-abl-covariance}
\end{table*}

\clearpage
\subsection{Equivariance}\label{app:equivariance-abl}
The scale-equivariance of our AAIS implementation gives the algorithm greater generalisation abilities, at the cost of some performance on the family it was trained on, as can be seen in \cref{tab:results-abl-equivariance,tab:results-generalisation-equivariance}, where we compare with a non-equivariant adaptation function (\cref{app:non-scale-equivariant}). We hypothesise this is because the non-equivariant sampler learns information about the scale that is useful for sampling within the family, but is detrimental when applied to a modified family.

\begin{table*}[ht]
  \centering
  \caption{Experimental results showing the effect of using a model that is not scale-equivariant \cref{app:non-scale-equivariant}. Metrics reported follow the convention of \cref{tab:results-funnel}. All targets are at $d=25$.}
  \vspace*{-0.5em}
  \resizebox{\linewidth}{!}{
    \begin{tabular}{@{}clcccc}
      \toprule
      Target Family &Method $\downarrow$ Metric $\rightarrow$
      & sliced $W_1^{\mathrm{prop}}$ $\downarrow$  & rev.\ KL $\downarrow$  & $\log Z$ err. $\downarrow$ & $\mathbb{E}[x]$ err. $\downarrow$  
      \\
      \midrule
        \multirow{2}{*}{BLR} &
        Scale-equivariant & \best{$1.174$\,{\scriptsize$\pm 0.0333$}\,{\scriptsize$(2.600)$}} & \best{$409.9$\,{\scriptsize$\pm 3.154$}\,{\scriptsize$(1.0\!\times\!10^{3})$}} & \best{$374.4$\,{\scriptsize$\pm 4.126$}\,{\scriptsize$(994.0)$}} & \best{$6.832$\,{\scriptsize$\pm 0.2652$}\,{\scriptsize$(15.645)$}} \\
        &Not scale-equivariant & {$1.461$\,{\scriptsize$\pm 0.0976$}\,{\scriptsize$(2.529)$}} & {$542.4$\,{\scriptsize$\pm 16.950$}\,{\scriptsize$(1.2\!\times\!10^{3})$}} & {$466.2$\,{\scriptsize$\pm 15.175$}\,{\scriptsize$(1.0\!\times\!10^{3})$}} & {$8.135$\,{\scriptsize$\pm 0.4538$}\,{\scriptsize$(14.913)$}} \\
        \cmidrule(l){1-6}
        \multirow{2}{*}{Hierarchical} &
        Scale-equivariant & \best{$0.9877$\,{\scriptsize$\pm 0.0201$}\,{\scriptsize$(2.414)$}} & \best{$2.157$\,{\scriptsize$\pm 0.0160$}\,{\scriptsize$(3.137)$}} & {$1.279$\,{\scriptsize$\pm 0.0215$}\,{\scriptsize$(2.194)$}} & {$4.531$\,{\scriptsize$\pm 0.0769$}\,{\scriptsize$(12.805)$}} \\
        &Not scale-equivariant & {$1.052$\,{\scriptsize$\pm 0.0421$}\,{\scriptsize$(3.274)$}} & {$2.192$\,{\scriptsize$\pm 0.1499$}\,{\scriptsize$(3.985)$}} & \best{$0.9876$\,{\scriptsize$\pm 0.0388$}\,{\scriptsize$(1.950)$}} & \best{$4.132$\,{\scriptsize$\pm 0.0887$}\,{\scriptsize$(14.853)$}} \\
        \cmidrule(l){1-6}
        \multirow{2}{*}{Funnel} &
        Scale-equivariant & \best{$0.5069$\,{\scriptsize$\pm 0.0061$}\,{\scriptsize$(1.082)$}} & {$0.1535$\,{\scriptsize$\pm 0.0032$}\,{\scriptsize$(0.2379)$}} & {$0.2672$\,{\scriptsize$\pm 0.0074$}\,{\scriptsize$(0.4131)$}} & {$0.6550$\,{\scriptsize$\pm 0.1238$}\,{\scriptsize$(2.128)$}} \\
        &Not scale-equivariant & {$0.5076$\,{\scriptsize$\pm 0.0036$}\,{\scriptsize$(1.105)$}} & \best{$0.1486$\,{\scriptsize$\pm 0.0018$}\,{\scriptsize$(0.2277)$}} & \best{$0.2304$\,{\scriptsize$\pm 0.0044$}\,{\scriptsize$(0.3747)$}} & \best{$0.5341$\,{\scriptsize$\pm 0.0807$}\,{\scriptsize$(1.515)$}} \\
        \cmidrule(l){1-6}
        \multirow{2}{*}{Banana}
        &Scale-equivariant & {$1.093$\,{\scriptsize$\pm 0.0069$}\,{\scriptsize$(1.492)$}} & {$6.112$\,{\scriptsize$\pm 0.0565$}\,{\scriptsize$(7.823)$}} & {$2.127$\,{\scriptsize$\pm 0.0138$}\,{\scriptsize$(3.341)$}} & \best{$4.295$\,{\scriptsize$\pm 0.0227$}\,{\scriptsize$(6.514)$}} \\
        &Not scale-equivariant & \best{$1.046$\,{\scriptsize$\pm 0.0004$}\,{\scriptsize$(1.444)$}} & \best{$5.869$\,{\scriptsize$\pm 0.0307$}\,{\scriptsize$(7.593)$}} & \best{$2.008$\,{\scriptsize$\pm 0.0154$}\,{\scriptsize$(3.319)$}} & {$4.386$\,{\scriptsize$\pm 0.0219$}\,{\scriptsize$(6.669)$}} \\
    
      \bottomrule
    \end{tabular}
  }
  \label{tab:results-abl-equivariance}
\end{table*}

\begin{table}[ht]
  \caption{Experimental results showing the generalisation of the equivariant AAIS and a non-equivariant AAIS \cref{app:non-scale-equivariant} to shifted targets of the same form. Metrics reported follow the convention of \cref{tab:results-funnel}.}
  \vspace*{-0.5em}
  \resizebox{\linewidth}{!}{
    \begin{tabular}{@{}llcccccc}
      \toprule
      Base Target & Shift $\downarrow$ Method $\rightarrow$
      & Equivariant & Non-equivariant \\
      \midrule
        \multirow{2}{*}{Banana ($d=5$)} & Curvature range: $\left[\frac{1}{10}, \frac{1}{2}\right] \to \left[\frac{1}{2}, 1\right]$ & \best{$1.365$\,{\scriptsize$\pm0.0344$}\,{\scriptsize$(2.459)$}} & $1.395$\,{\scriptsize$\pm0.0184$}\,{\scriptsize$(2.593)$} \\
        & Curvature range: $\left[\frac{1}{10}, \frac{1}{2}\right] \to [1,2]$ & \best{$2.912$\,{\scriptsize$\pm0.0520$}\,{\scriptsize$(5.382)$}} & $3.197$\,{\scriptsize$\pm0.0251$}\,{\scriptsize$(5.970)$} \\\midrule
        \multirow{4}{*}{BLR ($d=5$)} & $\beta^*\sim\mathcal{N}(0, 4I_p) \to \beta^*\sim\mathcal{N}(0, 64I_p)$ & \best{$3.980$\,{\scriptsize$\pm0.0428$}\,{\scriptsize$(7.859)$}} & $4.331$\,{\scriptsize$\pm0.0666$}\,{\scriptsize$(8.379)$} \\
        & $\beta^*\sim\mathcal{N}(0, 4I_p) \to \beta^*\sim\mathcal{N}(0, 400I_p)$ & \best{$12.231$\,{\scriptsize$\pm0.0404$}\,{\scriptsize$(20.850)$}} & $12.897$\,{\scriptsize$\pm0.1075$}\,{\scriptsize$(21.379)$} \\\cmidrule(lr){2-4}
        & $n$ range: $[3p, 12p] \to [12p, 48p]$ & \best{$0.7802$\,{\scriptsize$\pm0.0302$}\,{\scriptsize$(2.521)$}} & $0.8313$\,{\scriptsize$\pm0.0859$}\,{\scriptsize$(2.477)$} \\
        & $n$ range: $[3p, 12p] \to [48p, 192p]$ & \best{$0.8607$\,{\scriptsize$\pm0.0162$}\,{\scriptsize$(2.648)$}} & $1.086$\,{\scriptsize$\pm0.2062$}\,{\scriptsize$(2.863)$} \\\midrule
        \multirow{2}{*}{Funnel ($d=5$)} & $x_1$ variance range: $\left[\frac{3}{2}, 3\right] \to \left[3, \frac{9}{2}\right]$ & $3.280$\,{\scriptsize$\pm0.0184$}\,{\scriptsize$(7.076)$} & \best{$3.125$\,{\scriptsize$\pm0.0261$}\,{\scriptsize$(6.939)$}} \\
        & $x_1$ variance range: $\left[\frac{3}{2}, 3\right] \to \left[\frac{9}{2}, 6\right]$  & \best{$21.723$\,{\scriptsize$\pm0.0426$}\,{\scriptsize$(59.748)$}} & - \\\midrule
        \multirow{2}{*}{Hierachical ($d=5$)} & $\tau$ range: $\left[\frac{1}{2}, 8\right] \to \left[8, 20\right]$ & \best{$1.848$\,{\scriptsize$\pm0.1378$}\,{\scriptsize$(5.819)$}} & $2.494$\,{\scriptsize$\pm0.0916$}\,{\scriptsize$(9.985)$} \\
        & $\tau$ range: $\left[\frac{1}{2}, 8\right] \to \left[20, 32\right]$ & -\,{\scriptsize$(16.400)$} & - \\\midrule
        \multirow{2}{*}{Banana ($d=25$)} & Curvature range: $\left[\frac{1}{10}, \frac{1}{2}\right] \to \left[\frac{1}{2}, 1\right]$ & \best{$2.493$\,{\scriptsize$\pm0.0234$}\,{\scriptsize$(3.211)$}} & $2.526$\,{\scriptsize$\pm0.0078$}\,{\scriptsize$(3.323)$} \\
        & Curvature range: $\left[\frac{1}{10}, \frac{1}{2}\right] \to [1,2]$ & \best{$4.253$\,{\scriptsize$\pm0.0454$}\,{\scriptsize$(5.549)$}} & $4.876$\,{\scriptsize$\pm0.0158$}\,{\scriptsize$(6.439)$} \\\midrule
        \multirow{4}{*}{BLR ($d=25$)} & $\beta^*\sim\mathcal{N}(0, 4I_p) \to \beta^*\sim\mathcal{N}(0, 64I_p)$ & \best{$3.366$\,{\scriptsize$\pm0.0550$}\,{\scriptsize$(5.645)$}} & $5.233$\,{\scriptsize$\pm0.0538$}\,{\scriptsize$(6.910)$} \\
        &  $\beta^*\sim\mathcal{N}(0, 4I_p) \to \beta^*\sim\mathcal{N}(0, 400I_p)$ & \best{$11.207$\,{\scriptsize$\pm0.2200$}\,{\scriptsize$(15.456)$}} & $14.769$\,{\scriptsize$\pm0.1202$}\,{\scriptsize$(18.548)$} \\\cmidrule(lr){2-4}
        & $n$ range: $[3p, 12p] \to [12p, 48p]$ & \best{$1.310$\,{\scriptsize$\pm0.0673$}\,{\scriptsize$(2.897)$}} & $1.763$\,{\scriptsize$\pm0.1611$}\,{\scriptsize$(3.148)$} \\
        & $n$ range: $[3p, 12p] \to [48p, 192p]$ & \best{$1.512$\,{\scriptsize$\pm0.0747$}\,{\scriptsize$(3.192)$}} & $2.107$\,{\scriptsize$\pm0.1921$}\,{\scriptsize$(3.678)$} \\\midrule
        \multirow{2}{*}{Funnel ($d=25$)} & $x_1$ variance range: $\left[\frac{3}{2}, 3\right] \to \left[3, \frac{9}{2}\right]$ & $3.609$\,{\scriptsize$\pm0.0564$}\,{\scriptsize$(8.537)$} & \best{$3.603$\,{\scriptsize$\pm0.0229$}\,{\scriptsize$(8.592)$}} \\
        & $x_1$ variance range: $\left[\frac{3}{2}, 3\right] \to \left[\frac{9}{2}, 6\right]$ & \best{$24.502$\,{\scriptsize$\pm0.0988$}\,{\scriptsize$(65.907)$}} & - \\\midrule
        \multirow{2}{*}{Hierarchical ($d=25$)} & $\tau$ range: $\left[\frac{1}{2}, 8\right] \to \left[8, 20\right]$  & \best{$3.767$\,{\scriptsize$\pm0.0882$}\,{\scriptsize$(7.176)$}} & $6.486$\,{\scriptsize$\pm0.4556$}\,{\scriptsize$(14.981)$} \\
        & $\tau$ range: $\left[\frac{1}{2}, 8\right] \to \left[20, 32\right]$ & -\,{\scriptsize$(13.854)$} & - \\
      \bottomrule
    \end{tabular}
  }
  \label{tab:results-generalisation-equivariance}
\end{table}

\subsection{Neural Capacity} \label{app:capacity-abl}
Here we show the effect of network capacity by differing the hidden dimension of the neural networks used in the adaptation function. In \cref{tab:results-abl-capacity} we see generally that increasing the hidden dimension yields better results.

\begin{table*}[ht]
  \centering
  \caption{Experimental results showing the effect of increasing the hidden dimension used in the MLP layers of the adaptation function (\cref{app:update-function}). Metrics reported follow the convention of \cref{tab:results-funnel}. All targets are at $d=25$.}
  \vspace*{-0.5em}
  \resizebox{\linewidth}{!}{
    \begin{tabular}{@{}clcccc}
      \toprule
      Target Family &Hidden dim. $\downarrow$ Metric $\rightarrow$
      & sliced $W_1^{\mathrm{prop}}$ $\downarrow$  & rev.\ KL $\downarrow$  & $\log Z$ err. $\downarrow$ & $\mathbb{E}[x]$ err. $\downarrow$  
      \\
      \midrule
        \multirow{4}{*}{BLR} &
        $H=64$ & \best{$1.174$\,{\scriptsize$\pm 0.0333$}\,{\scriptsize$(2.600)$}} & \best{$409.9$\,{\scriptsize$\pm 3.154$}\,{\scriptsize$(1.0\!\times\!10^{3})$}} & \best{$374.4$\,{\scriptsize$\pm 4.126$}\,{\scriptsize$(994.0)$}} & \best{$6.832$\,{\scriptsize$\pm 0.2652$}\,{\scriptsize$(15.645)$}} \\
        &$H=16$ & $3.119$\,{\scriptsize$\pm 0.1201$}\,{\scriptsize$(3.784)$} & $758.6$\,{\scriptsize$\pm 16.843$}\,{\scriptsize$(1.3\!\times\!10^{3})$} & $644.2$\,{\scriptsize$\pm 16.823$}\,{\scriptsize$(1.2\!\times\!10^{3})$} & $14.417$\,{\scriptsize$\pm 0.8210$}\,{\scriptsize$(21.133)$} \\
        &$H=32$ & $1.808$\,{\scriptsize$\pm 0.1249$}\,{\scriptsize$(2.609)$} & $602.8$\,{\scriptsize$\pm 22.304$}\,{\scriptsize$(1.1\!\times\!10^{3})$} & $484.6$\,{\scriptsize$\pm 20.098$}\,{\scriptsize$(992.2)$} & \second{$8.506$\,{\scriptsize$\pm 0.5306$}\,{\scriptsize$(14.841)$}} \\
        &$H=128$ & \second{$1.687$\,{\scriptsize$\pm 0.1518$}\,{\scriptsize$(3.390)$}} & \second{$527.6$\,{\scriptsize$\pm 24.223$}\,{\scriptsize$(1.2\!\times\!10^{3})$}} & \second{$448.7$\,{\scriptsize$\pm 22.353$}\,{\scriptsize$(1.1\!\times\!10^{3})$}} & $8.904$\,{\scriptsize$\pm 0.9462$}\,{\scriptsize$(18.417)$} \\
        \cmidrule(l){1-6}
        \multirow{4}{*}{Hierarchical} &
        $H=64$ & \second{$0.9877$\,{\scriptsize$\pm 0.0201$}\,{\scriptsize$(2.414)$}} & \second{$2.157$\,{\scriptsize$\pm 0.0160$}\,{\scriptsize$(3.137)$}} & \second{$1.279$\,{\scriptsize$\pm 0.0215$}\,{\scriptsize$(2.194)$}} & \second{$4.531$\,{\scriptsize$\pm 0.0769$}\,{\scriptsize$(12.805)$}} \\
        &$H=16$ & $2.813$\,{\scriptsize$\pm 0.0464$}\,{\scriptsize$(3.642)$} & $13.661$\,{\scriptsize$\pm 0.6366$}\,{\scriptsize$(18.976)$} & $3.318$\,{\scriptsize$\pm 0.1224$}\,{\scriptsize$(5.651)$} & $11.879$\,{\scriptsize$\pm 0.2681$}\,{\scriptsize$(20.676)$} \\
        &$H=32$ & $1.300$\,{\scriptsize$\pm 0.0295$}\,{\scriptsize$(3.383)$} & $3.207$\,{\scriptsize$\pm 0.1227$}\,{\scriptsize$(5.951)$} & $1.651$\,{\scriptsize$\pm 0.0667$}\,{\scriptsize$(3.040)$} & $5.711$\,{\scriptsize$\pm 0.1024$}\,{\scriptsize$(15.312)$} \\
        &$H=128$ & \best{$0.7749$\,{\scriptsize$\pm 0.0127$}\,{\scriptsize$(1.673)$}} & \best{$1.687$\,{\scriptsize$\pm 0.0207$}\,{\scriptsize$(2.571)$}} & \best{$1.059$\,{\scriptsize$\pm 0.0069$}\,{\scriptsize$(1.824)$}} & \best{$3.764$\,{\scriptsize$\pm 0.0740$}\,{\scriptsize$(10.390)$}} \\
        \cmidrule(l){1-6}
        \multirow{4}{*}{Funnel} &
        $H=64$ & \second{$0.5069$\,{\scriptsize$\pm 0.0061$}\,{\scriptsize$(1.082)$}} & \second{$0.1535$\,{\scriptsize$\pm 0.0032$}\,{\scriptsize$(0.2379)$}} & \second{$0.2672$\,{\scriptsize$\pm 0.0074$}\,{\scriptsize$(0.4131)$}} & $0.6550$\,{\scriptsize$\pm 0.1238$}\,{\scriptsize$(2.128)$} \\
        &$H=16$ & $2.833$\,{\scriptsize$\pm 0.2592$}\,{\scriptsize$(3.203)$} & $2.200$\,{\scriptsize$\pm 0.0906$}\,{\scriptsize$(2.840)$} & $1.265$\,{\scriptsize$\pm 0.0686$}\,{\scriptsize$(1.596)$} & $2.769$\,{\scriptsize$\pm 0.1653$}\,{\scriptsize$(4.694)$} \\
        &$H=32$ & $0.5644$\,{\scriptsize$\pm 0.0160$}\,{\scriptsize$(1.247)$} & $0.2447$\,{\scriptsize$\pm 0.0102$}\,{\scriptsize$(0.3586)$} & $0.2977$\,{\scriptsize$\pm 0.0161$}\,{\scriptsize$(0.4720)$} & \best{$0.5783$\,{\scriptsize$\pm 0.0436$}\,{\scriptsize$(1.698)$}} \\
        &$H=128$ & \best{$0.5019$\,{\scriptsize$\pm 0.0041$}\,{\scriptsize$(1.068)$}} & \best{$0.1478$\,{\scriptsize$\pm 0.0056$}\,{\scriptsize$(0.2205)$}} & \best{$0.2610$\,{\scriptsize$\pm 0.0050$}\,{\scriptsize$(0.4050)$}} & \second{$0.6051$\,{\scriptsize$\pm 0.0082$}\,{\scriptsize$(1.691)$}} \\
        \cmidrule(l){1-6}
        \multirow{4}{*}{Banana} &
        $H=64$ & $1.093$\,{\scriptsize$\pm 0.0069$}\,{\scriptsize$(1.492)$} & \second{$6.112$\,{\scriptsize$\pm 0.0565$}\,{\scriptsize$(7.823)$}} & \second{$2.127$\,{\scriptsize$\pm 0.0138$}\,{\scriptsize$(3.341)$}} & \second{$4.295$\,{\scriptsize$\pm 0.0227$}\,{\scriptsize$(6.514)$}} \\
        &$H=16$ & \best{$0.6292$\,{\scriptsize$\pm 0.0107$}\,{\scriptsize$(0.8657)$}} & $24.527$\,{\scriptsize$\pm 0.2351$}\,{\scriptsize$(30.369)$} & $6.734$\,{\scriptsize$\pm 0.1963$}\,{\scriptsize$(10.040)$} & $5.299$\,{\scriptsize$\pm 0.0465$}\,{\scriptsize$(8.114)$} \\
        &$H=32$ & $1.142$\,{\scriptsize$\pm 0.0098$}\,{\scriptsize$(1.598)$} & $6.902$\,{\scriptsize$\pm 0.0425$}\,{\scriptsize$(8.885)$} & $2.411$\,{\scriptsize$\pm 0.0348$}\,{\scriptsize$(3.775)$} & $4.375$\,{\scriptsize$\pm 0.0058$}\,{\scriptsize$(6.627)$} \\
        &$H=128$ & \second{$1.029$\,{\scriptsize$\pm 0.0021$}\,{\scriptsize$(1.446)$}} & \best{$5.004$\,{\scriptsize$\pm 0.0154$}\,{\scriptsize$(6.429)$}} & \best{$1.872$\,{\scriptsize$\pm 0.0087$}\,{\scriptsize$(2.986)$}} & \best{$4.135$\,{\scriptsize$\pm 0.0205$}\,{\scriptsize$(6.177)$}} \\
      \bottomrule
    \end{tabular}
  }
  \label{tab:results-abl-capacity}
\end{table*}
\subsection{Generalisation to Dissimilar Families} \label{app:generalisation-other}
Here we explore the result of training AAIS on one category of target family and evaluating on another. \cref{tab:results-generalisation-other} shows that each family is best sampled by the sampler that was trained on it -- as expected. The sampler trained on the Banana family appears to generalise the best to the other families. All the samplers show limited generalisation performance when evaluated on the Bayesian hierarchical family.
\begin{table}[ht]
  \caption{Experimental results showing the generalisation of a trained AAIS algorithm to targets from different families. Metrics reported follow the convention of \cref{tab:results-funnel}.}
  \vspace*{-0.5em}
  \resizebox{\linewidth}{!}{
    \begin{tabular}{@{}llcccc}
      \toprule
      &Trained on $\downarrow$ \,\,\,\, Evaluated on $\rightarrow$ & Banana & BLR & Funnel & Hierarchical \\
      \midrule
        \multirow{4}{*}{\rotatebox{90}{$d=5$}} &
        Banana & \best{$0.5072$\,{\scriptsize$\pm0.0042$}\,{\scriptsize($1.132$)}} & \second{$1.790$\,{\scriptsize$\pm0.1044$}\,{\scriptsize($3.365$)}} & \second{$0.9594$\,{\scriptsize$\pm0.0191$}\,{\scriptsize($1.709$)}} & $3.040$\,{\scriptsize$\pm0.0285$}\,{\scriptsize($6.533$)} \\
        &BLR & $1.516$\,{\scriptsize$\pm0.0204$}\,{\scriptsize($2.158$)} & \best{$0.7486$\,{\scriptsize$\pm0.0271$}\,{\scriptsize($2.397$)}} & $1.598$\,{\scriptsize$\pm0.0177$}\,{\scriptsize($2.336$)} & $3.607$\,{\scriptsize$\pm0.0194$}\,{\scriptsize($6.645$)} \\
        &Funnel & \second{$0.7655$\,{\scriptsize$\pm0.0175$}\,{\scriptsize($1.349$)}} & $2.415$\,{\scriptsize$\pm0.1542$}\,{\scriptsize($4.008$)} & \best{$0.4658$\,{\scriptsize$\pm0.0016$}\,{\scriptsize($0.9783$)}} & \second{$2.877$\,{\scriptsize$\pm0.0391$}\,{\scriptsize($6.540$)}} \\
        &Hierarchical & $0.9649$\,{\scriptsize$\pm0.0458$}\,{\scriptsize($1.698$)} & $4.585$\,{\scriptsize$\pm1.097$}\,{\scriptsize($8.091$)} & $1.107$\,{\scriptsize$\pm0.0319$}\,{\scriptsize($1.843$)} & \best{$0.8718$\,{\scriptsize$\pm0.0289$}\,{\scriptsize($1.500$)}} \\
        \cmidrule{1-6}
        \multirow{4}{*}{\rotatebox{90}{$d=25$}}&
        Banana & \best{1.093\,{\scriptsize$\pm0.0069$}\,{\scriptsize$(1.492)$}} & \second{$1.581$\,{\scriptsize$\pm0.0354$}\,{\scriptsize$(2.745)$}} & \second{1.045\,{\scriptsize$\pm0.0203$}\,{\scriptsize$(1.872)$}} & $3.985$\,{\scriptsize$\pm0.0198$}\,{\scriptsize$(8.518)$} \\
        &BLR & $1.610$\,{\scriptsize$\pm0.0280$}\,{\scriptsize$(1.955)$} & \best{$1.174$\,{\scriptsize$\pm0.0334$}\,{\scriptsize$(2.602)$}} & $1.480$\,{\scriptsize$\pm0.0490$}\,{\scriptsize$(2.378)$} & $4.234$\,{\scriptsize$\pm0.0647$}\,{\scriptsize$(8.567)$} \\
        &Funnel & $1.294$\,{\scriptsize$\pm0.1043$}\,{\scriptsize$(1.680)$} & $1.848$\,{\scriptsize$\pm0.1444$}\,{\scriptsize$(2.702)$} & \best{$0.5069$\,{\scriptsize$\pm0.0061$}\,{\scriptsize$(1.082)$}} & \second{$3.783$\,{\scriptsize$\pm0.0331$}\,{\scriptsize$(8.079)$}} \\
        &Hierarchical & \second{$1.179$\,{\scriptsize$\pm0.0258$}\,{\scriptsize$(1.607)$}} & $2.466$\,{\scriptsize$\pm0.0576$}\,{\scriptsize$(4.047)$} & $1.152$\,{\scriptsize$\pm0.0756$}\,{\scriptsize$(1.931)$} & \best{$0.9878$\,{\scriptsize$\pm0.0201$}\,{\scriptsize$(2.414)$}} \\
      \bottomrule
    \end{tabular}
  }
  \label{tab:results-generalisation-other}
\end{table}

\subsection{Additional Metrics for Main Text Experiments}\label{app:additional-experiments-main}
Here we give the full set of metrics for the experiments in the main text: funnel (\cref{tab:results-funnel-full}) and the Bayesian hierarchical family (\cref{tab:results-hierarchical-full}). We also show experiments for the banana family (\cref{tab:results-banana-full}) and the Bayesisan linear regression (\cref{tab:results-blr-full}) family, and a mixture of $5$ Gaussians (\cref{tab:results-gmm5-full}). 

\begin{table*}[t]
  \centering
  \caption{Experimental results for Funnel $d=5,25, 50$. Metrics reported follow the convention of \cref{tab:results-funnel}.}
  \vspace*{-0.5em}
  \resizebox{\linewidth}{!}{
    \begin{tabular}{@{}llccccccc}
      \toprule
      &Method $\downarrow$ Metric $\rightarrow$
      & sliced $W_1^{\mathrm{prop}}$ $\downarrow$ & sliced $W_1$ $\downarrow$ & rev.\ KL $\downarrow$ & ESS $\uparrow$ & $\log Z$ err. $\downarrow$ & $\mathbb{E}[x]$ err. $\downarrow$ & $\mathbb{E}\|x\|^2$ err. $\downarrow$
      \\
      \midrule
        \multirow{5}{*}{\rotatebox{90}{$d=5$}} & AAIS & \best{0.4658\,{\scriptsize$\pm$0.002}\,{\scriptsize(0.978)}} & 
        \best{0.5919\,{\scriptsize$\pm$0.004}\,{\scriptsize(1.21)}} & 
        \best{0.1224\,{\scriptsize$\pm$0.003}\,{\scriptsize(0.194)}} & 
        \best{0.2605\,{\scriptsize$\pm$0.008}\,{\scriptsize(0.015)}} & 
        \best{0.1250\,{\scriptsize$\pm$0.005}\,{\scriptsize(0.241)}} & 
        \best{0.3619\,{\scriptsize$\pm$0.006}\,{\scriptsize(1.07)}} & 
        \best{70.126\,{\scriptsize$\pm$0.267}\,{\scriptsize(246.6)}} \\
        &DM-PMC & 0.9592\,{\scriptsize(1.69)} &
        0.9823\,{\scriptsize(1.83)} &
        11.360\,{\scriptsize(31.36)} &
        0.0301\,{\scriptsize(0.002)} &
        0.2691\,{\scriptsize(0.537)} &
        \second{0.5487\,{\scriptsize(1.31)}} &
        83.221\,{\scriptsize(267.1)} \\
        &M-PMC &
        \second{0.7222\,{\scriptsize(1.43)}} & 
        \second{0.7588\,{\scriptsize(1.57)}} & 
        \second{0.7444\,{\scriptsize(1.52)}} & 
        \second{0.0965\,{\scriptsize(0.004)}} & 
        \second{0.1777\,{\scriptsize(0.347)}} & 
        0.5617\,{\scriptsize(1.43)} &
        78.737\,{\scriptsize(258.9)} \\
        &GAPIS & 
        1.678\,{\scriptsize(2.44)} & 
        1.425\,{\scriptsize(2.44)} &
        $1.6\!\times\!10^{25}$\,{\scriptsize($1.1\!\times\!10^{12}$)} &
        0.0055\,{\scriptsize(0.001)} &
        0.4400\,{\scriptsize(1.05)} & 
        0.9232\,{\scriptsize(1.94)} &
        \second{71.511\,{\scriptsize(259.7)}} \\\cmidrule(l){2-9}
        &True samples &
        0.4451\,{\scriptsize(0.848)} &
        0.5206\,{\scriptsize(0.967)} &
        0.0000\,{\scriptsize(0.000)} &
        1.000\,{\scriptsize(1.00)} &
        0.0000\,{\scriptsize(0.000)} &
        0.1113\,{\scriptsize(0.241)} &
        15.528\,{\scriptsize(73.76)} \\
        \midrule
        \multirow{5}{*}{\rotatebox{90}{$d=25$}} & AAIS & 
        \best{0.5069\,{\scriptsize$\pm$0.006}\,{\scriptsize(1.08)}} & 
        \best{0.5561\,{\scriptsize$\pm$0.008}\,{\scriptsize(1.16)}} & 
        \best{0.1535\,{\scriptsize$\pm$0.003}\,{\scriptsize(0.238)}} & 
        \best{0.1953\,{\scriptsize$\pm$0.006}\,{\scriptsize(0.003)}} & 
        \best{0.2672\,{\scriptsize$\pm$0.007}\,{\scriptsize(0.413)}} & 
        \best{0.6550\,{\scriptsize$\pm$0.124}\,{\scriptsize(2.13)}} & 
        \best{428.2\,{\scriptsize$\pm$7.106}\,{\scriptsize($1.6\!\times\!10^{3}$)}} \\
        &DM-PMC & 
        1.423\,{\scriptsize(2.22)} & 
        \second{1.706\,{\scriptsize(2.67)}} &
        \second{27.603\,{\scriptsize(37.52)}} &
        0.0004\,{\scriptsize(0.000)} &
        \second{4.277\,{\scriptsize(7.06)}} &
        \second{3.920\,{\scriptsize(6.95)}} & 
        483.0\,{\scriptsize($1.7\!\times\!10^{3}$)} \\
        &M-PMC & 
        2.060\,{\scriptsize(2.85)} & 
        2.074\,{\scriptsize(2.88)} & 
        97.424\,{\scriptsize(126.9)} & 
        0.0003\,{\scriptsize(0.000)} & 
        10.626\,{\scriptsize(13.29)} & 
        5.674\,{\scriptsize(8.45)} & 
        461.2\,{\scriptsize($1.6\!\times\!10^{3}$)} \\
        &GAPIS & 
        \second{1.125\,{\scriptsize(1.68)}} & 
        2.045\,{\scriptsize(2.83)} & 
        152.6\,{\scriptsize(207.5)} &
        \second{0.0013\,{\scriptsize(0.000)}} &
        9.789\,{\scriptsize(12.11)} &
        5.947\,{\scriptsize(8.63)} &
        \second{448.7\,{\scriptsize($1.6\!\times\!10^{3}$)}} \\\cmidrule(l){2-9}
        &True samples &
        0.4787\,{\scriptsize(0.941)} & 
        0.5088\,{\scriptsize(1.08)} & 
        0.0000\,{\scriptsize(0.000)} & 
        1.000\,{\scriptsize(1.00)} & 
        0.0000\,{\scriptsize(0.000)} & 
        0.1816\,{\scriptsize(0.397)} & 
        79.879\,{\scriptsize(329.9)} \\
        \midrule
        \multirow{5}{*}{\rotatebox{90}{$d=50$}} & AAIS & 
        \best{0.5123\,{\scriptsize$\pm$0.002}\,{\scriptsize(1.11)}} & 
        \best{0.5039\,{\scriptsize$\pm$0.012}\,{\scriptsize(1.09)}} &
        \best{0.1306\,{\scriptsize$\pm$0.001}\,{\scriptsize(0.214)}} &
        \best{0.1327\,{\scriptsize$\pm$0.003}\,{\scriptsize(0.002)}} &
        \best{0.2921\,{\scriptsize$\pm$0.003}\,{\scriptsize(0.437)}} &
        \best{0.7300\,{\scriptsize$\pm$0.110}\,{\scriptsize(1.90)}} & 
        \best{875.2\,{\scriptsize$\pm$3.980}\,{\scriptsize($3.3\!\times\!10^{3}$)}} \\
        &DM-PMC & 
        1.837\,{\scriptsize(2.65)} & 
        \second{2.090\,{\scriptsize(3.05)}} &
        \second{60.322\,{\scriptsize(76.70)}} &
        0.0001\,{\scriptsize(0.000)} &
        \second{18.597\,{\scriptsize(22.77)}} &
        \second{9.295\,{\scriptsize(12.86)}} &
        934.0\,{\scriptsize($3.4\!\times\!10^{3}$)} \\
        &M-PMC &
        2.263\,{\scriptsize(3.07)} &
        2.266\,{\scriptsize(3.06)} & 
        245.1\,{\scriptsize(260.7)} & 
        0.0001\,{\scriptsize(0.000)} & 
        22.685\,{\scriptsize(26.08)} & 
        9.375\,{\scriptsize(12.92)} & 
        925.8\,{\scriptsize($3.3\!\times\!10^{3}$)} \\
        &GAPIS & 
        \second{1.070\,{\scriptsize(1.64)}} & 
        2.160\,{\scriptsize(3.07)} & 
        356.6\,{\scriptsize(464.3)} &
        \second{0.0006\,{\scriptsize(0.000)}} &
        23.015\,{\scriptsize(26.58)} & 
        9.638\,{\scriptsize(13.59)} & 
        \second{909.4\,{\scriptsize($3.4\!\times\!10^{3}$)}} \\\cmidrule(l){2-9}
        &True samples & 
        0.4830\,{\scriptsize(0.967)} & 
        0.4836\,{\scriptsize(0.910)} & 
        0.0000\,{\scriptsize(0.000)} & 
        1.000\,{\scriptsize(1.00)} & 
        0.0000\,{\scriptsize(0.000)} & 
        0.1553\,{\scriptsize(0.315)} & 
        99.954\,{\scriptsize(431.2)} \\
      \bottomrule
    \end{tabular}
  }
  \label{tab:results-funnel-full}
\end{table*}

\begin{table*}[t]
  \centering
  \caption{Experimental results for Banana $d=5,25, 50$. Metrics reported follow the convention of \cref{tab:results-funnel}.}
  \vspace*{-0.5em}
  \resizebox{\linewidth}{!}{
    \begin{tabular}{@{}llccccccc}
      \toprule
      &Method $\downarrow$ Metric $\rightarrow$
      & sliced $W_1^{\mathrm{prop}}$ $\downarrow$ & sliced $W_1$ $\downarrow$ & rev.\ KL $\downarrow$ & ESS $\uparrow$ & $\log Z$ err. $\downarrow$ & $\mathbb{E}[x]$ err. $\downarrow$ & $\mathbb{E}\|x\|^2$ err. $\downarrow$
      \\
      \midrule
        \multirow{5}{*}{\rotatebox{90}{$d=5$}} & AAIS & 
        \best{0.5072\,{\scriptsize$\pm$0.004}\,{\scriptsize(1.13)}} & 
        \second{0.5685\,{\scriptsize$\pm$0.016}\,{\scriptsize(1.25)}} &
        \best{0.4092\,{\scriptsize$\pm$0.008}\,{\scriptsize(0.867)}} &
        \best{0.1301\,{\scriptsize$\pm$0.003}\,{\scriptsize(0.006)}} & 
        \best{0.1988\,{\scriptsize$\pm$0.011}\,{\scriptsize(0.453)}} & 
        \second{0.9676\,{\scriptsize$\pm$0.029}\,{\scriptsize(2.44)}} &
        9.373\,{\scriptsize$\pm$0.085}\,{\scriptsize(27.18)} \\
        &DM-PMC & 
        0.7214\,{\scriptsize(1.38)} &
        0.6353\,{\scriptsize(1.35)} & 
        3.493\,{\scriptsize(6.80)} & 
        0.0466\,{\scriptsize(0.003)} &
        \second{0.2104\,{\scriptsize(0.473)}} & 
        \best{0.9608\,{\scriptsize(2.17)}} & 
        9.273\,{\scriptsize(27.58)} \\
        &M-PMC &
        \second{0.5467\,{\scriptsize(1.25)}} & 
        \best{0.5637\,{\scriptsize(1.27)}} & 
        \second{2.151\,{\scriptsize(6.55)}} & 
        \second{0.0619\,{\scriptsize(0.002)}} & 
        0.2597\,{\scriptsize(0.688)} & 
        0.9823\,{\scriptsize(2.41)} & 
        \second{8.910\,{\scriptsize(25.39)}} \\
        &GAPIS & 
        1.214\,{\scriptsize(1.69)} &
        0.9583\,{\scriptsize(1.57)} &
        162.3\,{\scriptsize(675.6)} &
        0.0053\,{\scriptsize(0.001)} &
        0.5060\,{\scriptsize(1.09)} &
        1.207\,{\scriptsize(2.30)} &
        \best{6.848\,{\scriptsize(22.08)}} \\\cmidrule(l){2-9}
        &True samples &
        0.2041\,{\scriptsize(0.339)} &
        0.2487\,{\scriptsize(0.408)} & 
        0.0000\,{\scriptsize(0.000)} &
        1.000\,{\scriptsize(1.00)} & 
        0.0000\,{\scriptsize(0.000)} &
        0.0610\,{\scriptsize(0.117)} &
        0.4696\,{\scriptsize(1.55)} \\
       \midrule
       \multirow{5}{*}{\rotatebox{90}{$d=25$}} & AAIS & 
       \second{1.093\,{\scriptsize$\pm$0.007}\,{\scriptsize(1.49)}} &
       \best{1.275\,{\scriptsize$\pm$0.011}\,{\scriptsize(1.90)}} &
       \best{6.112\,{\scriptsize$\pm$0.056}\,{\scriptsize(7.82)}} &
       \best{0.0032\,{\scriptsize$\pm$0.000}\,{\scriptsize(0.000)}} &
       \best{2.127\,{\scriptsize$\pm$0.014}\,{\scriptsize(3.34)}} &
       \best{4.295\,{\scriptsize$\pm$0.023}\,{\scriptsize(6.51)}} & 
       70.088\,{\scriptsize$\pm$0.358}\,{\scriptsize(113.7)} \\
        &DM-PMC & 
        1.356\,{\scriptsize(1.91)} &
        \second{1.699\,{\scriptsize(2.29)}} &
        \second{27.645\,{\scriptsize(41.22)}} &
        0.0004\,{\scriptsize(0.000)} &
        \second{5.999\,{\scriptsize(9.25)}} &
        \second{6.206\,{\scriptsize(8.75)}} &
        58.521\,{\scriptsize(99.73)} \\
        &M-PMC & 
        2.239\,{\scriptsize(2.64)} & 
        2.247\,{\scriptsize(2.66)} & 
        114.9\,{\scriptsize(140.8)} & 
        0.0001\,{\scriptsize(0.000)} & 
        26.816\,{\scriptsize(36.92)} & 
        8.774\,{\scriptsize(11.09)} & 
        \best{31.910\,{\scriptsize(88.34)}} \\
        &GAPIS & 
        \best{0.5332\,{\scriptsize(0.645)}} &
        1.848\,{\scriptsize(2.51)} & 
        445.5\,{\scriptsize(707.8)} &
        \second{0.0010\,{\scriptsize(0.000)}} &
        10.795\,{\scriptsize(14.04)} & 
        6.815\,{\scriptsize(9.97)} & 
        \second{40.659\,{\scriptsize(85.92)}} \\\cmidrule(l){2-9}
        &True samples & 
        0.2357\,{\scriptsize(0.292)} & 
        0.2583\,{\scriptsize(0.333)} &
        0.0000\,{\scriptsize(0.000)} & 
        1.000\,{\scriptsize(1.00)} & 
        0.0000\,{\scriptsize(0.000)} & 
        0.0980\,{\scriptsize(0.134)} &
        0.8277\,{\scriptsize(2.35)} \\
        \midrule
        \multirow{5}{*}{\rotatebox{90}{$d=50$}} & AAIS &
        \second{1.196\,{\scriptsize$\pm$0.016}\,{\scriptsize(1.51)}} &
        \best{1.482\,{\scriptsize$\pm$0.006}\,{\scriptsize(2.06)}} &
        \best{13.702\,{\scriptsize$\pm$0.186}\,{\scriptsize(16.68)}} &
        \best{0.0004\,{\scriptsize$\pm$0.000}\,{\scriptsize(0.000)}} &
        \best{5.208\,{\scriptsize$\pm$0.149}\,{\scriptsize(7.47)}} &
        \best{7.071\,{\scriptsize$\pm$0.033}\,{\scriptsize(10.02)}} &
        154.4\,{\scriptsize$\pm$0.762}\,{\scriptsize(218.1)} \\
        &DM-PMC & 
        1.761\,{\scriptsize(2.28)} & 
        2.159\,{\scriptsize(2.68)} &
        \second{75.504\,{\scriptsize(106.1)}} &
        0.0001\,{\scriptsize(0.000)} &
        \second{24.740\,{\scriptsize(34.40)}} &
        12.437\,{\scriptsize(16.33)} &
        80.989\,{\scriptsize(163.7)} \\
        &M-PMC & 
        2.394\,{\scriptsize(2.68)} &
        2.397\,{\scriptsize(2.69)} &
        275.8\,{\scriptsize(303.4)} &
        0.0000\,{\scriptsize(0.000)} &
        68.704\,{\scriptsize(91.90)} &
        13.831\,{\scriptsize(16.22)} &
        \best{50.547\,{\scriptsize(128.0)}} \\
        &GAPIS & 
        \best{0.4102\,{\scriptsize(0.480)}} &
        \second{2.047\,{\scriptsize(2.51)}} &
        101.2\,{\scriptsize(88.85)} &
        \second{0.0004\,{\scriptsize(0.000)}} &
        29.252\,{\scriptsize(34.85)} &
        \second{10.948\,{\scriptsize(14.38)}} &
        \second{80.799\,{\scriptsize(152.2)}} \\ \cmidrule(l){2-9}
        &True samples & 
        0.2424\,{\scriptsize(0.282)} & 
        0.2470\,{\scriptsize(0.291)} & 
        0.0000\,{\scriptsize(0.000)} & 
        1.000\,{\scriptsize(1.00)} & 
        0.0000\,{\scriptsize(0.000)} & 
        0.0857\,{\scriptsize(0.110)} & 
        0.7316\,{\scriptsize(1.99)} \\
      \bottomrule
    \end{tabular}
  }
  \label{tab:results-banana-full}
\end{table*}

\begin{table*}[t]
  \centering
  \caption{Experimental results for GMM5 $d=2,5,25$. Metrics reported follow the convention of \cref{tab:results-funnel}.}
  \vspace*{-0.5em}
  \resizebox{\linewidth}{!}{
    \begin{tabular}{@{}llccccccc}
      \toprule
      &Method $\downarrow$ Metric $\rightarrow$
      & sliced $W_1^{\mathrm{prop}}$ $\downarrow$ & sliced $W_1$ $\downarrow$ & rev.\ KL $\downarrow$ & ESS $\uparrow$ & $\log Z$ err. $\downarrow$ & $\mathbb{E}[x]$ err. $\downarrow$ & $\mathbb{E}\|x\|^2$ err. $\downarrow$
      \\
      \midrule
        \multirow{5}{*}{\rotatebox{90}{$d=2$}} 
        &Learnt & \second{$0.6752$\,{\scriptsize$\pm 0.0028$}\,{\scriptsize$(1.579)$}} & \second{$0.4618$\,{\scriptsize$\pm 0.0042$}\,{\scriptsize$(1.200)$}} & \second{$0.2126$\,{\scriptsize$\pm 0.0032$}\,{\scriptsize$(0.5842)$}} & \best{$0.4159$\,{\scriptsize$\pm 0.0054$}\,{\scriptsize$(0.0371)$}} & \second{$0.1191$\,{\scriptsize$\pm 0.0032$}\,{\scriptsize$(0.4545)$}} & \second{$0.4452$\,{\scriptsize$\pm 0.0067$}\,{\scriptsize$(1.262)$}} & $0.9646$\,{\scriptsize$\pm 0.0329$}\,{\scriptsize$(3.482)$} \\
        &DM-PMC & $1.116$\,{\scriptsize$(2.179)$} & $0.8316$\,{\scriptsize$(1.856)$} & $1.199$\,{\scriptsize$(2.092)$} & $0.1341$\,{\scriptsize$(0.0073)$} & $0.1933$\,{\scriptsize$(0.4949)$} & $0.6694$\,{\scriptsize$(1.910)$} & $1.110$\,{\scriptsize$(2.738)$} \\
        &M-PMC & \best{$0.3899$\,{\scriptsize$(0.9256)$}} & \best{$0.3840$\,{\scriptsize$(0.9017)$}} & \best{$0.2054$\,{\scriptsize$(0.6037)$}} & \second{$0.4153$\,{\scriptsize$(0.0352)$}} & \best{$0.0840$\,{\scriptsize$(0.3159)$}} & \best{$0.3541$\,{\scriptsize$(1.278)$}} & \best{$0.6824$\,{\scriptsize$(2.382)$}} \\
        &GAPIS & $0.7924$\,{\scriptsize$(1.595)$} & $0.5031$\,{\scriptsize$(1.269)$} & $0.4298$\,{\scriptsize$(0.9473)$} & $0.3072$\,{\scriptsize$(0.0115)$} & $0.1297$\,{\scriptsize$(0.4825)$} & $0.4465$\,{\scriptsize$(1.477)$} & \second{$0.9507$\,{\scriptsize$(3.325)$}} \\
        \cmidrule(l){2-9}
        &True samples & $0.1862$\,{\scriptsize$(0.3086)$} & $0.2578$\,{\scriptsize$(0.4077)$} & $0.0000$\,{\scriptsize$(0.0000)$} & $1.0000$\,{\scriptsize$(1.0000)$} & $0.0000$\,{\scriptsize$(0.0000)$} & $0.0505$\,{\scriptsize$(0.1023)$} & $0.1148$\,{\scriptsize$(0.2929)$} \\
            
       \midrule
       \multirow{5}{*}{\rotatebox{90}{$d=5$}} & Learnt & \second{$0.7183$\,{\scriptsize$\pm 0.0043$}\,{\scriptsize$(1.200)$}} & \second{$0.5029$\,{\scriptsize$\pm 0.0048$}\,{\scriptsize$(0.9176)$}} & \second{$0.3738$\,{\scriptsize$\pm 0.0023$}\,{\scriptsize$(0.7057)$}} & \second{$0.2325$\,{\scriptsize$\pm 0.0012$}\,{\scriptsize$(0.0095)$}} & \second{$0.2625$\,{\scriptsize$\pm 0.0030$}\,{\scriptsize$(0.5819)$}} & \second{$0.9929$\,{\scriptsize$\pm 0.0063$}\,{\scriptsize$(2.060)$}} & \second{$2.545$\,{\scriptsize$\pm 0.0308$}\,{\scriptsize$(7.899)$}} \\
        &DM-PMC & $1.251$\,{\scriptsize$(2.170)$} & $1.208$\,{\scriptsize$(2.205)$} & $3.393$\,{\scriptsize$(4.342)$} & $0.0457$\,{\scriptsize$(0.0008)$} & $0.6904$\,{\scriptsize$(1.585)$} & $2.560$\,{\scriptsize$(5.181)$} & $4.223$\,{\scriptsize$(10.547)$} \\
        &M-PMC & $1.578$\,{\scriptsize$(2.502)$} & $1.588$\,{\scriptsize$(2.504)$} & $1.532$\,{\scriptsize$(2.642)$} & $0.0929$\,{\scriptsize$(0.0008)$} & $1.237$\,{\scriptsize$(1.946)$} & $3.369$\,{\scriptsize$(5.651)$} & $6.179$\,{\scriptsize$(14.867)$} \\
        &GAPIS & \best{$0.7138$\,{\scriptsize$(1.194)$}} & \best{$0.4777$\,{\scriptsize$(0.8853)$}} & \best{$0.3551$\,{\scriptsize$(0.7013)$}} & \best{$0.2661$\,{\scriptsize$(0.0066)$}} & \best{$0.2188$\,{\scriptsize$(0.5239)$}} & \best{$0.9343$\,{\scriptsize$(2.171)$}} & \best{$2.116$\,{\scriptsize$(6.743)$}} \\
        \cmidrule(l){2-9}
        &True samples & $0.1965$\,{\scriptsize$(0.3019)$} & $0.2303$\,{\scriptsize$(0.3374)$} & $0.0000$\,{\scriptsize$(0.0000)$} & $1.000$\,{\scriptsize$(1.000)$} & $0.0000$\,{\scriptsize$(0.0000)$} & $0.0612$\,{\scriptsize$(0.1118)$} & $0.1294$\,{\scriptsize$(0.3106)$} \\
        \midrule
        \multirow{5}{*}{\rotatebox{90}{$d=25$}} &Learnt & \second{$0.5869$\,{\scriptsize$\pm 0.0012$}\,{\scriptsize$(0.9770)$}} & \best{$0.2691$\,{\scriptsize$\pm 0.0034$}\,{\scriptsize$(0.5840)$}} & \best{$0.2416$\,{\scriptsize$\pm 0.0072$}\,{\scriptsize$(0.5675)$}} & \best{$0.1370$\,{\scriptsize$\pm 0.0018$}\,{\scriptsize$(0.0044)$}} & \best{$0.3041$\,{\scriptsize$\pm 0.0022$}\,{\scriptsize$(0.5076)$}} & \best{$1.036$\,{\scriptsize$\pm 0.0124$}\,{\scriptsize$(3.098)$}} & \best{$2.918$\,{\scriptsize$\pm 0.0906$}\,{\scriptsize$(13.044)$}} \\
        &DM-PMC & $1.838$\,{\scriptsize$(2.153)$} & $2.102$\,{\scriptsize$(2.451)$} & $22.754$\,{\scriptsize$(27.329)$} & $0.0004$\,{\scriptsize$(0.0001)$} & $3.386$\,{\scriptsize$(5.442)$} & $10.139$\,{\scriptsize$(12.222)$} & $20.518$\,{\scriptsize$(48.346)$} \\
        &M-PMC & $2.390$\,{\scriptsize$(2.634)$} & $2.394$\,{\scriptsize$(2.633)$} & $177.4$\,{\scriptsize$(213.7)$} & $0.0001$\,{\scriptsize$(0.0001)$} & $87.586$\,{\scriptsize$(115.1)$} & $10.349$\,{\scriptsize$(12.579)$} & $34.230$\,{\scriptsize$(70.776)$} \\
        &GAPIS & \best{$0.5746$\,{\scriptsize$(0.9472)$}} & \second{$0.3353$\,{\scriptsize$(0.6209)$}} & \second{$0.6309$\,{\scriptsize$(1.057)$}} & \second{$0.0284$\,{\scriptsize$(0.0018)$}} & \second{$0.6016$\,{\scriptsize$(0.9685)$}} & \second{$2.163$\,{\scriptsize$(5.333)$}} & \second{$5.237$\,{\scriptsize$(17.045)$}} \\
        \cmidrule(l){2-9}
        &True samples & $0.1905$\,{\scriptsize$(0.2713)$} & $0.2085$\,{\scriptsize$(0.3031)$} & $0.0000$\,{\scriptsize$(0.0000)$} & $1.000$\,{\scriptsize$(1.000)$} & $0.0000$\,{\scriptsize$(0.0000)$} & $0.0943$\,{\scriptsize$(0.1493)$} & $0.1951$\,{\scriptsize$(0.5108)$} \\
      \bottomrule
    \end{tabular}
  }
  \label{tab:results-gmm5-full}
\end{table*}

\begin{table*}[t]
  \centering
  \caption{Experimental results for Bayesian linear regression $d=5,25$. Metrics reported follow the convention of \cref{tab:results-funnel}.}
  \vspace*{-0.5em}
  \resizebox{\linewidth}{!}{
    \begin{tabular}{@{}llccccccc}
      \toprule
      &Method $\downarrow$ Metric $\rightarrow$
      & sliced $W_1^{\mathrm{prop}}$ $\downarrow$ & sliced $W_1$ $\downarrow$ & rev.\ KL $\downarrow$ & ESS $\uparrow$ & $\log Z$ err. $\downarrow$ & $\mathbb{E}[x]$ err. $\downarrow$ & $\mathbb{E}\|x\|^2$ err. $\downarrow$
      \\
      \midrule
        \multirow{5}{*}{\rotatebox{90}{$d=5$}}& AAIS & \best{0.7489\,{\scriptsize$\pm$0.027}\,{\scriptsize(2.39)}} &
        \best{0.7158\,{\scriptsize$\pm$0.031}\,{\scriptsize(2.41)}} &
        \best{25.102\,{\scriptsize$\pm$1.016}\,{\scriptsize(90.71)}} &
        \best{0.0140\,{\scriptsize$\pm$0.001}\,{\scriptsize(0.000)}} &
        \second{21.971\,{\scriptsize$\pm$1.398}\,{\scriptsize(85.87)}} &
        \second{1.869\,{\scriptsize$\pm$0.081}\,{\scriptsize(6.14)}} &
        \second{3.852\,{\scriptsize$\pm$0.125}\,{\scriptsize(17.16)}} \\
        &DM-PMC &
        \second{1.099\,{\scriptsize(2.20)}} &
        \second{0.7844\,{\scriptsize(2.07)}} &
        $2.0\!\times\!10^{3}$\,{\scriptsize($6.4\!\times\!10^{3}$)} &
        0.0005\,{\scriptsize(0.000)} &
        \best{14.359\,{\scriptsize(54.53)}} &
        \best{1.419\,{\scriptsize(4.65)}} & 
        \best{3.311\,{\scriptsize(11.96)}} \\
        &M-PMC &
        2.263\,{\scriptsize(4.17)} &
        2.223\,{\scriptsize(4.18)} &
        \second{114.1\,{\scriptsize(262.1)}} &
        \second{0.0016\,{\scriptsize(0.000)}} & 
        111.3\,{\scriptsize(257.0)} &
        5.941\,{\scriptsize(11.14)} &
        30.751\,{\scriptsize(91.58)} \\
        &GAPIS 
        & 34.469\,{\scriptsize(163.6)} 
        & 2.999\,{\scriptsize(4.71)} 
        & $6.8\!\times\!10^{3}$\,{\scriptsize($1.5\!\times\!10^{4}$)} 
        & 0.0003\,{\scriptsize(0.000)} 
        & 69.437\,{\scriptsize(144.8)} 
        & 5.895\,{\scriptsize(10.22)} 
        & 30.323\,{\scriptsize(75.57)} \\\cmidrule(l){2-9}
        &True samples & 0.0169\,{\scriptsize(0.030)} & 0.0199\,{\scriptsize(0.036)} & 0.0000\,{\scriptsize(0.000)} & 1.000\,{\scriptsize(1.00)} & 0.0000\,{\scriptsize(0.000)} & 0.0048\,{\scriptsize(0.011)} & 0.0104\,{\scriptsize(0.027)} \\
        \midrule
        \multirow{5}{*}{\rotatebox{90}{$d=25$}}& AAIS &
        \best{1.174\,{\scriptsize$\pm$0.033}\,{\scriptsize(2.60)}} & 
        \best{1.091\,{\scriptsize$\pm$0.044}\,{\scriptsize(2.49)}} & 
        \best{409.9\,{\scriptsize$\pm$3.154}\,{\scriptsize($1.0\!\times\!10^{3}$)}} &
        \second{0.0004\,{\scriptsize$\pm$0.000}\,{\scriptsize(0.000)}} & 
        \best{374.4\,{\scriptsize$\pm$4.126}\,{\scriptsize(994.0)}} & 
        \best{6.832\,{\scriptsize$\pm$0.265}\,{\scriptsize(15.64)}} & 
        \best{33.606\,{\scriptsize$\pm$5.746}\,{\scriptsize(140.9)}} \\
        &DM-PMC &
        \second{2.033\,{\scriptsize(3.41)}} & 
        \second{1.838\,{\scriptsize(3.33)}} & 
        \second{938.4\,{\scriptsize($1.8\!\times\!10^{3}$)}} & 
        0.0001\,{\scriptsize(0.000)} & 
        \second{524.6\,{\scriptsize($1.0\!\times\!10^{3}$)}} & 
        \second{11.236\,{\scriptsize(20.26)}} & 
        \second{112.4\,{\scriptsize(310.3)}} \\
        &M-PMC & 
        3.180\,{\scriptsize(3.95)} & 
        3.180\,{\scriptsize(3.95)} & 
        $2.0\!\times\!10^{3}$\,{\scriptsize($6.6\!\times\!10^{3}$)} & 
        0.0001\,{\scriptsize(0.000)} & 
        $1.9\!\times\!10^{3}$\,{\scriptsize($6.4\!\times\!10^{3}$)} & 
        19.805\,{\scriptsize(24.38)} & 
        234.6\,{\scriptsize(350.7)} \\
        &GAPIS & 
        279.4\,{\scriptsize($1.4\!\times\!10^{3}$)} &
        3.269\,{\scriptsize(3.95)} & 
        $2.0\!\times\!10^{5}$\,{\scriptsize($1.0\!\times\!10^{6}$)} & 
        \best{0.0005\,{\scriptsize(0.000)}} &
        804.8\,{\scriptsize($1.4\!\times\!10^{3}$)} &
        19.089\,{\scriptsize(23.66)} &
        225.5\,{\scriptsize(358.0)} \\\cmidrule(l){2-9}
        &True samples & 0.0047\,{\scriptsize(0.009)} & 0.0051\,{\scriptsize(0.009)} & 0.0000\,{\scriptsize(0.000)} & 1.000\,{\scriptsize(1.00)} & 0.0000\,{\scriptsize(0.000)} & 0.0020\,{\scriptsize(0.004)} & 0.0056\,{\scriptsize(0.018)} \\
      \bottomrule
    \end{tabular}
  }
  \label{tab:results-blr-full}
\end{table*}

\begin{table*}[t]
  \centering
  \caption{Experimental results for the Bayesian hierarchical model $d=5,25$. Metrics reported follow the convention of \cref{tab:results-funnel}.}
  \vspace*{-0.5em}
  \resizebox{\linewidth}{!}{
    \begin{tabular}{@{}llccccccc}
      \toprule
      &Method $\downarrow$ Metric $\rightarrow$
      & sliced $W_1^{\mathrm{prop}}$ $\downarrow$ & sliced $W_1$ $\downarrow$ & rev.\ KL $\downarrow$ & ESS $\uparrow$ & $\log Z$ err. $\downarrow$ & $\mathbb{E}[x]$ err. $\downarrow$ & $\mathbb{E}\|x\|^2$ err. $\downarrow$
      \\
      \midrule
        \multirow{5}{*}{\rotatebox{90}{$d=5$}} & AAIS & 
        \best{0.8718\,{\scriptsize$\pm$0.029}\,{\scriptsize(1.50)}} &
        \best{1.077\,{\scriptsize$\pm$0.017}\,{\scriptsize(2.23)}} &
        \best{1.120\,{\scriptsize$\pm$0.021}\,{\scriptsize(1.43)}} &
        \best{0.0358\,{\scriptsize$\pm$0.003}\,{\scriptsize(0.002)}} &
        \best{0.5812\,{\scriptsize$\pm$0.007}\,{\scriptsize(0.972)}} &
        \best{1.678\,{\scriptsize$\pm$0.066}\,{\scriptsize(4.08)}} &
        \best{39.313\,{\scriptsize$\pm$1.213}\,{\scriptsize(94.82)}} \\
        &DM-PMC & 
        \second{1.594\,{\scriptsize(3.16)}} &
        \second{1.650\,{\scriptsize(3.33)}} &
        81.933\,{\scriptsize(316.4)} &
        \second{0.0120\,{\scriptsize(0.001)}} &
        \second{0.8585\,{\scriptsize(1.69)}} &
        \second{2.575\,{\scriptsize(6.85)}} &
        \second{62.201\,{\scriptsize(182.2)}} \\
        &M-PMC & 
        2.005\,{\scriptsize(4.31)} & 
        1.999\,{\scriptsize(3.93)} &
        \second{10.592\,{\scriptsize(28.46)}} &
        0.0081\,{\scriptsize(0.000)} &
        1.674\,{\scriptsize(3.81)} & 
        3.707\,{\scriptsize(9.78)} & 
        68.642\,{\scriptsize(187.8)} \\
        &GAPIS & 
        138.7\,{\scriptsize(543.3)} & 
        4.140\,{\scriptsize(8.03)} &
        $1.8\!\times\!10^{6}$\,{\scriptsize($8.9\!\times\!10^{6}$)} &
        0.0016\,{\scriptsize(0.000)} & 
        2.449\,{\scriptsize(5.90)} & 
        6.160\,{\scriptsize(14.54)} & 
        86.651\,{\scriptsize(242.6)} \\\cmidrule(l){2-9}
        &True samples & 
        0.3458\,{\scriptsize(0.565)} & 
        0.4084\,{\scriptsize(0.665)} & 
        0.0000\,{\scriptsize(0.000)} & 
        1.000\,{\scriptsize(1.00)} & 
        0.0000\,{\scriptsize(0.000)} & 
        0.0900\,{\scriptsize(0.182)} & 
        1.432\,{\scriptsize(4.17)} \\
        \midrule
        \multirow{5}{*}{\rotatebox{90}{$d=25$}} & AAIS & 
        \best{0.9877\,{\scriptsize$\pm$0.020}\,{\scriptsize(2.41)}} &
        \best{1.201\,{\scriptsize$\pm$0.028}\,{\scriptsize(3.01)}} & 
        \best{2.157\,{\scriptsize$\pm$0.016}\,{\scriptsize(3.14)}} &
        \best{0.0125\,{\scriptsize$\pm$0.001}\,{\scriptsize(0.001)}} &
        \best{1.279\,{\scriptsize$\pm$0.022}\,{\scriptsize(2.19)}} &
        \best{4.531\,{\scriptsize$\pm$0.077}\,{\scriptsize(12.81)}} &
        \best{193.4\,{\scriptsize$\pm$1.506}\,{\scriptsize(763.2)}} \\
        &DM-PMC & 
        \second{3.675\,{\scriptsize(6.99)}} & 
        \second{3.815\,{\scriptsize(7.02)}} &
        \second{125.9\,{\scriptsize(208.6)}} &
        0.0003\,{\scriptsize(0.000)} & 
        20.770\,{\scriptsize(33.99)} &
        \second{19.843\,{\scriptsize(39.56)}} &
        \second{710.4\,{\scriptsize($2.4\!\times\!10^{3}$)}} \\
        &M-PMC & 
        5.413\,{\scriptsize(9.27)} & 
        5.420\,{\scriptsize(9.29)} & 
        137.3\,{\scriptsize(169.1)} & 
        0.0001\,{\scriptsize(0.000)} & 
        48.089\,{\scriptsize(76.65)} & 
        29.205\,{\scriptsize(52.29)} & 
        742.7\,{\scriptsize($2.6\!\times\!10^{3}$)} \\
        &GAPIS & 
        429.0\,{\scriptsize($1.2\!\times\!10^{3}$)} & 
        8.042\,{\scriptsize(11.18)} &
        $7.5\!\times\!10^{7}$\,{\scriptsize($2.9\!\times\!10^{8}$)} &
        \second{0.0007\,{\scriptsize(0.000)}} &
        \second{19.314\,{\scriptsize(41.64)}} &
        39.403\,{\scriptsize(60.79)} &
        $1.2\!\times\!10^{3}$\,{\scriptsize($2.7\!\times\!10^{3}$)} \\\cmidrule(l){2-9}
        &True samples & 
        0.3343\,{\scriptsize(0.598)} & 
        0.3607\,{\scriptsize(0.657)} &
        0.0000\,{\scriptsize(0.000)} & 
        1.000\,{\scriptsize(1.00)} &
        0.0000\,{\scriptsize(0.000)} & 
        0.1370\,{\scriptsize(0.274)} & 
        3.084\,{\scriptsize(9.11)} \\
      \bottomrule
    \end{tabular}
  }
  \label{tab:results-hierarchical-full}
\end{table*}

\begin{table*}[t]
  \centering
  \caption{Experimental results showing the performance of AAIS amortised over a single target vs over a whole family, all targets are at $d=25$. Metrics reported follow the convention of \cref{tab:results-funnel}.}
  \vspace*{-0.5em}
  \resizebox{\linewidth}{!}{
    \begin{tabular}{@{}clccccc}
      \toprule
      Target Family&Amortised over & Evaluated on & sliced $W_1^{\mathrm{prop}}$ $\downarrow$ & rev.\ KL $\downarrow$ & $\log Z$ err. $\downarrow$ & $\mathbb{E}[x]$ err. $\downarrow$ \\
      \midrule
        \multirow{4}{*}{\rotatebox{0}{Banana}}
        &Single-target & \multirow{2}{*}{\rotatebox{0}{Single-target}} & \best{$0.6679$\,{\scriptsize$\pm 0.0009$}\,{\scriptsize$(0.7235)$}}  & \best{$3.445$\,{\scriptsize$\pm 0.0024$}\,{\scriptsize$(3.675)$}}  & \best{$1.096$\,{\scriptsize$\pm 0.0005$}\,{\scriptsize$(1.529)$}} & \best{$2.821$\,{\scriptsize$\pm 0.0426$}\,{\scriptsize$(5.602)$}}  \\
        &Family &  & {$0.8139$\,{\scriptsize$\pm 0.0048$}\,{\scriptsize$(0.8673)$}}  & {$4.671$\,{\scriptsize$\pm 0.0535$}\,{\scriptsize$(4.920)$}}  & {$1.500$\,{\scriptsize$\pm 0.0244$}\,{\scriptsize$(2.102)$}} & {$3.160$\,{\scriptsize$\pm 0.0164$}\,{\scriptsize$(5.759)$}}  \\
        \cmidrule(l){2-7}
        &Single-target & \multirow{2}{*}{\rotatebox{0}{Family}} & \best{$0.8439$\,{\scriptsize$\pm 0.0201$}\,{\scriptsize$(1.162)$}}  & {$147.0$\,{\scriptsize$\pm 17.239$}\,{\scriptsize$(500.2)$}}  & {$2.563$\,{\scriptsize$\pm 0.0315$}\,{\scriptsize$(4.500)$}} & {$4.627$\,{\scriptsize$\pm 0.0680$}\,{\scriptsize$(7.142)$}}  \\
        &Family &  & {$1.093$\,{\scriptsize$\pm 0.0069$}\,{\scriptsize$(1.492)$}}
        & \best{$6.112$\,{\scriptsize$\pm 0.0565$}\,{\scriptsize$(7.823)$}}  & \best{$2.127$\,{\scriptsize$\pm 0.0138$}\,{\scriptsize$(3.340)$}} & \best{$4.295$\,{\scriptsize$\pm 0.0227$}\,{\scriptsize$(6.514)$}}  \\
        \cmidrule(l){1-7}
        \multirow{4}{*}{\rotatebox{0}{Funnel}}
        &Single-target & \multirow{2}{*}{\rotatebox{0}{Single-target}} & \best{$0.3999$\,{\scriptsize$\pm 0.0002$}\,{\scriptsize$(0.5568)$}}  & {$0.1336$\,{\scriptsize$\pm 0.0009$}\,{\scriptsize$(0.1866)$}}  & {$0.2656$\,{\scriptsize$\pm 0.0062$}\,{\scriptsize$(0.3187)$}} & \best{$0.6306$\,{\scriptsize$\pm 0.0416$}\,{\scriptsize$(2.064)$}}  \\
        &Family &  & {$0.4041$\,{\scriptsize$\pm 0.0032$}\,{\scriptsize$(0.5610)$}}  & \best{$0.1309$\,{\scriptsize$\pm 0.0023$}\,{\scriptsize$(0.1845)$}}  & \best{$0.2647$\,{\scriptsize$\pm 0.0081$}\,{\scriptsize$(0.3222)$}} & {$0.6437$\,{\scriptsize$\pm 0.0836$}\,{\scriptsize$(2.051)$}}  \\
        \cmidrule(l){2-7}
        &Single-target & \multirow{2}{*}{\rotatebox{0}{Family}} & {$0.5072$\,{\scriptsize$\pm 0.0044$}\,{\scriptsize$(1.051)$}}  & {$0.1622$\,{\scriptsize$\pm 0.0017$}\,{\scriptsize$(0.2438)$}}  & \best{$0.2661$\,{\scriptsize$\pm 0.0049$}\,{\scriptsize$(0.4080)$}} & \best{$0.6296$\,{\scriptsize$\pm 0.0358$}\,{\scriptsize$(2.098)$}}  \\
        &Family &  & \best{$0.5069$\,{\scriptsize$\pm 0.0061$}\,{\scriptsize$(1.082)$}}  & \best{$0.1535$\,{\scriptsize$\pm 0.0032$}\,{\scriptsize$(0.2379)$}}  & {$0.2672$\,{\scriptsize$\pm 0.0074$}\,{\scriptsize$(0.4131)$}} & {$0.6550$\,{\scriptsize$\pm 0.1238$}\,{\scriptsize$(2.128)$}}  \\
        \cmidrule{1-7}
        \multirow{4}{*}{\rotatebox{0}{BLR}}
        &Single-target & \multirow{2}{*}{\rotatebox{0}{Single-target}} & \best{$1.308$\,{\scriptsize$\pm 0.0322$}\,{\scriptsize$(1.420)$}}  & \best{$412.6$\,{\scriptsize$\pm 3.939$}\,{\scriptsize$(432.2)$}}   & \best{$293.7$\,{\scriptsize$\pm 5.751$}\,{\scriptsize$(335.1)$}} & \best{$4.827$\,{\scriptsize$\pm 0.1062$}\,{\scriptsize$(5.905)$}}  \\
        &Family &  & {$1.399$\,{\scriptsize$\pm 0.1051$}\,{\scriptsize$(1.609)$}}  & {$461.4$\,{\scriptsize$\pm 4.167$}\,{\scriptsize$(492.9)$}}  & {$428.8$\,{\scriptsize$\pm 7.762$}\,{\scriptsize$(461.9)$}} & {$7.790$\,{\scriptsize$\pm 0.5420$}\,{\scriptsize$(9.006)$}}  \\
        \cmidrule(l){2-7}
        &Single-target & \multirow{2}{*}{\rotatebox{0}{Family}} & {$2.481$\,{\scriptsize$\pm 0.0830$}\,{\scriptsize$(3.223)$}}  & {$786.0$\,{\scriptsize$\pm 5.995$}\,{\scriptsize$(1.4\!\times\!10^{3})$}}  & {$605.2$\,{\scriptsize$\pm 9.940$}\,{\scriptsize$(1.1\!\times\!10^{3})$}} & {$12.123$\,{\scriptsize$\pm 0.3951$}\,{\scriptsize$(18.309)$}}  \\
        &Family &  & \best{$1.174$\,{\scriptsize$\pm 0.0332$}\,{\scriptsize$(2.600)$}}  & \best{$409.9$\,{\scriptsize$\pm 2.980$}\,{\scriptsize$(1.1\!\times\!10^{3})$}}  & \best{$374.3$\,{\scriptsize$\pm 4.139$}\,{\scriptsize$(998.9)$}} & \best{$6.826$\,{\scriptsize$\pm 0.2540$}\,{\scriptsize$(15.623)$}}  \\
        \cmidrule(l){1-7}
        \multirow{4}{*}{\rotatebox{0}{Hierarchical}}
        &Single-target & \multirow{2}{*}{\rotatebox{0}{Single-target}} & \best{$0.4696$\,{\scriptsize$\pm 0.0058$}\,{\scriptsize$(0.6323)$}}  & \best{$1.405$\,{\scriptsize$\pm 0.0140$}\,{\scriptsize$(1.546)$}}  & \best{$0.9722$\,{\scriptsize$\pm 0.0194$}\,{\scriptsize$(1.208)$}} & \best{$2.012$\,{\scriptsize$\pm 0.0388$}\,{\scriptsize$(4.466)$}}  \\
        &Family &  & {$0.5641$\,{\scriptsize$\pm 0.0198$}\,{\scriptsize$(0.7198)$}}  & {$1.849$\,{\scriptsize$\pm 0.0380$}\,{\scriptsize$(1.994)$}}  & {$1.281$\,{\scriptsize$\pm 0.0208$}\,{\scriptsize$(1.572)$}} & {$2.601$\,{\scriptsize$\pm 0.0934$}\,{\scriptsize$(5.415)$}}  \\
        \cmidrule(l){2-7}
        &Single-target & \multirow{2}{*}{\rotatebox{0}{Family}} & {$3.399$\,{\scriptsize$\pm 0.1079$}\,{\scriptsize$(7.696)$}}  & {$9.786$\,{\scriptsize$\pm 0.8974$}\,{\scriptsize$(29.533)$}}  & {$4.720$\,{\scriptsize$\pm 0.1137$}\,{\scriptsize$(15.600)$}} & {$13.206$\,{\scriptsize$\pm 0.1720$}\,{\scriptsize$(36.371)$}}  \\
        &Family &  & \best{$0.9878$\,{\scriptsize$\pm 0.0201$}\,{\scriptsize$(2.414)$}}  & \best{$2.157$\,{\scriptsize$\pm 0.0160$}\,{\scriptsize$(3.137)$}}  & \best{$1.279$\,{\scriptsize$\pm 0.0214$}\,{\scriptsize$(2.196)$}} & \best{$4.530$\,{\scriptsize$\pm 0.0764$}\,{\scriptsize$(12.802)$}}  \\
      \bottomrule
    \end{tabular}
  }
  \label{tab:results-single-vs-family-full}
\end{table*}

\end{document}